\documentclass[11pt]{article}

\usepackage{color}   %May be necessary if you want to color links
\usepackage{hyperref}
\hypersetup{
	colorlinks=true, %set true if you want colored links
	linktoc=all,     %set to all if you want both sections and subsections linked
	linkcolor=blue,  %choose some color if you want links to stand out
	citecolor=red,
}
\usepackage{geometry}                % See geometry.pdf to learn the layout options. There are lots.
\usepackage{graphicx}
\usepackage{amssymb}
\usepackage{amsmath}
\usepackage{appendix}
\usepackage{tikz}
\usepackage[numbers]{natbib}

\usetikzlibrary{positioning}% To get more advances positioning options
\usetikzlibrary{arrows}% To get more arrow heads

\newtheorem{notation}{Notation}

\newtheorem{proposition}{Proposition}[section]

\newtheorem{assumption}{Assumption}[section]
\newtheorem{theorem}{Theorem}[section]
\newtheorem{corollary}{Corollary}[section]
\newtheorem{lemma}{Lemma}[section]

\newenvironment{proof}{{\noindent\it Proof}\quad}{\hfill $\square$\par} 

\numberwithin{figure}{section}
\numberwithin{equation}{section}

\usepackage{algorithm}
\usepackage{algorithmic}
\usepackage{bm}

\usepackage{stmaryrd}
\usepackage{diagbox}

\newcommand{\footremember}[2]{%
    \footnote{#2}
    \newcounter{#1}
    \setcounter{#1}{\value{footnote}}%
}
\newcommand{\footrecall}[1]{%
    \footnotemark[\value{#1}]%
}

\begin{document}

\title{Side Effects of Learning from Low-dimensional Data Embedded in a Euclidean Space}
\author{Juncai He\footremember{MATH}{Department of Mathematics, The University of Texas at Austin, Austin, TX 78712, USA} \and Richard Tsai\footrecall{MATH} \footremember{ODEN}{Oden Institute for Computational Engineering and Sciences, The University of Texas at Austin, Austin, TX 78712, USA} \and Rachel Ward\footrecall{MATH} \footrecall{ODEN}}
\date{}                                           

\maketitle

\begin{abstract}
    The low dimensional manifold hypothesis posits that the data found in many applications, such as those involving natural images, lie (approximately) on low dimensional manifolds embedded in a high dimensional Euclidean space. In this setting, a typical neural network defines a function that takes a finite number of vectors in the embedding space as input. However, one often needs to consider evaluating the optimized network at points outside the training distribution. This paper considers the case in which the training data is distributed in a linear subspace of $\mathbb R^d$. We derive estimates on the variation of the learning function, defined by a neural network, in the direction transversal to the subspace. We study the potential regularization effects associated with the network's depth and noise in the codimension of the data manifold. We also present additional side effects in training due to the presence of noise. 
\end{abstract}

%%%%%%%%%%%%%%%%%%%%%%%%%%%%%%
%%%%%%%%%%Sections%%%%%%%%%%%%%%%
%%%%%%%%%%%%%%%%%%%%%%%%%%%%%%
%%%%%%%%%%%%%%%%%%%%%%%%%%%%%%
%%%%%%%%%%Sections%%%%%%%%%%%%%%%
%%%%%%%%%%%%%%%%%%%%%%%%%%%%%%

\section{Introduction}
In many machine learning problems, one observes that 
data points typically concentrate on a lower dimensional manifold  embedded in $\mathbb{R}^d$. Indeed, the {low dimensional manifold hypothesis}~\cite{tenenbaum2000global,hein2006manifold,niyogi2008finding,narayanan2010sample,fefferman2016testing} posits that the data found in many applications, such as those involving natural images, lie (approximately) on low dimensional manifolds which are embedded in high dimensional coding spaces. 
Manifold learning algorithms~\cite{tenenbaum2000global,roweis2000nonlinear,saul2003think,donoho2003hessian,belkin2003laplacian,weinberger2004learning} aim at finding low dimensional representations of the high dimensional data. 
There are many supervised or unsupervised linear dimensionality reduction methods. We mention Linear Discriminant Analysis (LDA)~\cite{balakrishnama1998linear}, Principal Component Analysis (PCA)~\cite{abdi2010principal}, Multiple Dimensional Scaling (MDS)~\cite{cox2008multidimensional}, and Canonical Correlation Analysis (CCA)~\cite{hardoon2004canonical}. 
The random projection framework for data compression provides a theoretical framework  for justification  ~\cite{johnson1984extensions,bourgain2011explicit,krahmer2011new}.
Nevertheless, even after a suitable dimension reduction, 
it is common to find that the data still concentrate on some lower dimensional manifold embedded in a higher dimensional Euclidean space.
This is at odds with the typical (and crucial) assumption found in many supervised machine learning theories: that the labeled data points are drawn i.i.d. from a probability distribution whose support has full measure in the embedding space \cite{bishop2006pattern}.

In this paper, we will assume that the data points are sampled from a linear subspace $\mathcal{M}$ of $\mathbb{R}^d$ and 
take the form $(\bm {x},g(\bm {x}))\in\mathbb{R}^d \times \mathbb R$, where $\bm {x}\in\mathcal{M}$, ${\rm dim}(\mathcal{M})<d$, and $g:\mathcal{M}\mapsto \mathbb{R}$ is a smooth function. 
The data points are used to identify a function $f_{\theta^*}: \mathbb{R}^d\mapsto\mathbb{R}$ from a parameterized family of functions $f_\theta$  defined by particular neural network architecture.  The ``trained" function $f_{\theta^*}$ is constructed by optimizing the network's parameters $\theta$ to fit the given data.
The approximation properties of neural networks for functions defined on embedded low-dimensional manifolds are studied  in~\cite{shaham2018provable,chui2018deep,chen2019efficient,schmidt2019deep,cloninger2021deep,liu2021besov}.
However, due to the presence of noise, the limitation to the training data acquisition, or distribution shift in the data that occurs post-training,  
one often needs to evaluate $f_{\theta^*}$ on points in a manifold $\mathcal M^\prime$ which is close to but not identical to $\mathcal{M}.$
As such, the behavior of the trained neural network $f_{\theta^*}$ on  $\mathcal{M}^\prime$ is a nontrivial but practically important question.
Not surprisingly, the performance of the trained network $f_{\theta^*}$ off of the data manifold $\mathcal{M}$ is more consistent the less that $f_{\theta^*}$ varies in the normal direction of $\mathcal{M}$. This becomes a question of estimating the magnitude of 
$\frac{\partial f_{\theta^*}}{\partial n_{\mathcal M}}$, with $n_{\mathcal M}$ denoting a normal direction of $\mathcal M$.
These observations motivate the following questions: Can $\frac{\partial f_{\theta^*}}{\partial n_{\mathcal M}}$ be regulated by choice of neural network architecture and optimization method? In which ways can noisy training data improve the stability performance of learning a neural network with low dimensional data? How does the low dimensional structure of the data manifold affect the stability of the performance of the trained neural network when applied to points away from the data manifold?   

We will analyze 
the training process of $f_\theta$ and the properties of $\frac{\partial f_{\theta^*}}{\partial n_{\mathcal M}}$ for deep linear neural networks or a nonlinear networks activated by ReLU.
We aim to reveal the effect of the arbitrariness of ambient space on the optimized neural networks.
We wll also discuss the approach of introducing noise to the non-label components of training data for reducing the effect of this "arbitrariness", i.e., for the regulation of $\frac{\partial f_{\theta^*}}{\partial n_{\mathcal M}}$.
{In many applications Principal Component Analysis can be used to reveal the low dimensional aspects of the data set. In those cases, the data sets can be described as samples from distributions with specific variances from a sequence of linear subspaces in a Euclidean ambient space. The analysis in this paper is highly relevant.}

The main contributions of this paper are listed below:
\begin{enumerate}
	\item If the data points, including noise, lie on $\mathcal M$, the linear network's depth may provide certain implicit regularization or side effects as shown in Figure~\ref{fig:L5L100} and Theorem~\ref{thm:LLN_wy_sigma=0}. For ReLU neural networks, Theorem~\ref{thm:stabilityL=2}, Theorem~\ref{thm:stabilityL>2}, and Corollary~\ref{coro:stabilityL>2} show that $\frac{\partial f_{\theta^*}}{\partial n_{\mathcal M}}$ is sensitive to the initialization of a set of ``untrainable" parameters.
	\item If the noise has a small positive variance in the orthogonal complement of $\mathcal M$, then:
	\begin{itemize}
		\item $\frac{\partial f_{\theta^*}}{\partial n_{\mathcal M}}$ can be made arbitrarily small, provided that the number of data points scales according to some inverse power of the variance as shown in Theorem~\ref{thm:w*scale} for deep linear neural networks and Figure~\ref{fig:DNN_DN_L} for deep nonlinear neural networks. From our experiments, the scaling laws for nonlinear ReLU networks is significantly different from the linear networks --- much more data points are needed to control the size of $\frac{\partial f_{\theta^*}}{\partial n_{\mathcal M}}$;
		\item  We show that gradient descent algorithms can be {very inefficient. The time needed for the gradient descent dynamics to reach a small neighborhood of the optimal parameters is reciprocal of the data set's variance in the normal space of $\mathcal M$.  See Theorem~\ref{thm:eigenvaluesF}.} 
		In addition, it may also need a long time to escape the near region of origin as shown in Theorem~\ref{thm:escapetime-multiD}.
	\end{itemize}
	\item The stability-accuracy trade-off. The role of noise can be interpreted as a stabilizer for a model when evaluated on points outside of the (clean) data distribution. {The regularization effect is equivalent to changing the loss function for learning functions defined in the ambient space.} However, adding noise to the data set will impact of the accuracy of the network's generalization error (for evaluation within the data distribution). For nonlinear data manifolds, uniform noise may render the labeled data incompatible.
\end{enumerate}

In the remainder of this section, we define the basic setting that we will work with and discuss the linear regression problem under this settings to motivate the rest of the paper.  
In Section~\ref{sec:dynamicLNN}, we present some special challenges in training deep linear neural networks via gradient descent. These challenges arise from embedding of data in a higher dimensional space. 
We will derive estimates for stability for linear networks in Section~\ref{sec:dynamicLNN} and nonlinear networks activated by ReLU in Section~\ref{sec:ReLUDNN}.
In Section \ref{sec:trade-off}, we briefly discuss the regularization of $\frac{\partial f_{\theta^*}}{\partial n_\mathcal{M}}$ by adding noise to data globally and the stability-accuracy trade-off.
In Section~\ref{sec:conclusion}, we give a final summary.

\subsection{The basic setting}\label{subsec:data}
Let $\mathcal M$ be a lower dimensional subspace of $\mathbb{R}^d$  defined as follows
\begin{equation*}
	\mathcal {M} = \left\{ \bm x = Q \begin{pmatrix}
		x \\ 0
	\end{pmatrix} \in \mathbb R^d : x\in\mathbb{R}^{d_x} \right\}
\end{equation*}
with $Q$ representing a unitary matrix, here and throughout.
Consider the distribution of points in $\mathbb R^d$ following
\begin{equation*}
	{M}_\sigma:= Q \begin{pmatrix}
		{X} \\ \sigma {Y}
	\end{pmatrix},
\end{equation*}
where $\sigma\ge0$, $Q\in\mathbb{R}^{d\times d}$ is a unitary matrix, and $X\in \mathbb R^{d_x}$ is a random vector representing the underlying distribution of data and $Y\in \mathbb R^{d_y}$ is a random vector independent from $X$. $Y$ is assumed to sample either the normal distribution ${N}(0,I_{d_y})$ or the uniform distribution $ U\left([-1,1]^{d_y}\right)$. 
$Y$ represents the noise model in the dimensions normal to $\mathcal{M}$.
In particular, $\bm x\in \mathcal M$ if $\bm x$ is sampled from $M_0$.
Finally, we consider labeled training data of the form  
\begin{equation}\label{data}
	D_N := \{ (\bm x_i, g_i)\}_{i=1}^N,~~~\bm x_i \sim M_\sigma, g_i\in \mathbb{R},
\end{equation}
where $\bm x_i \in \mathbb R^{d}$ is of the form 
\begin{equation}\label{data_components}
	\bm x_i 
	= Q \begin{pmatrix}
		x_i \\ \sigma y_i
	\end{pmatrix} \in \mathbb R^{d_x+d_y}, \quad \sigma\ge 0,
\end{equation}
with $x_i \sim X$, $y_i \sim Y$, and $ d = d_x + d_y$.
We further assume that
\begin{equation}\label{eq:data-full-rank}
	\text{rank}\left(\sum_{i=1}^N x_ix_i^T\right)=d_x,
\end{equation}
or equivalently, that the matrix $(x_1|x_2|\cdots|x_N)$ has full rank. This means that the data does samples every subspace of $\mathcal{M}$.

A crucial assumption in our paper is that the target function only depends on $x_i$, i.e., there exists a function $g: \mathbb R^{d_x} \mapsto \mathbb R$ such that
\begin{equation*}
	g_i = g(x_i) \in \mathbb R.
\end{equation*}
However, we point out that the typical learning model and training algorithms are agnostic to this assumption. As a result, we design our machine learning model $f_\theta: \mathbb R^{d} \mapsto \mathbb R$ rather than  $\mathbb R^{d_x} \mapsto \mathbb R$. 

A typical machine learning model with parameter set $\theta\in\mathbb{R}^p$ is used to define a function
$$f_\theta(\cdot)=f(\cdot;\theta):\mathbb{R}^d \mapsto \mathbb{R}.$$
In particular, we study the case of $f(\bm x; \theta)$ being a deep neural network
\begin{equation}\label{eq:def_NN}
	\begin{cases}
		f^{\ell}(\bm x) &= W^\ell \alpha(f^{\ell-1}(\bm x) ) + b^{\ell}, \quad \ell = 2:L, \\
		f (\bm x; \theta) &= f^{L}(\bm x),
	\end{cases}
\end{equation}
where $f^1(\bm x) = W^1\bm x + b^1$, $W^\ell \in \mathbb R^{n_{\ell}\times n_{\ell-1}}$, and $b^\ell, f^\ell \in\mathbb{R}^{n_{\ell}}$ with $n_0=d$ and $n_L=1$. {Here, $\theta = \{(W^\ell, b^\ell)\}_{\ell=1}^L$ denotes the set of all parameters in the deep neural network $f(\bm x; \theta)$.}
In the following, we will focus on two different networks: 
\begin{enumerate}
	\item linear networks: 
	\begin{equation}\label{eq:def_linear}
		\alpha(x) = x \quad \text{and} \quad b^\ell \equiv 0;
	\end{equation}
	\item ReLU-activated neural networks: 
	\begin{equation}\label{eq:def_relu}
		\alpha(x) = {\rm ReLU}(x) := \max\{0,x\}.
	\end{equation}
\end{enumerate}

A trained function
$f_{\theta^*}$ is constructed by gradient descent applied to the optimization problem
\begin{equation}\label{eq:optimization_model_J}
	\min_{\theta\in\mathbb{R}^p} J(\theta), \quad \quad J(\theta) = \frac{1}{2N}\sum_{i=1}^N|f_\theta(\bm x_i) - g_i|^2.
\end{equation}
More precisely, $\theta$ is updated by first initializing as $\theta^{0}$ and then updating 
\begin{equation}\label{eq:training-process}
	\theta^{t+1} = \theta^{t} - \eta_t \frac{\partial J(\theta^t)}{\partial \theta}
\end{equation}
with some $\eta_t > 0$ for $t\ge0$. {In this paper, we shall refer to this updating scheme as (full) gradient descent (FGD).
	We will also discuss the
	typical stochastic gradient descent (SGD) update, where $J$ and 
	$\frac{\partial J}{\partial \theta}$ are replaced respectively by $J_{B_t}$ and $\frac{\partial J_{B_t}}{\partial \theta}$, and 
	\begin{equation*}
		J_{B_t}(\theta) = \sum_{\bm x_i\in B_t}|f_\theta(\bm x_i) - g_i|^2,
	\end{equation*}
	where $B_t\subsetneq\{\bm x_1,\cdots, \bm x_N\}$ is randomly chosen and called a mini-batch.}

Let
\begin{equation*}
	\mathcal P_\mathcal{M} \bm x := Q\left(\begin{array}{cc}
		I_{d_{x}} & 0\\
		0 & 0
	\end{array}\right)Q^T\bm x, 
\end{equation*}
where $I_{d_x}$ is the  $d_x\times d_x$ identity matrix,
and define $\overline{g}:\mathbb{R}^d\mapsto \mathbb{R}$ as 
\begin{equation}\label{eq:g_bar}
	\overline{g}(\bm x) = g\left(\mathcal P_\mathcal{M} \bm x\right).
\end{equation}
$\mathcal P_\mathcal{M}$ is the orthogonal projection onto $\mathcal{M}$, and $\overline{g}(\bm x)$ is the extension of $g(x)$ that stays constant in the directions orthogonal to $\mathcal{M}.$ 
Correspondingly, we define $\overline{f}_\theta$ as the restriction of $f_\theta$ on $\mathcal M$:
\begin{equation*}
	\overline{f}_\theta(\bm x) = f_\theta(\mathcal P_{\mathcal M} \bm x).
\end{equation*}
Now consider $\mathcal{M}^\prime$, which is close to but not necessarily identical to $\mathcal{M}.$ 
We can estimate the error:
\begin{equation}\label{eq:error_decomposition}
	\left|f_{\theta^*}(\bm x)-\overline{g}(\bm x)\right| \le \left|f_{\theta^*}(\bm x) - \overline{f}_{\theta^*}(\bm x)\right| + \left|\overline{f}_{\theta^*}(\bm x)-\overline{g}(\bm x)\right|,~~~\bm x\in\mathcal{M}^\prime,
\end{equation}
where $f_{\theta^*}$ is learned from $\mathcal{M}_{\sigma}$ (clean data for $\sigma=0$ or noisy data for $\sigma>0$). 
The first term on the right-hand-side can be interpreted as the stability error of the learned neural network $f_{\theta^*}(\bm x)$. It measures the amount $f_{\theta^*}(\bm x)$ varies along the normal direction of the subspace $\mathcal M$.
In particular, we have
\begin{equation}\label{eq:general_error}
	\left|f_{\theta^*}(\bm x) - \overline{f}_{\theta^*}(\bm x)\right| \le \left\| \frac{\partial f_{\theta^*}}{\partial n_{\mathcal M}}\right\| \left\|\bm x - \mathcal P_{\mathcal M}\bm x\right\|, \quad \bm x \in \mathcal M'.
\end{equation}
The term $\left\|\bm x - \mathcal P_{\mathcal M}\bm x\right\|$ is controlled by the difference between the data subspace $\mathcal M$ and the test set in $\mathcal M'$.
The second term on the right-hand-side of \eqref{eq:error_decomposition} corresponds to the approximation ability of the neural network. An approximation theory of neural networks for functions of the form $\overline{g}(\bm x)=g(\mathcal P_{\mathcal M} \bm x)$ is established in \cite{cloninger2021deep}, where $\mathcal M$ is a general manifold and $\mathcal{P}_{\mathcal M} {\bm x} = \mathop{\arg\inf}_{\bm \xi \in \mathcal M} \| {\bm x} - \bm \xi\|$ defines the orthogonal projection onto a general manifold $\mathcal M$. In other words, in \cite{cloninger2021deep} the data is assumed to be sampled from $\left(\bm x, \overline{g}(\bm x)\right)$, where $\bm x \in \mathcal A \subset [0,1]^d$ and $\mathcal A$ is assumed to be contained in a tubular region around $\mathcal M$. Provided that the tubular region is has a radius smaller than the reach of $\mathcal{M}$, $\frac{\partial f_{\theta^*}}{\partial n_{\mathcal M}}$ of an optimal network would be 0 in the tubular region.

We remark that in the typical machine learning setup, one considers data sampled from the same manifold, which corresponds to $\mathcal{M}^\prime\equiv\mathcal{M}$. In comparison, we are interested in deriving bounds for ``out of distribution" error or a kind of stability metric. {Thus, we shall focus on \eqref{eq:general_error}, the right-hand-side of \eqref{eq:error_decomposition}, 
	and assume that the second term can be bounded appropriately.}

In this paper, the empirical means of quantities derived from the data will often play a role. We adopt the following notation: 
\begin{notation}\label{notation:1}
	Let $z$ be a random variable in $\mathbb{R}^m$ or $\mathbb R^{m\times n}$ over some probability space and let $z_i$ denote a sample realization of $z$. We denote the empirical average
	\begin{equation*}
		\left< z\right>_N := \frac{1}{N}\sum_{i=1}^N z_i
	\end{equation*}
	and the mean
	\begin{equation*}
		\left<z \right> :=\lim_{N\rightarrow\infty} \left< z\right>_N = \mathbb{E}[z].
	\end{equation*}
\end{notation}
\begin{notation}\label{notation:2}
	For vectors $(x_i,y_i)\in \mathbb{R}^{d_x} \times \mathbb R^{d_y}$, $i=1,2,\cdots, N,$ we denote the averaged correlation matrix by
	\begin{equation*}
		\left<A(x,y)\right>_N := \begin{pmatrix}
			\left<xx^T\right>_N & \left<xy^T\right>_N \\
			\left<yx^T\right>_N & \left<yy^T\right>_N
		\end{pmatrix}.
	\end{equation*}
	Unless explicitly stated otherwise, we will refer to $\left<A(x,y)\right>_N$ as $\left<A\right>_N$, and $\left<A(x,\sigma y)\right>_N$ as $\left<A_\sigma\right>_N.$
\end{notation}

\subsection{Warm up: linear regression}

As a special case of linear neural networks, we first use simple linear regression to demonstrate how $\frac{\partial f_{\theta^*}}{\partial n_{\mathcal M}}$ can be affected by 
the data and the model. 
Since $Q$ can be factored into parameters,
without loss of generality, we will assume that $Q\equiv I$.

For linear regression, $f_\theta$, with $\theta \equiv\bm w\in \mathbb{R}^d$, 
takes the form 
\begin{equation}\label{eq:linearregression}
	f(\bm x; \bm w) = {\bm w}^T \bm x = w_x^T x + w_y^T y,
\end{equation}
where $w_x \in \mathbb R^{d_x}$ and $w_y \in \mathbb R^{d_y}$. We solve
\begin{equation}\label{eq:regression-loss}
	\min_{\bm w\in \mathbb{R}^d} \frac{1}{2N} \sum_{i=1}^N \left( \bm w^T \bm x_i - g_i\right)^2,
\end{equation}
where $\bm x_i \sim M_\sigma$.

If $\sigma=0$, in which case $y \equiv 0$ equivalently, the loss defined in \eqref{eq:regression-loss} reduces to 
\[
J(\bm w) = \frac{1}{2N} \sum_{i=1}^N \left(  w_x^T x_i + w_y^T~0- g_i\right)^2.
\]
Every point in the set $\{(w_x^*, w_y) ~|~ w_y \in \mathbb R^{d_y}, w_x^* = \left<xx^T\right>_N^{-1}\left< gx\right>_N\}$ is a minimizer. However, if gradient descent is used for the minimization, the ``optimal" model takes the form 
\begin{equation*}
	f(\bm x; \bm w^*) = (\bm w^{*})^T \bm x = (w_x^*)^T x + (w_y^{(0)})^T y, 
\end{equation*}
where $w_y^{(0)}$ is the initial value set for the gradient descent
since $\frac{\partial J(\bm w)}{\partial w_y} = 0$.
Hence, we have
\begin{equation*}
	\frac{\partial f_{\theta^*}}{\partial n_{\mathcal M}} = \frac{\partial f(\bm x; \bm w^*)}{\partial y} = w_y^{(0)},
\end{equation*}
where $w_y^{(0)}$ keeps its initialization value. 
This means $\frac{\partial f_{\theta^*}}{\partial n_{\mathcal M}}$ is determined by  the initialization of $w_y$ and does not change during the training process.

In the case $\sigma\neq 0$ and $d_x=d_y=1$, there is a unique minimizer $(w_x^*,w_y^*)$ that can be quickly derived:
\begin{equation*}
	w_x^* = \frac{\left<gx\right>_N \left<y^2\right>_N-\left<gy\right>_N\left<xy\right>_N}{\left<x^2\right>_N\left<y^2\right>_N-\left<xy\right>_N^2},~~~
	w_y^* = \frac{1}{\sigma}\frac{\left<gy\right>_N\left<x^2\right>_N-\left<gx\right>_N \left<xy\right>_N}{\left<x^2\right>_N\left<y^2\right>_N-\left<xy\right>_N^2}.
\end{equation*}
In addition, if we assume that the distribution of $x_i$ and $y_i$ are independent and $\mathbb{E}[xy] = 0$,
then we will have $\left<xy\right>_N \sim\mathcal{O}(1/\sqrt{N})$, $\left<x^2\right>_N = \left<y^2\right>_N \sim\mathcal{O}(1)$,
$\left<xg\right>_N \sim\mathcal{O}(1)$, and $\left<yg\right>_N \sim\mathcal{O}(1/\sqrt{N})$. This leads to the following estimates
\[w_x^* = \frac{\left< gx\right>_N}{\left<x^2\right>_N}+\mathcal{O}(\frac{1}{{N}})\]
and
\[w_y^* = \frac{1}{\sigma\sqrt{N}}\frac{\left< x^2\right>_N - \left< xg\right>_N}{\left<x^2\right>_N\left<y^2\right>_N - \mathcal{O}(1/N)} \sim \mathcal O \left( \frac{1}{\sigma \sqrt{N}}\right).\]

To have $w_y^*\sim \mathcal{O}(1)$ as $\sigma\rightarrow 0$, one needs to take $N$ to infinity according to
\begin{equation}\label{eq:linreg_scalse}
	N\sim \mathcal O(\sigma^{-2}).
\end{equation}
In other words, the resulting linear function will have a small normal derivative only if the number of data points scales {super linearly} inversely with the variance of the noise in the co-dimensions of $\mathcal{M}$.

The linear regression example reveals an important aspect about learning from embedded low dimensional data that is persistent in more general settings. $\frac{\partial f_{\theta^*}}{\partial n_{\mathcal M}}$ 
depends on the set of parameters which are not trainable when there is no noise. The smaller $\frac{\partial f_{\theta^*}}{\partial n_{\mathcal M}}$ is, the more stable the network is for evaluation at points out of training data distribution. 
In the presence of noise with small variance in the codimension directions, the number of training examples needs to scale inversely proportional to the variance.

\section{Linear neural networks}\label{sec:dynamicLNN}
In this section, we study learning with deep linear multi-layer neural networks, in particular the gradient descent dynamics for minimizing the mean squared error.
Regression with multiple-hidden layer linear networks generalize simple linear regression models. The training of linear neural networks provides a way to construct linear operators satisfying certain structural constraints~\cite{arora2019implicit,kohn2021geometry}. Consequently, LNN models can be adapted to improve the performance of classic methods, for example in wave propagation~\cite{nguyen2021numerical} and linear convolutional neural networks in multigrid~\cite{he2019mgnet,chen2020meta,hsieh2019learning}.

As defined in \eqref{eq:def_NN} and \eqref{eq:def_linear} we have the linear network with $L-1$ hidden layers as
\begin{equation}\label{eq:LNNsModel}
	f(\bm x;\theta) = W^{L}W^{L-1}\cdots W^{2}W^1 \bm x = \bm w^T \bm x,
\end{equation}
where $\theta = (W^1, W^2, \cdots, W^L)$ denotes all parameter matrices in this model and the end-to-end parameter $\bm w = W^{L}W^{L-1}\cdots W^{2}W^1$ is defined as the product of the $W^k$ matrices.
Here, $W^k \in \mathbb{R}^{n_{k}\times n_{k-1}}$ are the weights connecting the $(k-1)$-th and the $k$-th layer, $k=1,2,\cdots,L,$ with the convention that the $0$-th layer is the input layer ($n_0 = d$) and $L$-th layer is the output layer ($n_L = 1$). In particular, we consider only the fixed-width case, i.e., $n_k = n \ge d$ for all $k=1,2,\cdots,L-1$.  We will refer to such networks as LNNs.

We denote the loss function in terms of $(W^1,  \cdots, W^L)$ as
\begin{equation}\label{eq:LW}
	J(W^1, \cdots, W^L) =  \frac{1}{2N}\sum_{i=1}^N |W^{L}W^{L-1}\cdots W^{2}W^1 \bm x_i -g_i |^2,
\end{equation}
and in terms of the end-to-end parameters $\bm w$ as 
\begin{equation}\label{eq:Lew}
	J^e(\bm w) = \frac{1}{2N}\sum_{i=1}^N (\bm w^T \bm x_i -g_i )^2,
\end{equation}
where $\bm x_i \sim M_\sigma$.
Here, the superscript $e$ in $J^e$ emphasizes the fact that  $J^e$ is the corresponding loss function for the end-to-end  weight set $\bm w$.

In \cite{arora2018optimization}, Arora et. al. proposed to minimize $J(W^1, W^2, \cdots, W^L)$ in terms of $(W^1,  \cdots, W^L)$, and derived that gradient descent of $J$ via the explicit stepping
\begin{equation*}
	W^\ell \leftarrow W^\ell - \eta \frac{\partial J}{\partial W^\ell},~~~\ell=1,2,\cdots,k,
\end{equation*}
leads to the following dynamical system for $\bm w$ in the limit of $\eta \rightarrow 0$:
\begin{equation}\label{eq:We}
	\frac{d}{dt} \bm w = -\|\bm w\|^{2-\frac{2}{L}} \left( \nabla_{\bm w} J^e(\bm w)  + (L-1) \mathcal{P}_{\bm w}\left( \nabla_{\bm w} J^e(\bm w)\right)\right),
\end{equation}
under the assumptions for the initialization of $(W^1, \cdots, W^L)$ that
\begin{equation}\label{eq:init_We}
	\left(W^{\ell+1}\right)^TW^{\ell+1} = W^{\ell}\left(W^{\ell}\right)^T
\end{equation}
for all $\ell=1:L-1$. Here $\mathcal{P}_{\bm w}(\cdot)$ denotes the operator that projects vectors onto the subspace spanned by $\bm w$:
\begin{equation*}
	\mathcal{P}_{\bm w}(\bm v) = \frac{\bm w \bm w^T}{\|\bm w\|^2} \bm v.
\end{equation*}
For convenience, we define the vector field $\bm {F}:\mathbb{R}^d\mapsto \mathbb{R}^d$ as
\begin{equation}\label{eq:F}
	\bm {F}(\bm {w}):=-\|\bm w\|^{2-\frac{2}{L}} \left( \nabla_{\bm w} J^e(\bm w)  + (L-1) \mathcal{P}_{\bm w}\left( \nabla_{\bm w} J^e(\bm w)\right)\right).
\end{equation}

\paragraph{Prior works related to LNNs with full-rank data.}
Early work on LNNs focused more on the side-effects of introducing more hidden layers. For example, the $\ell^2$ regression with two hidden linear layers was studied in \cite{fukumizu1998dynamics}. In that paper, the author studied the training process and demonstrated the existence of overtraining under the so-called over-realizable cases by employing the exact solution for a matrix Riccati equation.
A simplified nonlinear dynamical system was introduced in \cite{saxe2013exact} to show that increasing depth in linear neural networks may slow down the training.
However, it was proven in \cite{kawaguchi2016deep} that every local minimum is a global minimum for over-parameterized LNNs (width $n$ is larger than the number of data $N$).
It is shown recently in \cite{arora2018optimization} that involving more linear layers beyond the simplest linear regression brings some advantages to the training of networks and possibly to the network's generalization performance.
It is also reported in \cite{arora2018optimization} that \eqref{eq:We} yields an accelerated convergence of $\bm w$ compared to the linear regression case. 
Recently, the convergence of gradient flows related to learning deep LNNs was further studied in~\cite{bah2021learning,maxime2021convergence} by re-interpreting them as Riemannian gradient flows on the manifold of rank-$r$ matrices endowed with a suitable Riemannian metric.
It is worth stressing again that all these convergence results are established based on the assumption that $\left<\bm x \bm x^T\right>_N$ is full rank.

In the remainder of this section, we aim at analyzing \eqref{eq:We} in the context of embedded low dimensional data.

\subsection{Gradient descent for deep linear neural networks}
In this subsection, we first study some general properties of the dynamical system \eqref{eq:We}. Then, we provide some further
results if we involve the low-dimensional assumption of data.  We first point out that the dynamical system \eqref{eq:We} is invariant under unitary transformation: 
\begin{proposition}\label{prop:rota_inva}
	Suppose that the data $\{(\bm x_i, g_i)\}_{i=1}^N$ follows $
	\bm x_i 
	= Q \begin{pmatrix}
		x_i \\ \sigma y_i
	\end{pmatrix}
	\sim \mathcal M_\sigma$ for some unitary transform $Q$ on $\mathbb R^{d}$.  Denote $\widetilde{\bm x}_i = Q^T \bm x_i$ and $\widetilde{\bm w} = Q^T \bm w$. 
	If $\bm w(t)$ satisfies \eqref{eq:We} then  $\widetilde{\bm w}(t)$ also satisfies \eqref{eq:We}, and vice versa.
\end{proposition}
Thus, without loss of generality, we can focus on the case of $Q = I_d$, that is, 
${\mathcal{M}} = {\rm Span}\{e_1, e_2, \cdots, e_{d_x}\}$.
In this setup, 
\begin{equation}\label{eq:L_e-def}
	J^e(\bm w) = \frac{1}{2N} \sum_{i=1}^N (w_x^Tx_i + w_y^T \sigma y_i - g_i)^2.
\end{equation}
Next, we derive the gradient of the loss function $J^e$:
\begin{equation}\label{eq:nabla_L}
	\nabla_{\bm w} J^e(\bm w) = 
	\left< A_\sigma\right>_N  {\bm w} - \left<g \bm x\right>_N,
\end{equation}
where $\left<A_\sigma\right>_N$ is defined in Notation~\ref{notation:2} 
and $\left<g \bm x\right>_N = 
\begin{pmatrix}
	\left<gx\right>_N \\ \sigma \left<g y \right>_N
\end{pmatrix}$ 
by definition in Notation~\ref{notation:1}. Here we notice the relation between
$\left< A_\sigma\right>_N$ and $\left< A\right>_N$
\begin{equation}\label{eq:AG}
	\left< A_\sigma\right>_N = \begin{pmatrix}
		I_{d_x} & 0 \\
		0 & \sigma I_{d_y}
	\end{pmatrix} 
	\left<A\right>_N
	\begin{pmatrix}
		I_{d_x} & 0 \\
		0 & \sigma I_{d_y} 
	\end{pmatrix},
\end{equation}
which is useful in the following analysis.

Then, we summarize some observations about the stationary points of \eqref{eq:We}.
\begin{proposition}\label{prop:w-stationary}
	The stationary points of the dynamical system \eqref{eq:We} consist of point in the set 
	\begin{equation*}
		\{\bm{F}=\mathbf{0}\}\equiv\{ \bm w ~:~ \nabla J^e(\bm w) = \mathbf{0}~\textrm{or}~\bm w=\mathbf{0} \},
	\end{equation*}
	where $\bm{F}$ is defined in \eqref{eq:F}.
	Furthermore, if $L=2$, $\bm F(\bm w)$ is not
	differentiable at $\bm 0$; if $L>2$, the Jacobian matrix $\nabla \bm F(\bm 0) = 0$. 
\end{proposition}

\begin{proposition}\label{prop:critical-points-of-Le}
	Assume that $\left<xx^T\right>_N$ and $\left<A\right>_N$ are invertible.
	\begin{enumerate}
		\item If $\sigma=0$, 
		\begin{equation}
			\{ \bm w~:~ \nabla J^e(\bm w) = 0\} = \left\{(w_x^*, w_y) ~:~ w_y \in \mathbb R^{d_y} \right\},
		\end{equation}
		where $w_x^* = \left< xx^T\right>_N^{-1} \left<gx\right>_N$.
		\item If $\sigma \neq 0$,
		\begin{equation}\label{eq:w*}
			\bm w^* = 
			\begin{pmatrix}
				w^*_x \\ w^*_y 
			\end{pmatrix} = 
			\begin{pmatrix}
				\alpha^* \\ \sigma^{-1} \beta^* 
			\end{pmatrix} = 
			\begin{pmatrix}
				I_{d_x\times d_x} & 0 \\
				0 & \sigma^{-1}I_{d_y} 
			\end{pmatrix} \left<A\right>_N^{-1} \begin{pmatrix}
				\left<gx\right>_N \\\left<g y \right>_N
			\end{pmatrix}
		\end{equation}
		is the unique critical point for $\nabla J^e(\bm w) $.
		Furthermore, we have $w_x^* = \alpha^*$ and 
		\begin{equation}
			\begin{pmatrix}
				\alpha^* \\ \beta^*
			\end{pmatrix}
			= \left<A\right>_N^{-1} \begin{pmatrix}
				\left<gx\right>_N \\\left<g y \right>_N
			\end{pmatrix}
		\end{equation}
		which is independent from $\sigma$ in data. 
	\end{enumerate}
\end{proposition}
We remark that the assumption made in \eqref{eq:data-full-rank} implies that $\left<xx^T\right>_N$ is invertible.

\begin{assumption}\label{as:X-and-expectation}
	In $ M_\sigma = \begin{pmatrix}
		X \\ \sigma Y
	\end{pmatrix}$, $X$ and $Y$ are two independent random vectors where $Y \equiv N(0,I_{d_y})$ and $X$ is a random vector in $\mathbb{R}^{d_x}$ such that $\mathbb{E}[XX^T]$ is invertible.
\end{assumption}

Analogous to the two dimensional linear regression problem, the following theorem relates $\|w_y^*\|$ to 
the standard deviation of the noise and the cardinality of the data set. 
\begin{theorem}\label{thm:w*scale}
	Suppose that $\sigma \neq 0$ and  $(x_i, y_i), i=1:N$ are independently sampled from the distributions $X$ and $Y$ satisfying Assumption~\ref{as:X-and-expectation}. Let $(w_x^*, w_y^*)$ denote a stationary point of \eqref{eq:We}. 
	For sufficiently large $N$, with a high probability,
	$$
	\|w_y^*\| \le\frac{C_{g, X,Y}}{\sigma\sqrt{N}},
	$$ and for some $C_{g,X,Y}\ge 0$ which depends only on $g(x)$ and the distribution $(X,Y)$. 
\end{theorem}
\begin{proof}
	Let us denote
	$$
	\left<xx^T\right>_N = \widetilde \Sigma_X \quad \text{and} \quad
	\left<yy^T\right>_N = \widetilde \Sigma_Y,
	$$
	which are the maximum likelihood estimations of the covariance matrices $\Sigma_X$ and $\Sigma_Y = I_{d_y}$.
	Given $\Sigma_X$ is invertible and $N$ is large enough, we have $\left<A\right>_N$ and $S$ are all invertible by matrix perturbation theory~\cite{stewart1990matrix}.
	Moreover, we have
	$$
	w_y^* = \sigma^{-1}S^{-1}\left(\left<gy\right>_N - \left<yx^T\right>_N \widetilde \Sigma_X^{-1}\left<gx\right>_N \right),
	$$
	by representing $A^{-1}$ in~\eqref{eq:w*} in terms of block matrix where
	$$
	S = \widetilde \Sigma_Y - \left<yx^T\right>_N\widetilde \Sigma_X^{-1}\left<xy^T\right>_N.
	$$
	According to the independence of $X$ and $Y$ and the law of large numbers, we have
	$$
	\left[\left<xy^T\right>_N \right]_{ij} = \mathcal O\left(\frac{1}{\sqrt{N}}\right) \quad \text{and} \quad
	\left[\left<yx^T\right>_N \right]_{ji} = \mathcal O\left(\frac{1}{\sqrt{N}}\right),
	$$
	and
	$$
	\left[\left<gy\right>_N \right]_{j} = \mathcal O\left(\frac{1}{\sqrt{N}}\right),
	$$
	for $i = 1:d_x$ and $j=1:d_y$.
	In addition, similar results for correlated matrix~\cite{adamczak2010quantitative,cai2010optimal} show that
	$$
	\widetilde \Sigma_X = \Sigma_X + \mathcal O\left(\frac{1}{\sqrt{N}}\right) \quad
	\text{and}\quad\widetilde \Sigma_Y = I_{d_y} + \mathcal O\left(\frac{1}{\sqrt{N}}\right)
	$$ with high probability if $N$ is large.
	Furthermore, we have
	$$
	\left<A\right>_N \equiv \begin{pmatrix}
		\left<xx^T\right>_N & \left<xy^T\right>_N \\
		\left<yx^T\right>_N & \left<yy^T\right>_N
	\end{pmatrix}
	= \begin{pmatrix}
		\Sigma_X & 0 \\
		0 & I_{d_y}
	\end{pmatrix} + \mathcal O\left(\frac{1}{\sqrt{N}}\right),
	$$
	and notice
	$$
	S = \widetilde \Sigma_Y - \left<yx^T\right>_N\widetilde \Sigma_X^{-1}\left<xy^T\right>_N = I_{d_y} + \mathcal O\left(\frac{1}{\sqrt{N}}\right) + \mathcal O\left(\frac{1}{N}\right).
	$$
	This means
	$$
	\|S^{-1}\| \le C_Y\left(1-\mathcal O\left(\frac{1}{\sqrt{N}}\right)\right)^{-1} ~\text{ and }~ \|\widetilde \Sigma_X^{-1}\| \le C_X\left(\|\Sigma_X\|_{\rm min}-\mathcal O(\frac{1}{\sqrt{N}})\right)^{-1},
	$$
	where $\|\Sigma_X\|_{\rm min}$ denotes the minimal singular value of $\Sigma_X$ and $C_X$ and $C_Y$ are constants depended only on $X$ and $Y$.
	Thus, for some $C_{g,X,Y}\ge 0,$ we have
	$$
	\|w_y^*\| \le 
	\sigma^{-1}\|S^{-1}\|\left(\|\left<gy\right>_N\| + \|\left<yx^T\right>_N\| \left\|\widetilde \Sigma_X^{-1} \right\| \left\|\left<gx\right>_N \right\|\right) \le \frac{C_{g,X,Y}}{\sigma \sqrt{N}}.
	$$ 
\end{proof}

Finally, we have the following estimate for $w_y^*$ when the target function $g(x) = \tilde g(x) + \mu^T x$ is a perturbation of a  linear function $\mu\in \mathbb{R}^{d_x}$. 
\begin{corollary}\label{coro:Cg}
	If $g(x) = \tilde g(x) + \mu^T x$ and $|\tilde g(x)|\le \delta$ for all $x\in \mathbb{R}^{d_x}$, then 
	$$
	\|w_x^* - \mu\| \le b_{X,Y} \delta \quad \text{and} \quad \|w_y^*\| \le \frac{{\delta}~C_{X,Y}}{\sigma \sqrt{N}},
	$$
	for some constants $b_{X,Y}$ and $C_{X,Y}$ depending only on the distribution $X$ and $Y$. Furthermore,
	$$
	\left| g(x) - (\bm w^*)^T \bm x \right| \le \delta\left(1+b_{X,Y}\|x\| + \frac{C_{X,Y}\|y\|}{\sigma \sqrt{N}}\right),
	$$
	for any $\bm x = (x,y) \in \mathbb R^d$.
\end{corollary}

The following numerical results in Figure~\ref{fig:Wy_func} verify the estimate of $\|w_y^*\|$ in Theorem~\ref{thm:w*scale} and the claim in Corollary~\ref{coro:Cg}. Here $d_x=2$ ($x=(x_1,x_2)$), $d_y=1$, and we take $g_0(x_1, x_2) = \pi (\sin(\pi x_1) + \sin(\pi x_2))$ in the left figure. For the right figure, we have $g_1(x_1.x_2) = 4(x_1 + x_2) + 0.1(\sin(\pi x_1) + \sin(\pi x_2))$, $g_2(x_1.x_2) = 2(x_1 + x_2) + 0.1(\sin(\pi x_1) + \sin(\pi x_2))$, $g_3(x_1.x_2) = \pi(\sin(\pi x_1) + \sin(\pi x_2))$, and $N=10^{6}$. We sample the data as $x_1, x_2 \sim  U[-1,1]$ and $y \sim  N(0,1)$ and then compute $(w_x^*, w_y^*)$ by averaging 10 results using \eqref{eq:w*}.
\begin{figure}[h]
	\centering
	\includegraphics[scale=0.8]{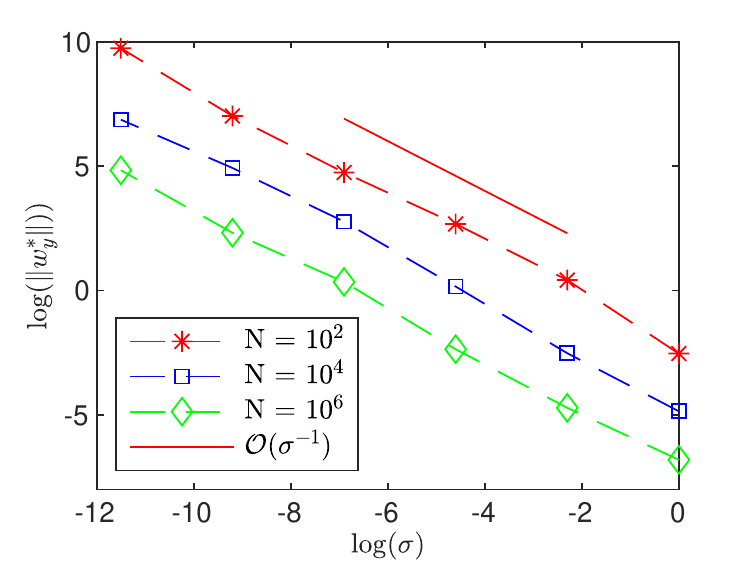}\includegraphics[scale=0.8]{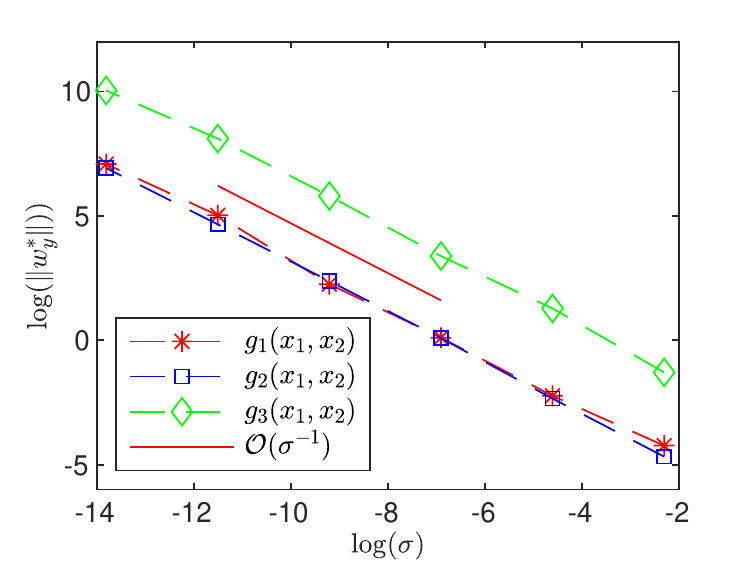}
	\caption{The log-log diagram of $\|w_y^*\|$ (left) with different $\sigma$ and $N$ and $\|w_y^*\|$ with different $g(x)$ (right).}
	\label{fig:Wy_func}
\end{figure}

\subsection{Bifurcation and slow manifold when $\sigma$ is small}
In Proposition~\eqref{prop:critical-points-of-Le}, we showed that when $\sigma=0$, the dynamical system \eqref{eq:We} has a stationary manifold defined as
\begin{equation}
	\Gamma_0 :=\left\{(w_x^*, w_y) ~:~ w_y \in \mathbb R^{d_y} \right\}.
\end{equation}
For small positive $\sigma$, $\Gamma_{\sigma}$ degenerates into a single point $(w_x^*, w_y^*)$ denoted as the slow manifold $\Gamma_\sigma$. In this section, we present a phase plane analysis of \eqref{eq:We} and relate the consequence in training a deep LNN. 

In Figure~\ref{fig:gl}, we present the phase portrait of the dynamical system~\eqref{eq:We} {on the $w_xw_y$-plane}.
We see that $w_x(t)$ first converges to a neighborhood of $\Gamma_\sigma$. Once in the neighborhood, $w_y(t)$ converges to $w_y^*$ on a slower time scale. 
Asymptotically, $(w_x(t), w_y(t))$ converges to the stationary point $(w_x^*, w_y^*)$. 
Indeed, the following Theorem confirms that  $\Gamma_0$ and $\Gamma_\sigma$ are stable.
\begin{figure}[h]
	\centering
	\includegraphics[width=0.6\textwidth]{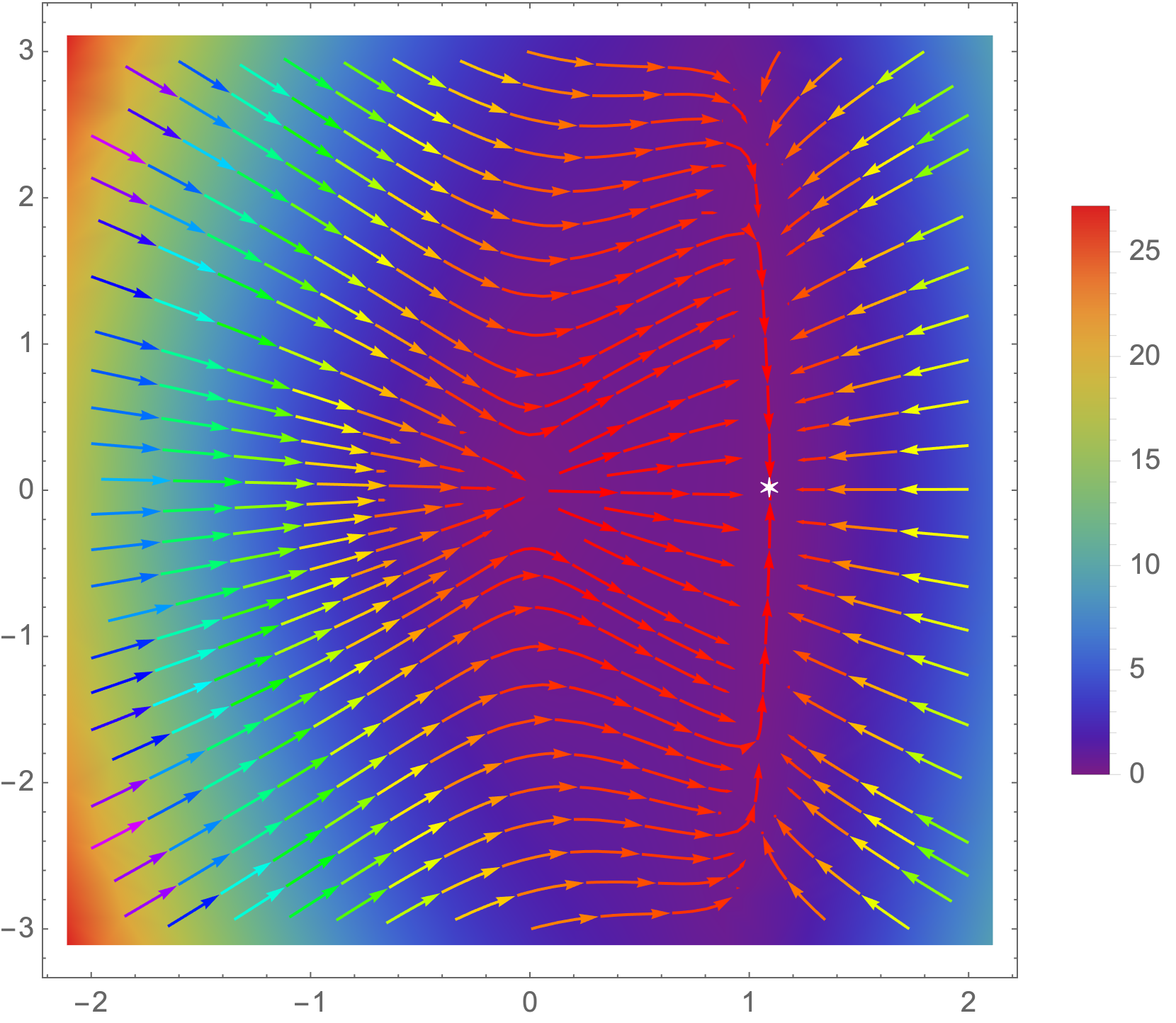}
	\caption{Stream lines of the system with $\sigma = 0.1$ {on the $w_xw_y$-plane. The horizontal and the vertical axes are respectively  the $w_x$- and the $w_y$-axis. We take $L=5$, $\sigma = 0.1$, $N = 10^{4}$, and $g(x) = x + 0.1\sin(\pi x)$. 
			As shown in Corollary~\ref{coro:Cg}, we have $w_x^* \approx 1$ and $w_y^* \approx 6.596\times 10^{-3} \le \frac{C_{x,y}}{10\sigma \sqrt{N}}$.} The color in the background corresponds to the value of $\|\bm F\|$. }
	\label{fig:gl}
\end{figure}

\begin{theorem}\label{thm:eigenvaluesF}
	Suppose that $x_i, y_i, i=1,2,\cdots,N$ are independently sampled from distributions $X$ and $Y$ satisfying Assumption~\ref{as:X-and-expectation}.
	Consider the vector field $\bm {F}$ defined in \eqref{eq:F}. 
	\begin{itemize}
		\item If $\sigma = 0 $, then the eigenvalues of $\nabla \bm F(\bm w^*)$ are non-positive and the associated eigenvectors to the zero eigenvalues are $\{(0,w_y) | w_y \in \mathbb R^{d_y}\} = \Gamma_0 - (w_x^*,0)$ for any $\bm w^* \in \Gamma_0$.
		\item If $\sigma > 0$, $\frac{1}{\sqrt{N}} \ll \sigma$, and $\bm w^*$ is the unique non-zero stationary point, then there are $d_y$ negative eigenvalues of $\nabla \bm F(\bm w^*)$ with scale $\mathcal O(\sigma^2)$  with  high probability.
	\end{itemize}
\end{theorem}
\begin{proof}
	If $\sigma = 0$ and $\bm w^* \in \Gamma_0$, first we have 
	the eigenvalues of $\nabla \bm F(\bm w^*)$ are non-positive 
	as shown in Proposition~\ref{prop:critical-points-of-Le}.
	Moreover, we have
	$$
	\nabla \bm F(\bm w^*) = -\|\bm w^*\|^{-\frac{2}{L}}\bm M(\bm w^*) \nabla^2J^e(\bm w^*),
	$$
	where
	$$
	\bm M(\bm w^*) = \|\bm w^*\|^2 I + (L-1)\bm w^*(\bm w^*)^T.
	$$
	Recall $\nabla^2J^e(\bm w^*) = \left<A_0\right>_N =  
	\begin{pmatrix}
		\left<xx^T\right>_N &  0 \\
		0 & 0
	\end{pmatrix}$ and $(\bm w^*)^T \left<A_0\right>_N = \left<g\bm x\right>_N^T = 
	\begin{pmatrix}
		\left<gx\right>_N \\
		0
	\end{pmatrix}$, 
	thus it follows that
	$$
	\begin{aligned}
		\bm M(\bm w^*) \nabla^2 J^e(\bm w^*) &= \|\bm w^*\|^2 \left<A_0\right>_N + (L-1) \bm w^* \left<g \bm x\right>_N^T \\
		&= 
		\begin{pmatrix}
			\|\bm w^*\|^2 \left<xx^T\right>_N + (L-1)w_x^* \left<gx\right>_N^T & 0 \\
			(L-1)w_y^* \left<gx\right>_N^T & 0
		\end{pmatrix}.
	\end{aligned}
	$$
	Thus, the eigenvectors of $\nabla \bm F(\bm w^*)$ corresponding to zero eigenvalues belong to $\Gamma_0 - (w_x^*,0)$ since $\nabla \bm F(\bm w^*)$ has the form $\begin{pmatrix}
		\ast & 0\\ \ast & 0
	\end{pmatrix}$.
	
	If $\sigma > 0$ and $\bm w^* \in \Gamma_\sigma$, we still have
	$$
	\nabla \bm F(\bm w^*) = -\|\bm w^*\|^{-\frac{2}{L}}\bm M(\bm w^*) \nabla^2J^e(\bm w^*)
	$$
	and 
	$$
	\bm M(\bm w^*) \nabla^2 J^e(\bm w^*) = \|\bm w^*\|^2 \left<A_\sigma\right>_N + (L-1) \bm w^* \left<g \bm x\right>_N^T.
	$$
	In addition, we have
	$$
	\left<A_\sigma\right>_N = 
	\begin{pmatrix}
		\left<xx^T\right>_N & \sigma\left<xy^T\right>_N \\
		\sigma\left<yx^T\right>_N & \sigma^2\left<yy^T\right>_N
	\end{pmatrix} = 
	\begin{pmatrix}
		\Sigma_X & 0 \\
		0 & \sigma^2I_{d_y}
	\end{pmatrix} + \mathcal O\left({\frac{\sigma}{\sqrt{N}}} \right).
	$$
	Furthermore, we notice
	$$
	\bm w^* \left<g \bm x\right>_N^T = 
	\begin{pmatrix}
		w_x^* \left< gx\right>_N^T & \sigma w_x^* \left<gy\right>^T_N \\
		w_y^* \left<gx\right>^T_N & \sigma w_y^* \left<gy\right>^T_N
	\end{pmatrix} = 
	\begin{pmatrix}
		w_x^* \left< gx\right>_N^T & 0 \\
		w_y^* \left<gx\right>^T_N & 0
	\end{pmatrix} + \mathcal O\left({\frac{\sigma}{\sqrt{N}}} \right).
	$$
	It follows that
	$$
	\begin{aligned}
		\bm M(\bm w^*) \nabla^2 J^e(\bm w^*) &= 
		\begin{pmatrix}
			\|\bm w^*\|^2\Sigma_X + (L-1)w_x^* \left< gx\right>_N^T & 0 \\
			(L-1)w_y^* \left<gx\right>^T_N & \|\bm w^*\|^2\sigma^2I_{d_y}
		\end{pmatrix} + \mathcal O\left({\frac{\sigma}{\sqrt{N}}} \right) \\
		&=: \bm K + \mathcal O\left({\frac{\sigma}{\sqrt{N}}} \right).
	\end{aligned}
	$$
	Here, we notice that there are $d_y$ eigenvalues of $\bm K $ equals $\|\bm w^*\|^2\sigma^2$ with eigenspace $\Gamma_0 - (w_x^*,0)$. Given the matrix perturbation theory~\cite{stewart1990matrix}, there exist at least $d_y$ negative eigenvalues of $\nabla \bm F(\bm w^*)$ with scale $\mathcal O(\sigma^2)$ if $\frac{\sigma}{\sqrt{N}} \ll \sigma^2$. 
\end{proof}

{In the regime $0<\sigma \ll 1$ and $N\gg \sigma^{-2}$, the
	gradient descent flow \eqref{eq:We} tend to converge slowly to the optimal parameter $\bm w^*$ due to the gap in the eigenvalues of $\nabla \bm F(\bm w^*)$, as Theorem~\ref{thm:eigenvaluesF} shows. 
	We refer to this slow convergence as one of the side effects of learning from embedded data because it stems from the fact that data distribution essentially concentrates on a lower dimensional manifold.}  

In Figure~\ref{fig:T_0_Epsilon}, we present a set of numerical simulations demonstrating this slow convergence when $N$ (the number of data points) is sufficiently large.
In the experiment, $d_x=2$ and $d_y=1$, so $\Gamma_\sigma$ is a point on the line $\{(w_x^*, w_y):w_y\in \mathbb{R}\}.$ In the left subplot, we report the smallest eigenvalue of $\nabla \bm F$, corresponding to the direction parallel to $\Gamma_\sigma$, for different $N$ and $\sigma$. 
In the right subplot, we report the quantities
$$
[T_\sigma]_i := \inf_{t} \left\{ t ~:~ \left\| [\bm w_\sigma]_i(t) - [\bm w_\sigma^*]_i \right\| \le 10^{-6}\right\},
$$
where $[\bm w_\sigma]_i(t)$ stands for the $i$-th component of $\bm w_\sigma$ at the time $t$ and $\bm w_\sigma^*$ is
the non-zero stationary point as in Proposition~\ref{prop:critical-points-of-Le}. 
$[T_\sigma]_i$ gives the first time that the $i$-th component of $\bm w_\sigma$ becomes within $10^{-6}$ distance to  $[\bm w_\sigma^*]_i$.
{We now focus on the convergence of the third component, corresponding to $w_y$. Assuming that $\bm w_\sigma(0)$ is in a sufficiently close neighborhood of  $\bm w_\sigma^*$ so that linear theory applies. We then have $\|[\bm w_\sigma]_3(t) - [\bm w^*_\sigma]_3\| \le e^{-C\lambda_y(\nabla \bm F)t}$ which means $[T_\sigma]_3 \sim C\sigma^{-2}$. This indicates that 
	the time to reach within a small distance of $\bm w^*$ is proportional to $1/\sigma^2$.
	Numerical results in Figure~\ref{fig:T_0_Epsilon} verifies the slow convergence phenomenon.
	Here, $\bm w_\sigma$ is computed by simulating the system \eqref{eq:We} directly with {\it ode45} in MATLAB with time step size $5\times 10^{-3}$.  
	Correspondingly, it takes $200\times e^{[T_\sigma]_i}$ iterations in {\it ode45} such that the $i$-th component of $\bm w_\sigma$ becomes within $10^{-6}$ distance to  $[\bm w_\sigma^*]_i$.}
\begin{figure}[h]
	\centering
	\includegraphics[scale=0.8]{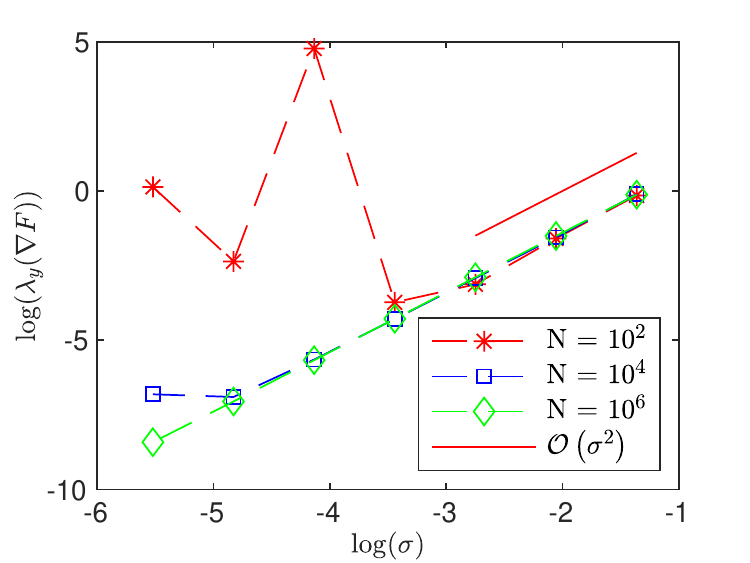}\includegraphics[scale=0.8]{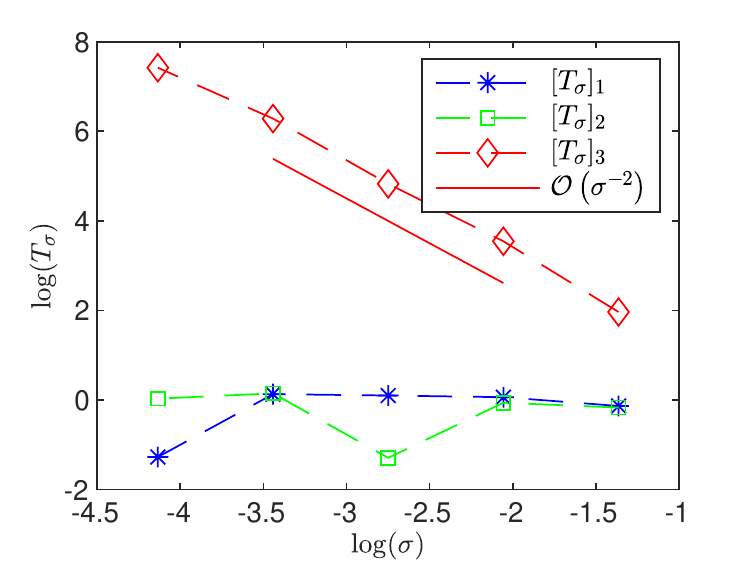}
	\caption{The log-log diagram of $\lambda_y(\nabla \bm F)$ (left) and $[T_\sigma]_i$ for $i=1:3$ (right), where $g(x_1, x_2) = \pi*(\sin(\pi x_1)+\sin(\pi x_2))$. {We observe that indeed the smallest eigenvalue of $\nabla \bm F$ follows the scale of $\sigma^{2}$ when $\frac{1}{\sqrt{N}} \ll \sigma$ while $[T_\sigma]_3$ follows the scale of $\mathcal{O}(\sigma^{-2})$, which confirms that the convergence of the slowest component takes place in the $\mathcal{O}(\sigma^{-2})$ time scale.}}
	\label{fig:T_0_Epsilon}
\end{figure}

{The following Proposition shows that a similar gap in the eigenvalues may exit
	even for systems defined with relatively small number of data points.}
{
	\begin{proposition}\label{prop:eigenF_SGD}
		Under that same conditions in Theorem~\ref{thm:eigenvaluesF} with $0< \sigma \ll 1$, for any $N\ge1$ and $\left\|\|\bm w^*\|^2\left<xy^T\right>_N + (L-1)w_x^* \left<gy\right>_N^T\right\|\le C$,
		denoting $\lambda(\cdot)$ as the spectrum of a matrix and $\nabla \bm F(\bm w^*) = -\|\bm w^*\|^{-\frac{2}{L}} 
		\begin{pmatrix}
			F_{11} & F_{12} \\
			F_{21} & F_{22}
		\end{pmatrix}$,
		where 
		\begin{scriptsize}
			\begin{equation*}
				\begin{aligned}
					&F_{11} = \|\bm w^*\|^2 \left<xx^T\right>_N + (L-1)w_x^* \left< gx\right>_N^T, &&F_{12} = \sigma \left(\|\bm w^*\|^2\left<xy^T\right>_N +  (L-1)w_x^* \left<gy\right>^T_N \right), \\
					&F_{21} = (L-1)w_y^* \left<gx\right>^T_N + \sigma\|\bm w^*\|^2\left<yx^T\right>_N, &&F_{22} = \|\bm w^*\|^2\sigma^2 \left<yy^T\right>_N + \sigma w_y^*(L-1) \left<gy\right>^T_N,
				\end{aligned}
			\end{equation*}
		\end{scriptsize}
		then $\lambda(\nabla \bm F(\bm w^*)) \subset G_1 \cup G_2$, where
		$$
		G_i = \lambda(F_{ii}) \cup \left\{ \lambda \notin \lambda(F_{ii})  \left| \left\|(F_{ii} - \lambda I)^{-1}\right\|^{-1} \le \|F_{ji}\| \right.\right\}, i=1:2, j\neq i.
		$$
		More precisely, for $i=2$, we have 
		$$
		G_2 = \lambda(F_{22}) \cup \left\{ \lambda \notin \lambda(F_{22})  \left| \left\|(F_{22} - \lambda I)^{-1}\right\|^{-1} \le \|F_{12}\| \right.\right\}.
		$$
		In particular, since $ \|F_{12}\| \le \sigma \left\| \|\bm w^*\|^2\left<xy^T\right>_N +  (L-1)w_x^* \left<gy\right>^T_N \right\| \le \sigma C$ and $\lambda(F_{22})\sim \mathcal O\left(\sigma\right)$, it follows that $\lambda \sim \mathcal O(\sigma)$ for any $\lambda \in G_{2}$. 
	\end{proposition}
}
{The proof of this theorem is a quick application of Gershgorin’s Theorem for block matrices~\cite{tretter2008spectral}.
	Following this proposition, $\nabla \bm F(\bm w^*)$ may have eigenvalues falling in the set $G_2$. In that case, the magnitudes of those eigenvalues are $\mathcal{O}(\sigma).$
	Hence, the proposition can be applied to understand the flow in a mini-batch stochastic gradient descent algorithm.
	Each step of SGD can be understood as one discrete step of \eqref{eq:We} with a relatively small $N$ corresponding to the mini-batch size. 
	Thus, this proposition suggests that employing SGD in training can be more efficient, as the eigenvalues of the smallest amplitude scale as $\mathcal{O}(\sigma)$ instead of $\mathcal{O}(\sigma^2)$ (if $N\gg \sigma^{-2}$), although it will not always avoid the slow convergence caused by the small variance $\sigma$ in the $y$-directions. See Figure~\ref{fig:LNN_GD} for a supporting numerical study. 
}

\subsection{Slow  convergence}\label{sec:escape_time}
In this subsection, 
we show that deep LNNs may have yet another hindrance to convergence, depending on the initialization.
The following theorem shows that the trajectories of \eqref{eq:We} may be attracted to a neighborhood of the origin, and if that happens, it will take a very long time to escape.
\begin{theorem}\label{thm:escapetime-multiD}
	Assume $ 0 < C_1 \le \left<A_\sigma\right>_N \le C_2$ and $\|\left<g\bm x\right>_N\| = \mathcal O(1)$, then for $\epsilon \ll 1$ we have 
	\begin{equation}\label{eq:TL}
		T_{L}(\epsilon) := \inf \left\{ t ~:~ \|\bm w(0)\| = \epsilon,~ \|\bm w(t) - \bm w(0)\| \ge \frac{\epsilon}{2} \right\} \ge C \epsilon^{\frac{2}{L}-1},
	\end{equation}
	where $\bm w(t)$ is solution of \eqref{eq:We} and $C$ depends on
	$L$, $\left<A_\sigma\right>_N$, and $\left<g\bm x\right>_N$.
\end{theorem}
For brevity, this theorem shows that deeper LNNs requires more time for
convergence if the initialization is very close to the origin or the training process reaches the near field of the origin.
In practice, a commonly accepted heuristics is to avoid
initializing weights near the origin. The above theorem provides an theoretical interpretation for that heuristics, at least in the context of training deep linear networks. 
However, as shown in Figures~\ref{fig:gl} and \ref{fig:LNN_GD}, even if one initializes the weights to be far from the origin, 
the weights can be attracted to a neighborhood of the origin during the gradient flow. This phenomenon, which has not been discovered before, can still cause the slow convergence in training LNNs. 

Figure~\ref{fig:LNN_pass_0_to_Gamma} demonstrate the convergence issues corresponding to Theorem~\ref{thm:eigenvaluesF} and Theorem~\ref{thm:escapetime-multiD}. 
Here, we simulate the dynamical system~\eqref{eq:We}, with $L=10$ and  $\bm w(0) = (-2,1)$.
The data are sampled as follows: $\bm x_i = (x_i, \sigma y_i) \in \mathbb{R}^2$, $x_i\sim  U[-1,1]$, $y_i\sim {N}(0,1)$, $g(x)=\pi\sin(\pi x)$, $N=10^4$,  and $\sigma=0.05$.
\begin{figure}[h]
	\centering
	\includegraphics[scale=0.8]{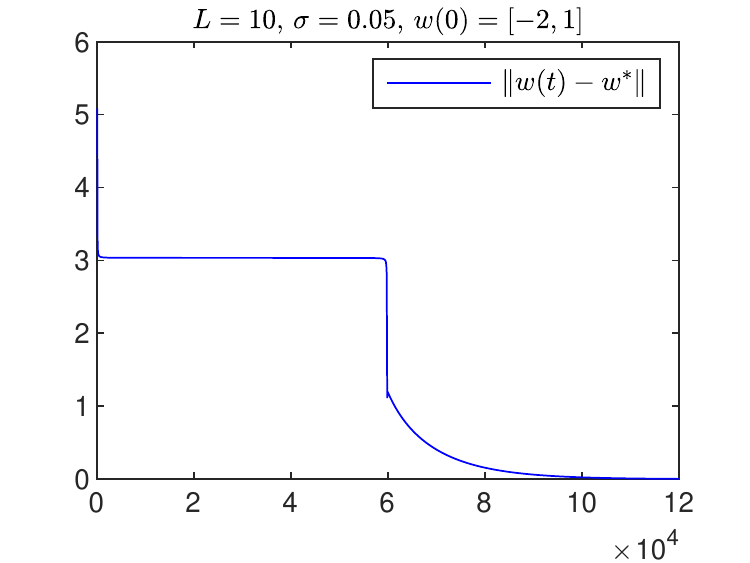}\includegraphics[scale=0.8]{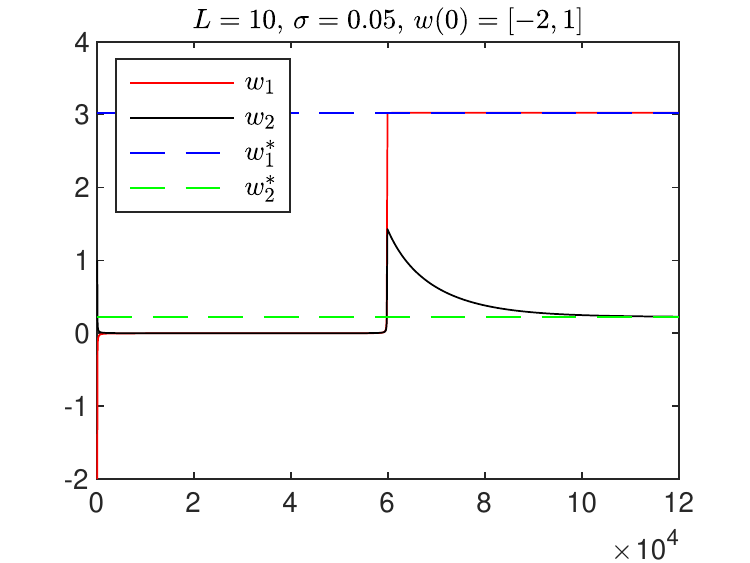}
	\caption{Convergence of $\bm {w}(t)=(w_x(t), w_y(t)).$ We see that the trajectory is attracted to the origin and stay a long time  ($\sim 6\times 10^{4}$ iterations) before escaping. Furthermore, once flowing pass the origin, $w_x(t)$ quickly converge to a small neighborhood of the slow manifold $\Gamma_\sigma$. But it will take another long period before $\bm {w}(t)$ gets close to the optimal point due to the slow convergence of $w_y(t).$ }
	\label{fig:LNN_pass_0_to_Gamma}
\end{figure}

Furthermore, in Figure~\ref{fig:LNN_GD}, we also observe similar results when we train a LNN with the full gradient descent method with a special initialization that $[W^\ell]_{i,j}$ is a fixed constant for each $i,j$ such that $\|\bm w\|=2^{-6}$. 
This initialization can satisfy the condition in \eqref{eq:init_We} as required in~\cite{arora2018optimization} to make the dynamic system~\eqref{eq:We} as the continuous limit of the FGD method. Thus, we take a full gradient descent training algorithm with a decreasing learning rate from $2.5\times 10^{-3}$ to $2.5\times10^{-5}$ under a cosine annealing schedule~\cite{loshchilov2016sgdr}. 
In addition, we take $L=6$ and $n=10$ for this LNN. The training data are created by taking $d_x = 3$ and $d_y=2$, $\bm x_i = (x_i, \sigma y_i)$, $x_i \sim  U\left([-1,1]^{d_x}\right)$, $y_i \sim  N(0,I_{d_y})$, and $g(x) = 2\sum_{i=1}^3[x]_i + 0.1\sum_{i=1}^3\sin(\pi [x]_i)$, $N=4\times 10^3$,  and $\sigma=0.05$.
Moreover, we are also interested in how SGD will perform under this situation. We apply SGD for the same LNN and training data with Kaiming's initialization~\cite{he2015delving} for $W^\ell$ and a mini-batch size $50$. We also show the results in Figure~\ref{fig:LNN_GD}.
\begin{figure}[h]
	\centering
	\includegraphics[scale=0.8]{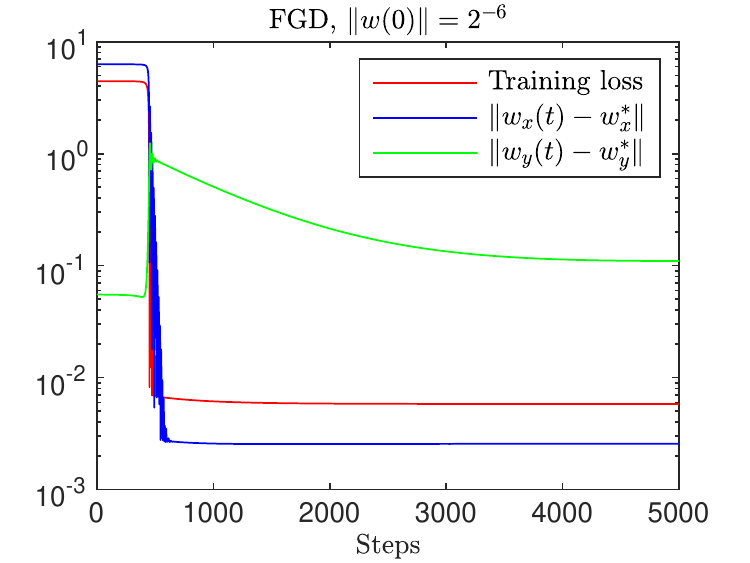}\includegraphics[scale=0.8]{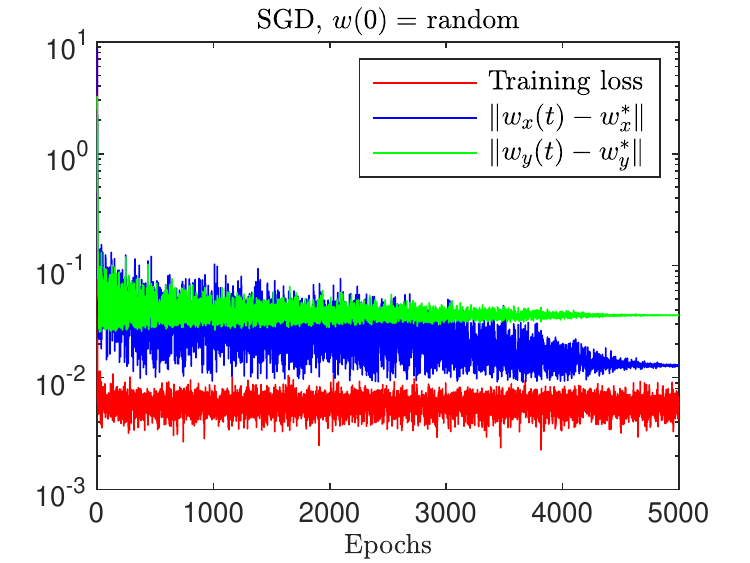}
	\caption{The convergence of the loss function, $\|w_x(t)-w_x^*\|$, and $\|w_y(t)-w_y^*\|$ {when $\sigma=0.05$}. The FGD results match the previous analysis very well. Given the initialization that $\|\bm w\| = 2^{-6}$, it gets stuck around the origin for a while and then $w_x$ converges very quick while $w_y$ converges very slowly after escaping the origin. {In the SGD results, a suitable random initialization strategy is important to the success of the SGD in DNNs~\cite{glorot2010understanding,he2015delving,chen2022weight}. Generally, it is hard to notice the trapping issue around the origin for SGD with random initialization. However, as shown here and in Proposition~\ref{prop:eigenF_SGD},   $w_y$ converges slowly for both FGD and SGD.}}
	\label{fig:LNN_GD}
\end{figure}

\paragraph{Related work.} Theorem~\ref{thm:eigenvaluesF} shows that \eqref{eq:We} has a slow manifold $\Gamma_\sigma$ and the  convergence of $w_y(t)$ to $w_y^*$ takes place in the $\mathcal O(\sigma^{-2})$ time scale. 
Similar results about the slow convergence (in the components corresponding to small singular values in the data matrix) are also reported in \cite{steinerberger2021randomized} for randomized Kaczmarz iterations and~\cite{hacohen2021principal} for gradient descent in neural networks. 
In the setting of this paper, if $\sigma \sqrt{N} \ll 1$ and $\tilde g$ is not small enough, then Corollary~\ref{coro:Cg} shows that $\|w_y^*\| \gg 1$. In this case,  ``early stopping'' \cite{yao2007early} may be employed to control $\|w_y(T)\|$. The similar results can also be found in~\cite{ma2020slow}, which presents that small eigenvalues for the associated Gram matrix makes the convergence of gradient descent very slow. In that case, the slow convergence gives us ample time to stop the training process and obtain solutions with good generalization property. 
On the other hand,  Corollary~\ref{coro:Cg} and Theorem~\ref{thm:eigenvaluesF} also indicate that there exist some cases in which the early stopping is not recommended. For example, $\|w_y^*\|$ could be small if $\sigma \sqrt{N} \gg 1$ and $\tilde g$ in Corollary~\ref{coro:Cg} is relatively small. 

\subsection{Regularization effects of noise and network's depth}
\subsubsection{Regularization effect of noise}\label{sec:effectnoise}
Theorem~\ref{thm:w*scale} states that the presence of noise in the $y$-components, i.e. $\sigma\neq 0$, can result in $w^*_y$ with a small amplitude, provided that the training data set is sufficiently large. 
Moreover, if the noise scale is fixed in data $\bm x_i$, Theorem~\ref{thm:w*scale} present that more data
are need to control the amplitude of $\|w_y^*\|$. Figure~\ref{fig:LNN_DN_L} demonstrates these results in training LNN models using SGD.

In Figure~\ref{fig:LNN_DN_L}, we notice that $w_y(t)$ is non-constant even when $\sigma =0$. It is due to the ``mixing" that come from the multiple hidden layers, and can be seen from \eqref{eq:We} (more explicitly from \eqref{eq:We2dx}). 
This is different from 
pure linear regression case where $w_y$ will keep constant after initialization. 
Given this observation, we will further study the properties of training LNNs when $\sigma=0$ in the next subsection. 

The basic setup is same to what we have done in Figure~\ref{fig:LNN_GD}. 
Noticing that $w_y$ in LNNs may be difficult to converge when $\sigma$ is small, we test only $\sigma=2^{k}$ for $k=0:-5$. Thus, we apply SGD only 500 epochs for these experiments and the reported values of $\|w_y\|$ are obtained by averaging over 5 individual tests.
\begin{figure}[h]
	\centering
	\includegraphics[scale=0.8]{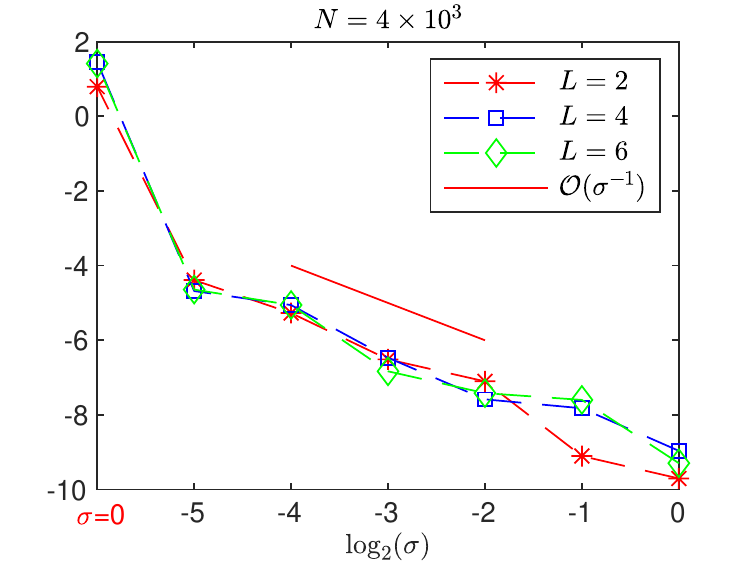}\includegraphics[scale=0.8]{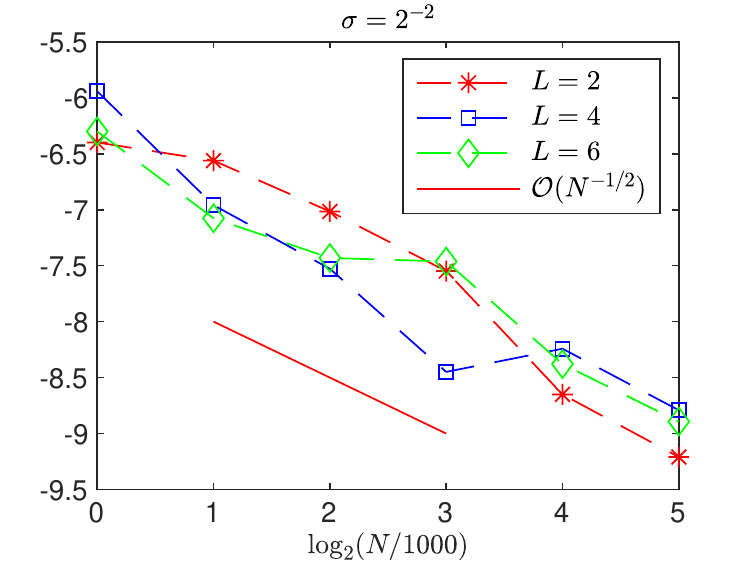}  
	\caption{$\log_2(\|w_y\|)$ of trained LNNs for $L=2,4,6$. 
		Here we still observe that  $\|w_y^*\| = \mathcal O\left(\sigma^{-1} N^{-1/2}\right)$.}
	\label{fig:LNN_DN_L}
\end{figure}

\subsubsection{The regularization and side effects of depth when $\sigma =0$}\label{sec:regular_sigma=0}
In this subsection, we focus on the setting where the training data lie on the
low dimensional manifold $\mathcal M$ exactly, i.e., $\bm x_i \sim M_0$. We prove that the size of $\|w_y^*\|$ trained with this data may decrease as the depth of the network increases, for the initial value $\bm w(0)$ in certain subregion of $\mathbb{R}^d$.

Since $\bm x_i \sim  M_0$, we have the data points $\bm x_i = (x_i, 0) \in \mathbb{R}^{d_x + d_y}$ and $g_i\in\mathbb{R}$.
Under this situation, the loss function will degenerate to
\begin{equation*}
	J^e(\bm w) = \frac{1}{2N} \sum_{i=1}^N (w_x^T x_i - g_i)^2,
\end{equation*}
where
\begin{equation*}
	\frac{\partial J^e(\bm w)}{\partial w_x} = \left<xx^T\right>_N w_x- \left<gx\right>_N \quad \text{and} \quad 
	\frac{\partial J^e(\bm w)}{\partial w_y} = 0.
\end{equation*}
Equations of $\bm w$ in \eqref{eq:We}  are reduced to 
\begin{small}
	\begin{equation}\label{eq:We2dx}
		\begin{cases}
			\frac{d}{dt}w_x  &=  f(w_x, w_y) = -\|\bm w\|^{-\frac{2}{L}}\left( \|\bm w\|^2\frac{\partial J^e(\bm w)}{\partial w_x}  + (L-1)\left( w_x^T\frac{\partial J^e(\bm w)}{\partial w_x} \right)w_x \right), \\
			\frac{d}{dt}w_y &=  g(w_x, w_y) = -(L-1)\|\bm w\|^{-\frac{2}{L}}\left(\left( w_x^T\frac{\partial J^e(\bm w)}{\partial w_x}   \right)w_y \right),
		\end{cases}
	\end{equation}
\end{small}
since $\bm w^T\frac{\partial J^e(\bm w)}{\partial \bm w} = w_x^T\frac{\partial J^e(\bm w)}{\partial w_x}$.

According to Proposition~\ref{prop:w-stationary}, the stationary points of the above system consist of $\bm 0$ and
\begin{equation*}
	\Gamma_0 =\left\{ (w_x^*, w_y)~:~ w_y \in \mathbb R^{d_y}\right\}
\end{equation*}
where we assume $w_x^*$ is the unique solution of $\frac{\partial J^e(\bm w)}{\partial w_x} = \left<xx^T\right>_N w_x- \left<gx\right>_N = 0$.

In the following, we study the relationship between $L$, the network's depth, and $\frac{\partial f}{\partial y} = w_y^*$. Naturally, the smaller the magnitude of $\frac{\partial f}{\partial y}$, the more consistent the network's output would be when the testing data deviates from the training data manifold.

To begin this study, we first show the following diagram about the phase portraits of the system~\eqref{eq:We2dx} {on the $w_xw_y$-plane} with $L=5$ and $L=100$.
\begin{figure}[h]
	\centering
	\includegraphics[width=.4\textwidth]{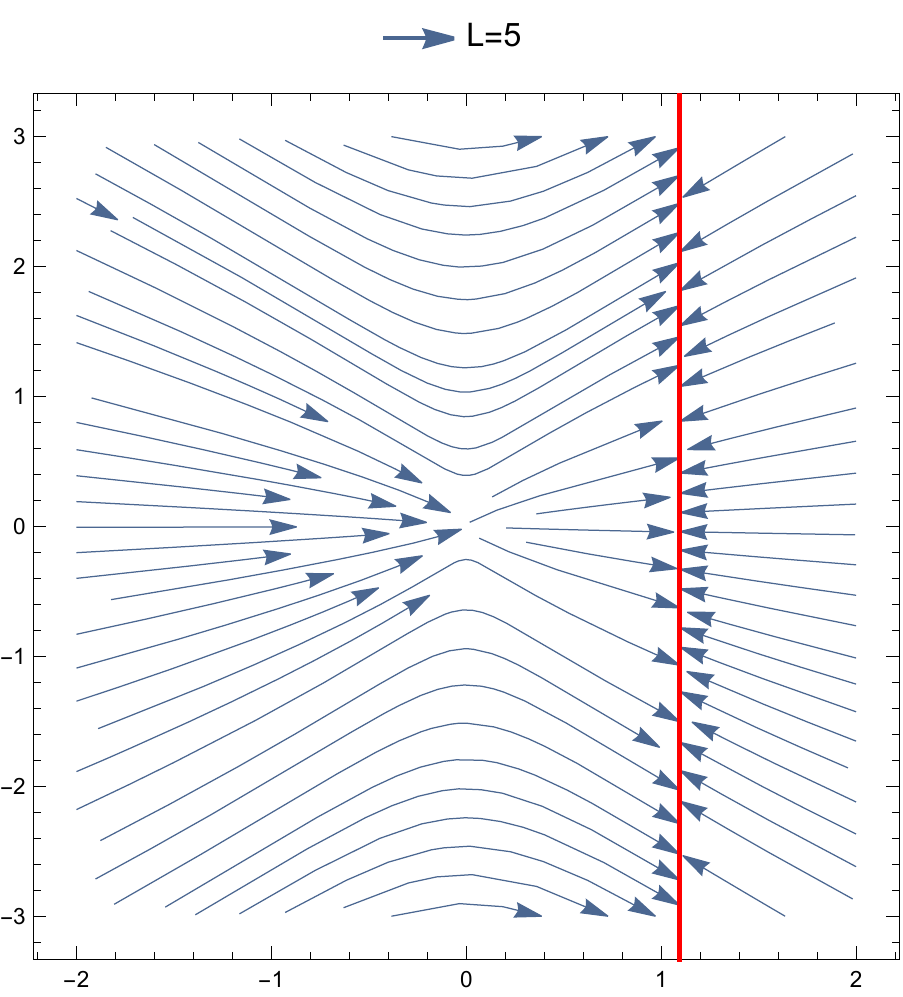}\quad\includegraphics[width=.4\textwidth]{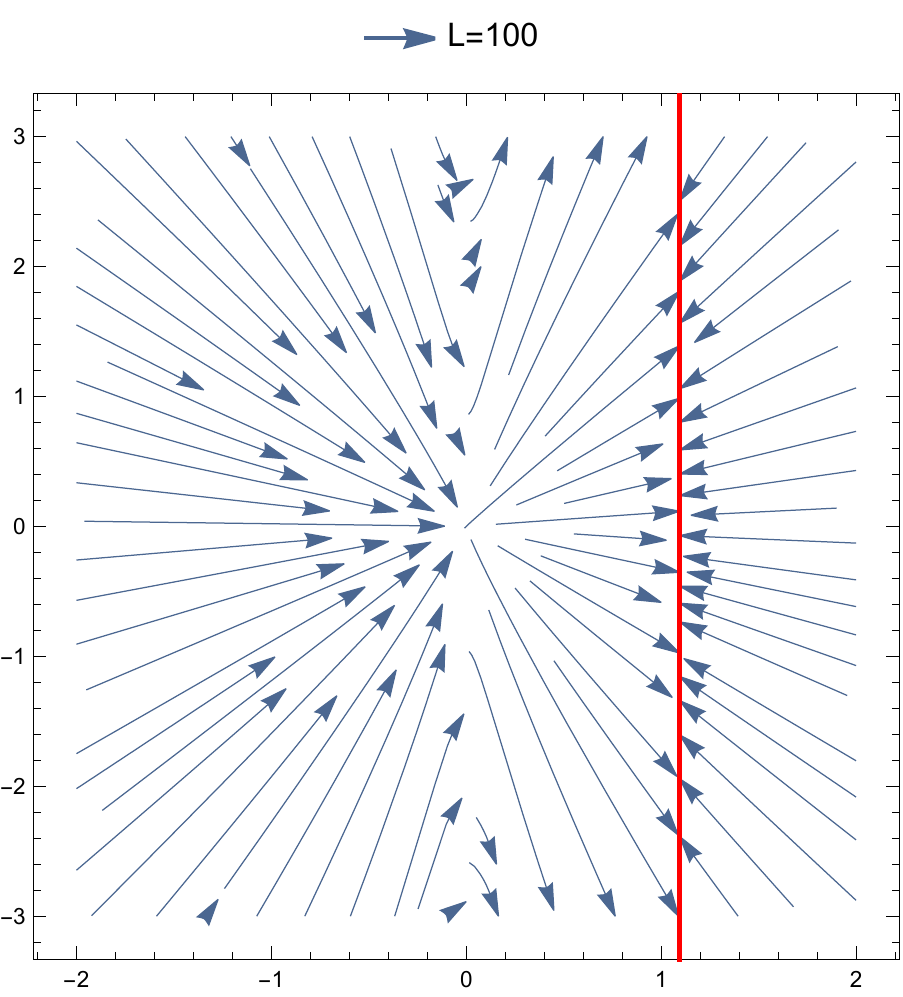} 
	\caption{Phase portraits of the system~\eqref{eq:We2dx} {on the $w_xw_y$-plane} with $L=5$ (left) and $L=100$ (right).}
	\label{fig:L5L100}
\end{figure}

According to the above phase portraits, if $\bm w$ is initialized on the right of $\Gamma_0$ (the red line in Figure~\ref{fig:L5L100}), we have $|w_y^*| \le |w_y(0)|$,
which can be understood as the regularization effect of the LNN structure since $w_y^* =w_y(0)$ in classical linear regression model when $\sigma = 0$. In addition, we also notice that $|w_y^*| \ge |w_y(0)|$ if $\bm w$ is initialized between the y-axis and $\Gamma_0$. This aspect of training can be interpreted as a side effect of the LNN structure comparing to the linear regression case. 
We now present generalization of this regularization and side effects.

Again, let $\bm w = (w_x, w_y)\in \mathbb R^{d_x+d_y}$ with $d_x,d_y \ge 1$.
First, we define
\begin{equation}\label{eq:Ex}
	E_x := \left\{ w_x \in \mathbb R^{d_x}  ~:~ w_x^T \frac{\partial J^e(\bm w)}{\partial w_x} = 0 \right\} \subset \mathbb R^{d_x}.
\end{equation}
$E_x$ is an ellipsoid of dimension $d_x-1$ centered at $\left<xx^T\right>_N^{-1} \left<gx\right>_N/2 = w_x^*/2$, since 
\begin{equation*}
	w_x^T \frac{\partial J^e(w)}{\partial w_x} = w_x^T \left< xx^T\right>_N w_x - w_x^T \left< gx\right>_N
\end{equation*}
and $\left< xx^T\right>_N$ is a symmetric positive definite matrix.

We denote the cylinder generated by $E_x$ as
\begin{equation}\label{eq:E}
	E := E_x \times \mathbb R^{d_y}
\end{equation}
and the enclosed region as 
\begin{equation*}
	E^- := \left\{ (w_x,w_y) ~:~ w_x^T \frac{\partial J^e(\bm w)}{\partial w_x} < 0, \quad w_y \in \mathbb R^{d_y} \right\}.
\end{equation*}
$E^-$ can be regarded as the generalization of the region between y-axis and $\Gamma_0$ as in Figure~\ref{fig:L5L100}; a region in which $||w_y(t)||$ increases following the flow of \eqref{eq:We2dx}.

To define an analogy to global flow structure of \eqref{eq:We2dx} depicted in Figure~\ref{fig:L5L100}, we introduce the  hyperplane 
\begin{equation*}
	H \equiv \left\{ \bm w \in \mathbb R^{d} ~ : ~ \left(n_E^*\right)^T \left(\bm{w} - (w_x^*, 0)\right) = 0\right\},
\end{equation*}
where $n_E^*$ denotes the exterior normal direction of $E$ at $(w_x^*, 0)$ in $\mathbb R^{d}$. Thus
$H$ is the tangent plane of $E$ at $(w_x^*,0)$ in $\mathbb R^{d}$, separating $R^d$ into two disjoint open sets (half spaces). We denote $U^-$ as the part which contains $(0,0)$ while $U^+$ as the other part. More precisely, 
\begin{equation*}
	\begin{aligned}
		U^- := \{\bm w ~:~ {n_E^*}^T \left(\bm{w} - (w_x^*, 0)\right) < 0\}, \\
		U^+ := \{\bm w ~:~ {n_E^*}^T \left(\bm{w} - (w_x^*, 0)\right) > 0\}.
	\end{aligned}
\end{equation*}
Here, we also notice that, 
\begin{equation*}
	\mathbb{R}^d = U^- \cup H \cup U^+,~~(0,0)\in E \subset \overline{U^-}, \quad	\Gamma_0 \subset H.
\end{equation*}
We remark that $\Gamma_0 = H \iff d_x = 1$. 
Figure~\ref{fig:MH_3D} illustrates a corresponding diagram for the case $d_x = 2$ and $d_y=1$. 
\begin{figure}[h]
	\centering
	\includegraphics[width=.8\textwidth]{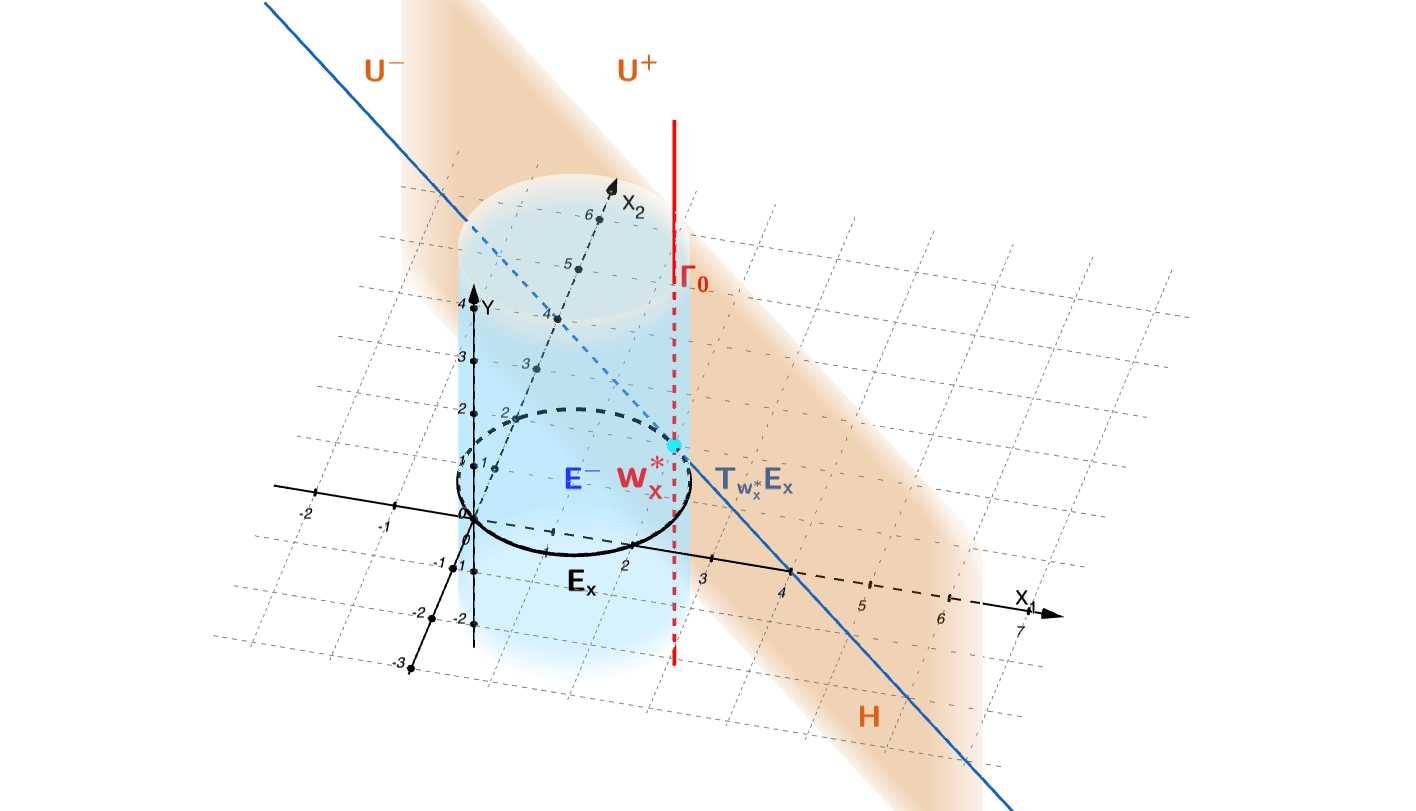} 
	\caption{Example for $d_x = 2$ and $d_y=1$.}
	\label{fig:MH_3D}
\end{figure}

\begin{assumption}\label{as:hyperplane-init}
	Let $\bm{w}(t)$ be a solution of \eqref{eq:We2dx} with $\bm w(0)$, and $\bm w(t) \cap E = \emptyset$ for any $ 0 \le t \le T$.
\end{assumption}
The following proposition states that Assumption~\ref{as:hyperplane-init} holds for some positive time under some conditions on the location of $\bm w(0)$ and the data. 
\begin{proposition}\label{prop:E-}
	If $\bm w(0) \in E^-$ and the correlation matrix of $X$, $\Sigma_X$, satisfies $\Sigma_X = c I_{d_x}$ for some positive constant $c>0$, then $\bm w(t) \in \overline{E^-}$ for all $0 \le t \le T_X$, where $T_X := \inf \{t : \|w_x(t)-w_x^*\| \le \frac{2\sqrt{3}}{c}\|w_x^*\|\|\Sigma_X - \left<xx^T\right>\| \}$.
\end{proposition}
Since $\|\Sigma_X - \left<xx^T\right>\| = \mathcal O(\frac{1}{N})$ can be made arbitrary small if one increases the number of data points $N$. In that case, $\bm w(t)$ will stay in $\overline{E^-}$ before it reaches a neighborhood of the stationary manifold $\Gamma_0$ (when $\bm w(0) \in E^-$).

\begin{lemma}\label{lem:Rt}
	Suppose that $w_y(0) \neq 0$ and $\bm w(t)$ satisfies Assumption~\ref{as:hyperplane-init} for $0\le t\le T$.
	Then
	\begin{enumerate}
		\item if $\bm w(0) \in U^+$, 
		$$\frac{d}{dt} \|w_x(t)\|^2 \le 0~~\text{and}~~\frac{d}{dt} \frac{\|w_x(t)\|^2}{\|w_y(t)\|^2} \le 0,$$
		\item if $\bm w(0) \in E^-$,
		$$\frac{d}{dt} \|w_x(t)\|^2 \ge 0~~\text{and}~~\frac{d}{dt} \frac{\|w_x(t)\|^2}{\|w_y(t)\|^2} \ge 0,$$
	\end{enumerate}
	for $0\le t\le T.$
\end{lemma}
As a consequence of the monotonicity of $\|w_x\|^2$ and $\frac{\|w_x\|^2}{\|w_y\|^2}$, we have the next main theorem about the regularization and side effects of LNNs.
\begin{theorem}\label{thm:LLN_wy_sigma=0}
	Suppose that $w_y(0)\neq 0$ and $\bm{w}(t)$ satisfies Assumption~\ref{as:hyperplane-init} for $0\le t \le T$. Then
	\begin{enumerate}
		\item if $\bm w(0) \in U^+$,
		\begin{equation}\label{eq:w_yT}
			\|w_y(T)\|^2 - \|w_y(0)\|^2 \le \frac{(L-1)\|w_y(0)\|^2}{L\|w_x(0)\|^2 + \|w_y(0)\|^2} \left( \|w_x(T)\|^2 - \|w_x(0)\|^2 \right)\le 0;
		\end{equation}
		\item if $\bm w(0) \in E^-$,
		\begin{equation}\label{eq:w_yTE-}
			\|w_y(T)\|^2 - \|w_y(0)\|^2 \ge \frac{(L-1)\|w_y(T)\|^2}{L\|w_x(T)\|^2 + \|w_y(T)\|^2} \left( \|w_x(T)\|^2 - \|w_x(0)\|^2 \right)\ge 0.
		\end{equation}
	\end{enumerate}
\end{theorem}

Recall that for LNNs, $w_y$ determines the Lipschitz bound of the trained network function, $f_\theta$ in the direction orthogonal to the data manifold. Therefore, the first case in Theorem~\ref{eq:w_yT} can be interpreted as
the regularization effect of LNNs: under the stated conditions, the $w_y(T)$ is smaller than $w_y(0)$, implying that the variation of $f_\theta$ in the $y$-directions will reduce.
On the other hand, the second case in Theorem~\ref{eq:w_yT} 
reveals a ``side effect'' of LNNs that the variation of $f_\theta$ in the $y$-directions will increase; i.e. the stability (for out of distribution evaluations) of the network will reduce as training progresses. 

Furthermore, we can derive the following {\it{a priori}} estimate:
\begin{equation}\label{eq:w_yw*}
	\|w_y(T)\|^2  \le \|w_y(0)\|^2 + \frac{(L-1)\|w_y(0)\|^2}{L\|w_x(0)\|^2 + \|w_y(0)\|^2} \left( \|w_x^*\|^2 - \|w_x(0)\|^2 \right)
\end{equation}
from \eqref{eq:w_yT}. 
By reorganizing \eqref{eq:w_yw*}, we have the following a priori estimate
\begin{equation}\label{eq:w_yTupper}
	\left\|w_y(T)\right\|^2 \le h(L)\left( \|w_x(0)\|^2 - \|w_x^*\|^2\right) + \left( \frac{\|w_x^*\|^2}{\|w_x(0)\|^2}\right)\|w_y(0)\|^2,
\end{equation}
where $h(L) = \frac{1+\|w_y(0)\|^2/\|w_x(0)\|^2}{L(\|w_x(0)\|^2/\|w_y(0)\|^2) + 1}$
is a decreasing function in terms of $L$. 
That is, the upper bound for $\|w_y(T)\|$ with $L=100$ is smaller than the case of $L=5$ under the same initialization. Thus, the estimate in \eqref{eq:w_yTupper} can partially explains the phenomenon in Figure~\ref{fig:L5L100} in the right of $\Gamma_0$ that
$|w_y(T)|$ with $L=100$ is smaller than the case of $L=5$ under the same initial when $\bm w(t)$ achieves $\Gamma_0$.

\section{ReLU activated networks}\label{sec:ReLUDNN}
In this section, we analyze the stability for ReLU deep neural networks (DNNs) when data are sampled from $\mathcal M$, i.e., $\bm x_i \sim M_0$. We first show how the low dimensional data will affect the training process. Given that, we establish the stability estimate for ReLU DNNs with one hidden layer ($L=2$). By using the recursive structure of ReLU DNNs, we finally prove the stability estimate for deep cases.

As defined in \eqref{eq:def_NN} and \eqref{eq:def_relu}, we have the ReLU DNN function with $L-1$ hidden layers as  
\begin{equation}\label{eq:DNN}
	\begin{cases}
		f^{\ell}(\bm x) &= W^\ell \alpha(f^{\ell-1}(\bm x) ) + b^{\ell}, \quad \ell = 2:L, \\
		f (\bm x,\theta) &= f^{L}(\bm x),
	\end{cases}
\end{equation}
where$f^1(\bm x) = W^1\bm x + b^1$, $\alpha = {\rm ReLU}$, $W^\ell \in \mathbb R^{n_{\ell}\times n_{\ell-1}}$, $b^\ell, f^\ell \in\mathbb{R}^{n_{\ell}}$ with $n_0=d=d_x+d_y$ and $n_L=1$.
Here, $W^1$ is a $n_1 \times (d_x+d_y)$ matrix, and 
for the convenience of exposition, we write $W^1 = \begin{pmatrix}
	W^1_x~W^1_y\end{pmatrix}$, where $W_x^1$ and $W_y^1$ are, respectively, $n_1\times d_x$ and $n_1\times d_y$ matrices. With the data of the form prescribed in Section~\ref{subsec:data}, we assume $Q=I_d$ and have
\begin{equation*}
	W^1 \bm x_i + b^1 = \begin{pmatrix}
		W^1_x~ W^1_y\end{pmatrix} \begin{pmatrix}
		x_i \\ \sigma y_i
	\end{pmatrix} + b^1.
\end{equation*}
Then, the loss function is defined as
\begin{equation}\label{eq:lossf}
	J(\theta) =\frac{1}{2N} \sum_{i=1}^{N} ( f(\bm x_i; \theta) -g_i )^2,
\end{equation}
where $\bm x_i \sim M_\sigma$ and $\theta = \{W^1, b^1, \cdots, W^L, b^L\}$ denotes all parameters in ReLU DNNs.

If $\bm x_j \in \mathcal M$, the key observation here is that 
\begin{equation*}
	\frac{\partial J}{\partial \widetilde W_y^1} =  0,~~~\widetilde W^1_y = [W^1Q]_y.
\end{equation*}
Furthermore, according to the gradient descent update of $W^1$, we have
\begin{equation*}
	W^1Q \leftarrow W^1Q - \eta \frac{\partial J}{\partial W^1}Q
	\Longrightarrow 
	W^1Q \leftarrow W^1Q - \eta \frac{\partial J}{\partial (W^1Q)}.
\end{equation*}
Thus, $W^1_y$ or $\widetilde W_y^1$ will not change for any pure gradient descent-based training algorithms.
Therefore, without loss of generality, we shall assume in the remaining of this section that  $Q= I_d$.The results can be easily  extended to $\widetilde W^1=W^1Q$ and $(\widetilde x, \widetilde y) = Q^T\bm x$ if $Q\neq I_d$.

\begin{lemma}\label{lemm:w_y}
	If {$\bm x_j \sim M_0$ in the training data} and either the {full} gradient descent or stochastic gradient descent training algorithm is applied {to \eqref{eq:DNN} and \eqref{eq:lossf}}, then the following conclusions hold.
	\begin{enumerate}
		\item $W^1_y$ in $W^1 = \begin{pmatrix} W^1_x, W^1_y\end{pmatrix}$ will not change during the training process \eqref{eq:training-process}.
		\item If there is a $\ell^2$ regularization term $\lambda \|\theta\|_{\ell^2}^2$ with an appropriate $\lambda$, then $W_y^1$ will decay to $0$.
	\end{enumerate}
\end{lemma}

Although Lemma~\ref{lemm:w_y} also holds for LNNs, estimating $\|\bm w_y^*\|$ directly for LNNs as in Theorem~\ref{thm:LLN_wy_sigma=0} is a more precise and efficient approach to bound the stability metric.
However, there is no such linear structure that we can use for ReLU DNNs. 
Thus, we notice the first consequence in Lemma~\ref{lemm:w_y} which shows an invariant property of weights $W_y^1$ in training ReLU DNNs with $\sigma=0$ {for both full and stochastic gradient descent methods}. {The invariant property of $W_y^1$ in training ReLU DNN with (stochastic) gradient descent method} plays a critical role in analyzing the stability metric which will be detailed explained in the remaining subsections.
For simplicity, we denote $W^\ell (b^\ell)$ and as the initialized weights and $\overline{W}^\ell$ ($\overline{b}^\ell$) as the weights (biases) after training. In the following, we will use $\theta^*$ to denote the parameter set obtained after training. From the discussion above, $\theta^*= \left\{(\overline{W}^\ell,\overline{b}^\ell)\right\}_{\ell=1}^L$ while $\overline{W}^1 = (\overline{W}^1_{x}, W^1_y)$ due to the conclusion in Lemma~\ref{lemm:w_y} if $\theta^*$ is obtained by FGD or SGD.

\paragraph{An example.}
We train and obtain a neural network classifier, $f_{\theta^*}:\mathbb{R}^{784}\mapsto \mathbb{R}^{10}$, using the MNIST data set \cite{deng2012mnist}. The first layer of the network is fully connected. Each image in the MNIST data set is a black-and-white image consisting of $28\times 28$ pixels, and it is regarded as a point in $\mathbb{R}^{784}$. 
	Let $\bm{\bar x}$ be the mean of the data points. Let the unit vector $\bm{v}_n$ denote a direction corresponding to the least eigenvalue of the covariance matrix.  The ratio between the largest and the least eigenvalue of the covariance matrix of MNIST is $5.26\times 10^{16}$. We shall regard the data manifold $\mathcal{M}$ to be the subspace, centered at $\bm{\bar x}$, spanned by the first $783$ principal directions. 

Let $\overline{W}^1$ denote the weights in $f_{\theta^*}$ that connects to the input vector. We introduce perturbation to the weight set $\overline{W}^1+s W_y$, where $W_y:= {\bm v}_n {\bm v}_n^T$, and denote the corresponding perturbed network as $f_{\theta_s^*}$. Let $f^{[2]}_{\theta^*}$ denote the second component of the output vector that corresponds to the digit `2'.
	Classification of an input image $\bm{x}$ is performed by 
	the maximal component of  $\text{Softmax}(f_{\theta^*}(\bm{x})),$
	using a trained network $f_{\theta^*}$ with 
	$98.14\%$ testing accuracy. 
	In Figure \ref{MNIST}, we show the function
	\begin{equation}\label{eq:I_st}
		I_j(s,t):=||f^{[2]}_{\theta^*}( \bm{x}_j)-f^{[2]}_{\theta^*_s}(\bm{x}_j+ t \bm{v}_n)||_2^2,
	\end{equation}
	for $\bm{x}_j$.  We observe that $I(s, 0)$ remains 0 as $s$ varies; in other words, variations in the $W_y$ component of $\overline{W}^1$ does not change the perturbed network's output when evaluated at $\bm{x}_j$. This means that the data point $\bm{x}_j$ has no role in the optimization of $W_y$ in $W^1$, in a gradient descent-based training. Furthermore, 
	$f_{\theta^*_s}$ starts to deviate from $f_{\theta^*}$ only when one introduces perturbation to the input $\bm{x}_j$ in the direction, ${\bm v}_n$, normal to the data set.

In following subsections, we derive upper bounds on the effect of the perturbation discussed above.

\begin{figure}[h]
	\centering
	\includegraphics[width=0.32\textwidth]{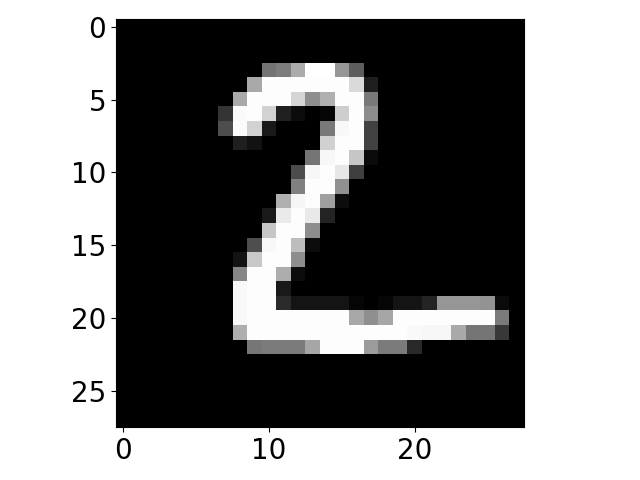}
	\includegraphics[width=0.32\textwidth]{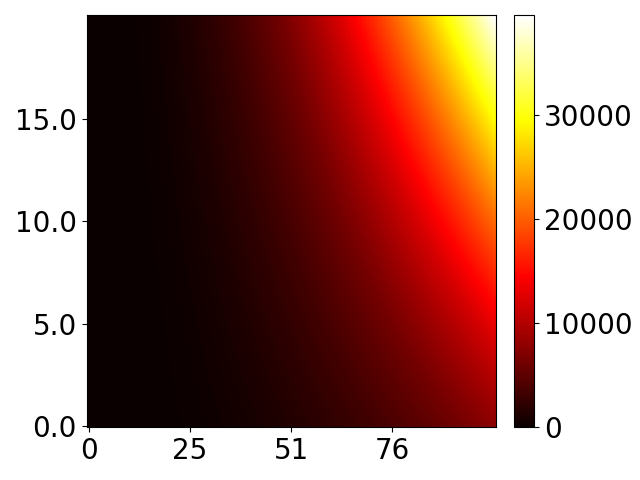}
	\includegraphics[width=0.32\textwidth]{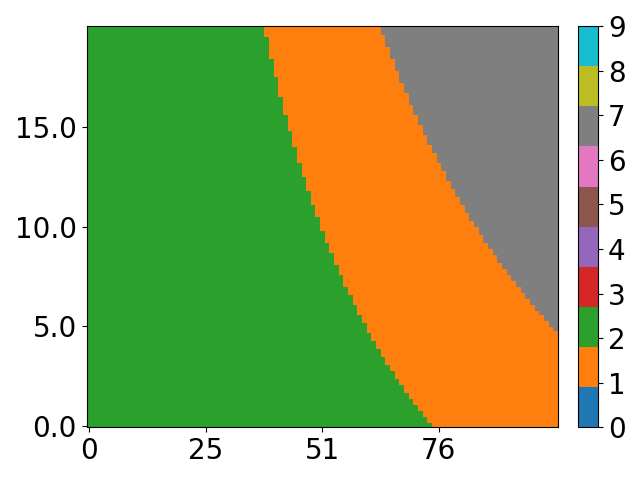}
	\caption{Left: An image corresponding to the digit `2'. Center: The change in the third component of the output vector, resulting from the perturbation to the input image (parameter $t$ in \eqref{eq:I_st}, horizontal axis) and to the trained  network's first layer weights (parameter $s$ in \eqref{eq:I_st}, vertical axis). Right: The classification based on the perturbed output. }
	\label{MNIST}
\end{figure}

\subsection{Stability estimate for $L=2$}
First, let us consider networks with only one hidden layer, which means $L=2$. For input training data, we have $\bm x_i = (x_i, 0)\sim  M_0$. In addition, we also denote
$\Omega_x = (-1,1)^{d_x}$ as the domain of input of $x_i$.
That is, we have
\begin{equation}\label{eq:1layer_f}
	f(\bm x;{\theta^*}) = f(x,y) := \sum_{i=1}^n \overline{W}^2_i \alpha(\overline{W}^1_{i,x} x + \overline{b}^1_i + W^1_{i,y} y) + \overline{b}^2
\end{equation}
as the approximation of $g(x)$ after training. 
According to Lemma~\ref{lemm:w_y}, $W^1_{i,y}$ is given by initialization since $\sigma = 0$ in the training data. 

Then, for any $y\neq 0$, we propose to estimate the following deviation along the $y$-direction
\begin{equation*}
	\left\| f(\cdot,y) - f(\cdot,0)\right\|^2_{L^2(\Omega_x)}= \left\|\sum_{i=1}^{n} e_i(\cdot,y) \right\|^2_{L^2(\Omega_x)},
\end{equation*} 
where
\begin{equation}\label{eq:e_i}
	e_i(x,y) = \overline{W}^2_i\left( \alpha(\overline{W}^1_{i,x} x + \overline{b}^1_i + W^1_{i,y} y) -  \alpha(\overline{W}^1_{i,x} x + \overline{b}^1_i) \right).
\end{equation}
{In other words, $e_i(x,y)$ describes the stability of each  neuron's activation in the first hidden layer.}

Using the property of ${\rm ReLU}$ function, {one can easily describe the support of $e_i(x,y)$ given the trained parameters $\overline{W}^1_{i,x}$, $W^1_{i,y}$, and $\overline{b}^1_i$. See the strip depicted in Figure~\ref{fig:Omegaix-}.} Thus, we have the following estimate for $e_i(x,y)$.

\begin{lemma}\label{lamm:e^2(x,y)}
	{Let $\bm x_i \sim M_0$ in the training data, $f(x,y)$ be a network with a single hidden layer ($L=2$) defined in \eqref{eq:1layer_f} and trained by FGD or SGD, and $e_i(x,y)$ be defined in \eqref{eq:e_i}.} For any $i=1:n_1$, we have
	\begin{equation*}
		\left\|e_i({\cdot},y)\right\|^2_{L^2(\Omega_x )} \le 
		\frac{\left\|\nabla {h_i}\right\|_{L^2(\Omega_x)}^2\left|W^1_{i,y}y\right|^2}{\left\|\overline{W}^1_{i,x}\right\|^2} + C_{d_x} \frac{\left|\overline{W}^2_i\right|^2\left|W^1_{i,y}y\right|^3}{3\left\|\overline{W}^1_{i,x}\right\|},
	\end{equation*}
	where $C_{d_x}$ denotes the measure of the largest $(d_x - 1)$-hyperplane in $\Omega_x$ and
	\begin{equation*}
		h_i(x) =\overline{W}^2_i\alpha(\overline{W}^1_{i,x} x + \overline{b}^1_i),~~~f(x,0)=\sum_{i=1}^{n_1} h_i(x)+\overline b^2.
	\end{equation*}
\end{lemma}
{Notice that $h_i(x)$ are Lipschitz in $x$ so $||\nabla h_i||^2_{L^2(\Omega)}$ is well-defined. We denote $h_i(x)$ explicitly and separately since $\nabla_x f(x,0) = \sum_{i=1}^{n_1} \nabla h_i(x)$ where $f(x,0)$ could be the approximation of the target function $g(x)$ on $\mathcal M$. 
	The estimate presented in Lemma~\ref{lamm:e^2(x,y)} is a type of {\it a posteriori} estimate since it depends on the parameters $\overline{W}^\ell$ and $\bar{b}^2$ obtained as the results of training.  }

{We first notice that the stability of each trained neuron depends on the derivative of $h$ with respect to each input variable. The derivatives depends on the trained parameters that are directly connected to the input vector. These parameters depend on the data and the training algorithm. Furthermore, we observe that the stability of a neuron is dependent on the ``untrainable" parameters in $W^1_y$! Finally, the lemma suggests that if the trained network is more stable if the weight $W_i^2$ connecting to the output is small. This matches with our intuition that $W_i^2$ may amplify the contribution of the $y$ components of the input.By summing all $e_i(x,y)$ together and applying the triangle inequality, we have the following estimate for trained ReLU DNNs with one hidden layer.}
\begin{theorem}\label{thm:stabilityL=2}
	Let $\bm x_i \sim M_0$ in the training data and $f(x,y)$ be a network with a single hidden layer ($L=2$) defined in \eqref{eq:1layer_f} {and trained by FGD or SGD}, then 
	\begin{small}
		\begin{equation*} 
			\left\| f(\cdot,y) - f(\cdot,0)\right\|^2_{L^2(\Omega_x)} \le \sum_{i=1}^{n_1} \left( \frac{\left|W^1_{i,y}y\right|^2\left\|\nabla h_i\right\|_{L^2(\Omega_x)}^2}{\left\|\overline W^1_{i,x}\right\|^2} + C_{d_x} \frac{\left|\overline W^2_i\right|^2\left|W^1_{i,y}y\right|^3}{3\left\|\overline W_{i,x}^1\right\|} \right),
		\end{equation*}
	\end{small}
	{where $C_{d_x}$ and $h_i(x)$ follow the same definitions in Lemma~\ref{lamm:e^2(x,y)}}.
\end{theorem}
{This theorem gives the stability estimate for a ReLU DNN with one hidden layer trained by FGD or SGD. 
	It is the building block for understanding the stability of a deep neural network. The next step is to use the nonlinear recursion relations that define the deep network to propagate the influence of having nonzero $y$ components in the input vector input the other hidden layers. }

\subsection{Stability estimate for $L>2$}
For a general multi-layer neural network with ReLU activation function, as shown in \eqref{eq:DNN}, we denote the function {trained by FGD or SGD} as $f (\bm x; \theta) = f^{L}(\bm x)$ where
\begin{equation*}
	f^{\ell}(\bm x) =\overline W^\ell \alpha(f^{\ell-1}(\bm x) ) + \overline b^{\ell}, \quad \ell = 2:L,
\end{equation*}
with $f^1(\bm x) =\overline W^1\bm x +\overline b^1$. 
Let $f^\ell(\bm x)$, $\ell=1,\cdots, L$ be the functions in  \eqref{eq:DNN} and 
\begin{equation*}
	\triangle_y f^\ell (x,y):= f^\ell(x,y) - f^\ell(x,0).
\end{equation*}
In particular, 
\begin{equation*}
	\triangle_y f(x,y):= f^L(x,y) - f^L(x,0).
\end{equation*}

We have the following recursion relation of $\triangle_y f^\ell (x,y)$. 
\begin{lemma}\label{lemm:e-recursion}
	For any fixed $x\in \mathbb{R}^{d_x}$ and $y \in \mathbb{R}^{d_y}$, we have
	\begin{equation*}
		\left\|\triangle_y f^\ell (x,y)\right\| \le \left\|\overline W^{\ell}\right\|\left\|\triangle_y f^{\ell-1} (x,y)\right\|,
	\end{equation*}
	where $\left\|\triangle_y f^\ell (x,y))\right\|$ denotes the $\ell^2$ vector norm of $\triangle_y f^\ell (x,y)$ and 
	$\left\|\overline W^{\ell}\right\|$ is the operator norm of $\overline W^{\ell}$ with respect to $\ell^2$ norm.
\end{lemma}
\begin{proof}By definition, 
	\begin{align*}
		\left\|\triangle_y f^\ell (x,y)\right\|^2
		= &\left\|\overline{W}^{\ell} \left(\alpha\left(f^{\ell-1}(x,y) \right) - \alpha\left(f^{\ell-1}(x,0) \right)\right) \right\|^2 \\
		=&\left\|\overline W^\ell \left(\alpha\left(f^{\ell-1}(x,0) + \triangle_y f^{\ell-1}(x,y)\right) - \alpha\left(f^{\ell-1}(x,0) \right)\right) \right\|^2 \\
		\le &\left\|\overline W^\ell \right\|^2 \left\|\left(\alpha\left(f^{\ell-1}(x,0)  + \triangle_y f^{\ell-1}(x,y)\right) - \alpha\left(f^{\ell-1}(x,0) \right)\right) \right\|^2  \\
		\le &\left\|\overline W^{\ell}\right\|^2\left\|\triangle_y f^{\ell-1}(x,y)\right\|^2.
	\end{align*}
	The last inequality holds because of the property of ${\rm ReLU}$ that $|{\rm ReLU}(x+h) - {\rm ReLU}(x)| \le |h|$ for any $x,h\in \mathbb{R}$. 
\end{proof}
By applying the previous recursion result, we have
\begin{equation*}
	\left\|\triangle_y f^\ell(x,y)\right\| \le \left\|\overline W^{\ell}\right\|\left\|\triangle_y f^{\ell-1}(x,y)\right\| \le \cdots \le \left(\prod_{j = 3}^\ell \left\|\overline W^{j}\right\|\right)\left\|\triangle_y f^{2}(x,y)\right\|
\end{equation*}

Combining Lemma~\ref{lamm:e^2(x,y)} and Lemma~\ref{lemm:e-recursion}, we have the following {\it{a posteriori}} estimate for $\left\|\triangle_y f(x,y)\right\|^2_{L^2(\Omega_x)}$.

\begin{theorem}\label{thm:stabilityL>2} Let $\bm x_i \sim M_0$ in the training data and $f(x,y)$ be a network with $L$ layers defined in \eqref{eq:1layer_f}, then the following inequality holds for any fixed $y\in \mathbb R^{d_y}$ {if $f(x,y)$ is trained by FGD or SGD}: 
	\begin{small}
		\begin{equation}\label{eq:estimate_L>2}
			\left\|\triangle_y f(\cdot,y)\right\|^2_{L^2(\Omega_x)} \le \left(\prod_{\ell = 3}^L \left\|\overline{W}^{\ell}\right\|^2\right) \sum_{\substack{i=1:n_2 \\j=1:n_1}}\left(\frac{\left|{W}^1_{j,y}y\right|^2\left\|\nabla_x h_{i,j}\right\|_{L^2(\Omega)}^2}{\left\|\overline{W}^1_{j,x}\right\|^2} + C_{d_x} \frac{\left|\overline{W}^2_{i,j}\right|^2\left|W^1_{j,y}y\right|^3}{3\left\|\overline{W}^1_{j,x}\right\|}\right),
		\end{equation}
	\end{small}
	where {$C_{d_x}$ follows the definition in Lemma~\ref{lamm:e^2(x,y)} and}
	\begin{equation*}
		h_{i,j} = \overline{W}^2_{i,j}\alpha(\overline{W}^1_{j,x}x+\overline{b}^1_j).
	\end{equation*}
\end{theorem}

\begin{proof} By definition, we have
	\begin{equation*}
		\begin{aligned}
			&\left\|\triangle_y f(\cdot,y) \right\|^2_{L^2(\Omega_x)}\equiv\left\|\triangle_y f^L(\cdot,y) \right\|^2_{L^2(\Omega_x)} 
			\le \left(\prod_{\ell = 3}^L \left\|\overline W^{\ell}\right\|^2\right)\left\|\triangle_y f^2(\cdot,y)\right\|^2_{L^2(\Omega_x)}\\
			\le &\left(\prod_{\ell = 3}^L \left\|\overline W^{\ell}\right\|^2\right) \sum_{i=1}^{n_2}\left\|\sum_{j=1}^{n_1} \overline W^2_{i,j}\left(\alpha(\overline W_{j,x}^{1} x + \overline b^1_j + W^1_{j,y} y) - \alpha(\overline W^1_{j,x} x + \overline b^1_j)\right)\right\|^2_{L^2(\Omega_x)}\\
			\le&\left(\prod_{\ell = 3}^L \left\|\overline W^{\ell}\right\|^2\right) \sum_{i=1}^{n_2}\sum_{j=1}^{n_1}\left(\frac{\left|W^1_{j,y}y\right|^2\left\|\nabla_x h_{i,j}\right\|_{L^2(\Omega)}^2}{\left\|\overline W^1_{j,x}\right\|^2} + C_{d_x} \frac{\left|\overline W^2_{i,j}\right|^2\left|W^1_{j,y}y\right|^3}{3\left\|\overline W^1_{j,x}\right\|}\right).
		\end{aligned}
	\end{equation*} 
\end{proof}
Theorem~\ref{thm:stabilityL>2} provides an estimation for the variation of a ReLU DNN trained by FGD or SDG along the normal direction of the data manifold.
It is by no means sharp, because of the approximation \eqref{eq:para_Omega}.
However, as in the case of LNNs, the initialization of $W^1_y$ and the network's depth $L$ play a role in the stability of the trained network as shown in Corollary~\ref{coro:stabilityL>2}.
Theorem~\ref{thm:stabilityL>2} has an interesting implication for DNNs that employ a latent space of a smaller dimensionality. 
	The estimate in the theorem does not assume that the hidden layers in the DNN have the same width. This means that when $y\neq 0$, the effect of the ``untrainable parameters" $W^1_y y$ will propagate into the subsequent layers, even when the layers have smaller widths. In training for data without noise, $y\equiv 0$, there is no mechanism to learn how to project $W^1_y \tilde y$ out for any $\tilde y\neq 0$ in noise test data.
	As far as we know, this is the first stability estimate ($\left\|\triangle_y f(\cdot,y) \right\|^2_{L^2(\Omega_x)}$) for a general ReLU DNN trained by FGD or SGD.
\begin{corollary}\label{coro:stabilityL>2}
	Under the same assumptions in Theorem~\ref{thm:stabilityL>2} and $$
	D(y): = \max_{i,j,k}\left\{ \frac{[y]_k^2\|\nabla_x h_{i,j}\|^2_{L^2(\Omega)}}{\left\|\overline{W}^1_{j,x}\right\|^2}, C_{d_x}\frac{[y]_k^3\left|\overline{W}^2_{i,j}\right|^2}{3\left\|\overline{W}^1_{j,x}\right\|^2}\right\},
	$$ if $\left[W_{j,y}^1\right]_k \sim \mathcal N(0,\nu^2)$ for all $k=1:d_y$, then there exists a constant $\widetilde D$ such that
	\begin{footnotesize}
		\begin{equation}
			\left\|\triangle_y f(\cdot,y)\right\|^2_{L^2(\Omega_x)} \le \left(\prod_{\ell = 3}^L \left\|\overline{W}^{\ell}\right\|^2\right)
			\left(  \left(\nu^2+2\sqrt{\frac{2}{\pi}}\nu^3\right)n_2n_1d_y + \widetilde D \sqrt{n_2n_1 d_y} \right)D(y),
		\end{equation}
	\end{footnotesize}
	with high probability.
	Here, $n_1$ and $n_2$ are the widths of the first and second hidden neuron layers defined in \eqref{eq:DNN}.
\end{corollary}
Commonly used initialization strategies correspond to $\nu^2 = \frac{1}{d}$ in~\cite{glorot2010understanding} or $\nu^2 = \frac{2}{d+n_1}$ in \cite{he2015delving}. Recently, the authors in \cite{chen2022weight} propose to take $\nu^2 = \frac{2}{\sqrt{n_1d}}$ which leads to the following estimate 
$$
\left\|\triangle_y f(\cdot,y)\right\|^2_{L^2(\Omega_x)} \le \overline{D}\left(\prod_{\ell = 3}^L \left\|\overline{W}^{\ell}\right\|^2\right)
n_2\sqrt{n_1 d_y}D(y),
$$
where $\overline{D} = \max\{4,\widetilde D\}$.

The above corollary suggests that, in additional to the common practice, the the width of the second hidden layer should be considered in the initialization of $W_{j,y}^1$.

For classification problems, our theory provides additional 
understanding of adversarial examples~\cite{goodfellow2014explaining}. 
Particularly, our theory may explain the existence of those
adversarial examples which are close to the training examples according to some norm defined on the ambient space but are not a member of some idealized lower dimensional data manifold.
The estimate in \eqref{eq:estimate_L>2} indicates that the variation of a trained ReLU DNN can significantly move the ``decision boundary" for a small $y$  provided $\left\|\overline{W}^\ell \right\|$ or $\left\|\nabla_x h_{i,j}(x)\right\|_{L^2(\Omega)}^2$ are sufficiently large. In this case, one can obtain adversarial examples easily with a very small perturbation along the normal direction of the data manifold. 
In addition, this result combined with the second conclusion in Lemma~\ref{lemm:w_y} and numerical results in Figure~\ref{fig:error_WD} indicate that including a ``weight decay" term in the loss function may reduce the reliability of a ReLU DNN based  classifier, as least when the data manifold is nearly flat. 

In this section, we focused on estimating the stability of ReLU neural networks trained by FGD or SGD.Theorems~\ref{thm:stabilityL=2} and \ref{thm:stabilityL>2} reveal the influence of the ``trainable" and ``non-trainable" parameters, $\overline{W}^\ell$ (including $\overline{W}^1_{x}$) and $W_y^1$, to the inference stability. The influence of the non-trainable parameters is unchanged, even if $\overline{W}^\ell$ are replaced by non-optimal ones.
However, if the target functions fall into those considered in \cite{cloninger2021deep}, $W^1_y$ will be trainable, and the theoretical optimal inference error derived there is applicable.

We present some numerical results in the following subsection to demonstrate the above estimates. In particular, the stability metrics of ReLU DNNs with one hidden layer ($L=2$) may differ from multi-hidden-layer ($L>2$) cases since the product term will disappear if $L=2$. This is observed in Figure~\ref{fig:Error_L1-20}.

\subsection{Numerical experiments}\label{sec:numerics}
In this section, we present a series of numerical examples demonstrating the theorems presented in this paper.

\paragraph{The setup}
We take $d_x = 3$ and $d_y = 2$, i.e., $x\in \mathbb{R}^3$, $y\in\mathbb{R}^2$ and $\bm x = (x,0)$. 
A total of  $5\times10^3$ training data points generated by sampling $g_i = g(x_i) = \sum_{j=1}^3\sin\left(\pi[x_i]_j\right)$ 
with $x_i \sim  U\left( [-1,1]^{d_x}\right)$.
The hidden layers in a network have the same with, denoted by $n$.

The ReLU DNNs and their optimization are implemented using PyTorch~\cite{paszke2019pytorch}. 

The networks are trained for 100 epochs by using SGD without momentum or weight decay. The mini-batch size is chosen as $50$ and the learning rate decays from $10^{-2}$ to $10^{-4}$ under a cosine annealing schedule~\cite{loshchilov2016sgdr}.

To compute the stability estimates, we adopt the Monte Carlo approximation 
\begin{equation}\label{eq:stabilitycal}
	\mathbb E_{y} \left[\left\|\triangle_y f(\cdot,y)\right\|_{L^2(\Omega_x)}^2\right] 
	\approx \frac{1}{M} \sum_{i=1}^M \left(f(x_i,y_i) - f(x_i,0)\right)^2,
\end{equation}
where $x_i \sim U\left( [-1,1]^{d_x}\right)$ and $y_i \sim N(0,\gamma^2I_{d_y})$. 
The weights $W_y^1$ are initialized following the special form
\begin{equation*}
	W_{i,y}^1 = \eta(1,1)^T.
\end{equation*}
All other weights are initialized according to~\cite{he2015delving}.
We take $M=5\times10^3$ to evaluate the stability metric and the final results are obtained by averaging 10 individual tests.

\paragraph{Numerical confirmation of various rates}

Theorem~\ref{thm:stabilityL=2} and Theorem~\ref{thm:stabilityL>2} state that with a fixed weight set $W_y^1$, 
\begin{equation*}
	\mathbb E_{y} \left[\left\|\triangle_y f(\cdot,y)\right\|_{L^2(\Omega_x)}^2\right] \sim \mathcal O\left( \gamma^2 \right),
\end{equation*}
if $y\sim \mathcal N(0,\gamma^2I_{d_y})$. 
On the other hand,
\begin{equation*}
	\mathbb E_{y} \left[\left\|\triangle_y f(\cdot,y)\right\|_{L^2(\Omega_x)}^2\right] \sim \mathcal O\left( \eta^2\right)
\end{equation*}
if $W^1_{i,y}$ is initialized $\eta(1,1)$ and $\gamma=1$ in the distribution of $y$. 
Figure~\ref{fig:Error_L1-20} demonstrates such scalings for networks of different depths.
Also from Figure~\ref{fig:Error_L1-20} one may observe a gap between the curve from $L=2$ and those $L>2$. This gap seems to suggest that ReLU DNNs with one hidden layer differ from multi-hidden-layer models. Results in Figure~\ref{fig:Error_L1-20} further support this observation if we compare with some deeper ReLU DNNs. This phenomenon can by partially interpreted as the effect of the term $\prod_{\ell = 3}^L \left\|\overline W^{\ell}\right\|^2$ as shown in Theorem~\ref{thm:stabilityL>2}.
\begin{figure}[h]
	\centering
	\includegraphics[scale=0.8]{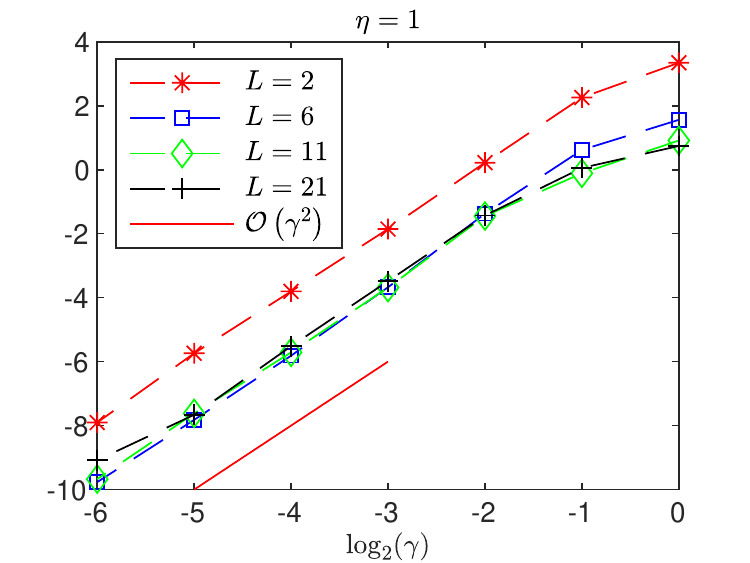}\includegraphics[scale=0.8]{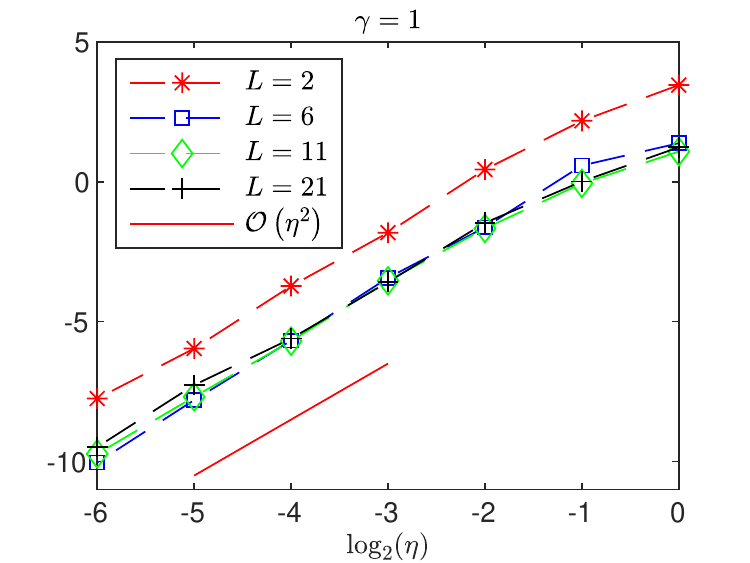}
	\caption{Plots of $\log_2\left(\mathbb E_{y} \left[\left\|\triangle_y f(\cdot,y)\right\|_{L^2(\Omega_x)}^2\right]\right)$, with $y\sim \mathcal N(0,\gamma^2I_{d_y})$ and $W_y^1=\eta(1,1)^T$. Each hidden layer of the networks has $n=100$ neurons.
		The plots verify the estimate Theorem~\ref{thm:stabilityL>2}.}
	\label{fig:Error_L1-20}
\end{figure}

\paragraph{Regularization by adding a ``weight decay'' term}
We recall the second statement in Lemma~\ref{lemm:w_y} that the $\ell^2$ regularization term $\lambda \|\theta\|^2_{\ell^2}$ will significantly affect the
stability factor $\mathbb E_{y} \left[\left\|\triangle_y f(\cdot,y)\right\|_{L^2(\Omega_x)}^2\right]$.
Thus, we show the training process and final loss (training loss, test loss and stability metric) with different values of $\lambda$ in Figure~\ref{fig:error_WD}.

In this example, the training loss is defined in~\eqref{eq:lossf} with $y\equiv 0$ and test loss is calculated with the same formula while it shares the same sampled date points in computing $\mathbb E_{y} \left[\left\|\triangle_y f(\cdot,y)\right\|_{L^2(\Omega_x)}^2\right]$ with $y\sim \mathcal N(0,\gamma^2I_{d_y})$ and $\gamma=2^{-1}$, and the initialization of $W^1_{i,y}$ is $(1,1)$, i.e. $\eta=1$.

This example shows that (i) there is no surprise that regularizing the $\ell^2$ norm of the weight set reduces the stability metric $\mathbb E_{y} \left[\left\|\triangle_y f(\cdot,y)\right\|_{L^2(\Omega_x)}^2\right]$; however, (ii) both the training and test losses will increase as the magnitude of the regularization, $\lambda$, increases.
In practice, a suitable scale of $\lambda$ is critical to balance the approximation error and the regularization effect for the stability metric $\mathbb E_{y} \left[\left\|\triangle_y f(\cdot,y)\right\|_{L^2(\Omega_x)}^2\right]$.
\begin{figure}[h]
	\centering
	\includegraphics[scale=0.8]{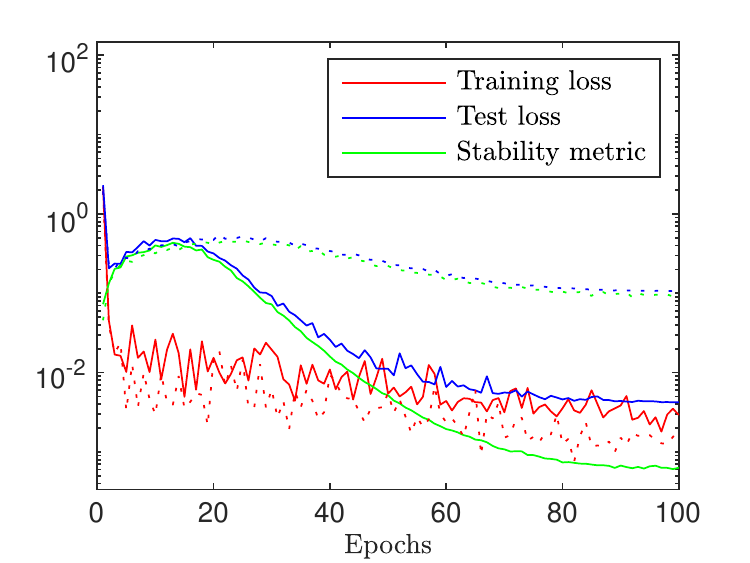}\includegraphics[scale=0.8]{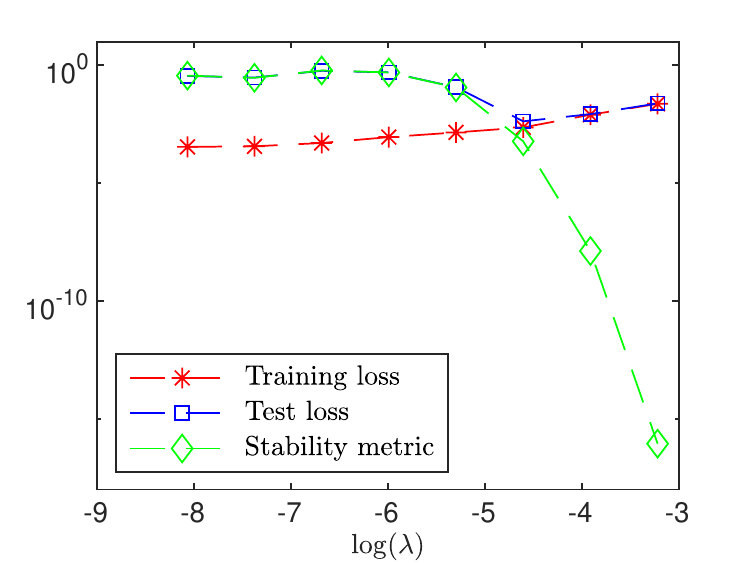}
	\caption{Effect of add the ``weight decay'' term $\lambda\|\theta\|^2_{\ell^2}$ in the total loss function. Training loss, test loss, and stability metric $\mathbb E_{y} \left[\left\|\triangle_y f(\cdot,y)\right\|_{L^2(\Omega_x)}^2\right]$ in training process (left) and their final results (right). In training process (left), the dashed lines represent the results of $\lambda=5\times10^{-3}$ and the solid lines represent $\lambda=10^{-2}$. The final results (right), shows that trade-off between the model accuracy and stability metric. }
	\label{fig:error_WD}
\end{figure}

\paragraph{Regularization by introducing noise to the data}
Motivated by the analysis for LNNs in Section~\ref{sec:effectnoise}, we study numerically the potential of stabilization by adding noise to the data set. 
We follow the setup introduced above, except that we have  noisy data $\bm x_i \sim M_\sigma$, i.e., $
\bm x_i = \begin{pmatrix}
	x_i \\ \sigma y_i
\end{pmatrix}$,
where $x_i \sim X$ and $y_i \sim N(0,I_{d_y})$. 
In addition, we take $\eta = 1$, i.e. $W^1_{i,y} = (1,1)$, to initialize $W^1_{i,y}$. 

The stability metric $\mathbb E_{y} \left[\left\|\triangle_y f(\cdot,y)\right\|_{L^2(\Omega_x)}^2\right]$ is evaluated with 
$y \sim N(0,\gamma I_{d_y})$, $\gamma=2^{-2}$, and approximated by
summation of $M=8\times 10^3$ independent samples for the case of comparing different $\sigma$ (different level of added noise) and $M = 4\times10^{4}$ samples for the case of comparing training sets of different cardinality, $N$.

The curves shown in Figure~\ref{fig:DNN_DN_L} are obtained by averaging 5 individual tests. Figure~\ref{fig:DNN_DN_L} verifies our conjectures about stabilization effect of noise to the normal direction of data manifold and increasing the data points. More discussion about these results will be presented in the following section.
\begin{figure}[h]
	\centering
	\includegraphics[scale=0.8]{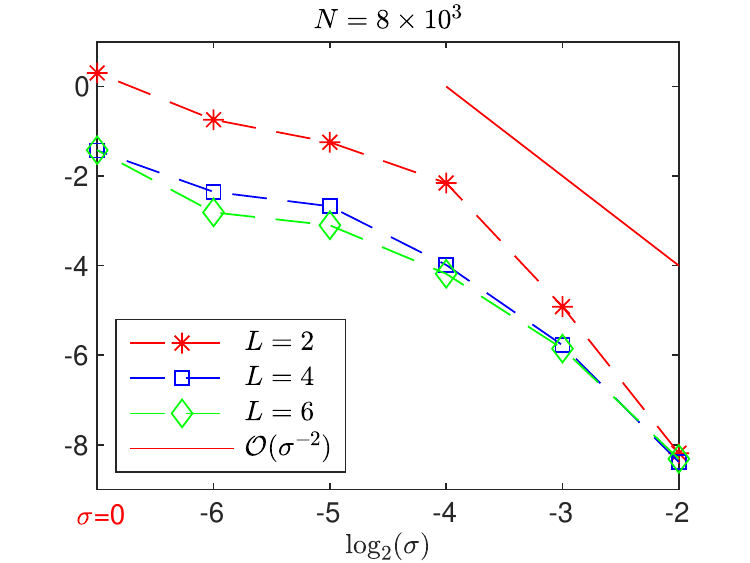}\includegraphics[scale=0.8]{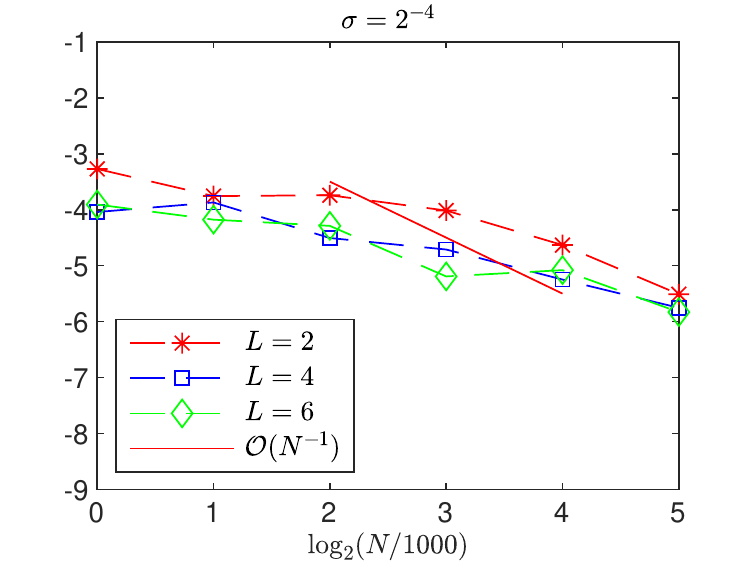}
	\caption{$\log_2\left(\mathbb E_{y} \left[\left\|\triangle_y f(\cdot,y)\right\|_{L^2(\Omega_x)}^2\right]\right)$ as a function of $\sigma$, the variance in the noise added to the training data (left).  $\log_2\left(\mathbb E_{y} \left[\left\|\triangle_y f(\cdot,y)\right\|_{L^2(\Omega_x)}^2\right]\right)$ as a function of the size of training data set, $N$ (right).}
	\label{fig:DNN_DN_L}
\end{figure}

\section{Stability from adding noise to the data manifold}\label{sec:trade-off}
In this section, we consider on a more abstract level the effects of adding noise to the embedded low dimensional data. 
The aim is to improve the trained neural network's stability, evaluating points that lie out of the training data distribution.
We have seen in the previous sections that adding noise may regularize the optimization problem in some sense and provide stability. In the following we shall relate adding noise in the normal directions of the given data manifold to implicitly defining an extension of the loss function \eqref{eq:optimization_model_J}.  
The change in the loss function subsequently
enables the learning function to approximate 
the constant normal extension, $\overline{g}$ as defined in \eqref{eq:g_bar}, of the label function $g$. 
This view provides a more intuitive explanation of how adding noise according to the geometry of the data may enhance the stability of a trained network, provided that the data set is sufficiently large. 

\subsection{Implicit extension of the loss functional}
Let $\mathcal M$ be a $d_x$-dimensional compact  $C^2$-manifold in $\mathbb{R}^d$.
Denote by $N_{\bm x}M$ the normal space of $\mathcal M$ at $\bm x\in \mathcal{M}$ and $r>0$  the reach of $\mathcal{M}$. 
For any  $\sigma\in(0,r)$ we introduce the $\sigma$-tubular neighborhood of $\mathcal M$ in $\mathbb{R}^d$  as
$$
T_\sigma := \left\{ \bm x + \epsilon \bm n_x~:~\bm x \in \mathcal M, \epsilon \in (-\sigma, \sigma),  \bm n_x  \in N_{\bm x}M , \text{ and } \|\bm n_x\| = 1 \right\}.
$$
For points in $T_\sigma $, define the projection 
$$\mathcal{P}_{\mathcal M} {\bm x} = \mathop{\arg\inf}_{\bm \xi \in \mathcal M} \| {\bm x} - \bm \xi\|.$$
Now, let $U[\mathcal M]$ denote the uniform distribution defined on $\mathcal M$;  i.e. the density of $U[\mathcal{M}]$ is uniform with respect to the measure on $\mathcal{M}$, induced by the Euclidean norm of $\mathbb{R}^d$. 
To each data point $\bm x$ sampled independently from $U[\mathcal{M}]$ we introduce  noise that lifts $\bm x$ to $\widetilde{\bm x}$ in the normal space $N_{\bm x}M$. More precisely,
$$
\widetilde{\bm x} = \bm x + \epsilon \bm n_x,
$$
where $\bm n_x \in N_{\bm x} M$ is sampled from the uniform distribution on $\mathbb S^{d_y-1}$ embedded in $N_{\bm x}M$ and $\epsilon \sim U[-\sigma, \sigma]$ with $\sigma < r$. 
We shall denote the resulting joint distribution as $M_\sigma$ and its density $\rho_{\sigma}$.
Thus, $\tilde{\bm x}$ is a point in $T_\sigma$, sampled from $M_\sigma$. According to the coarea formula, $\rho_{\sigma}$ is uniform on $T_\sigma$ only if $\mathcal{M}$ is flat. See  \cite{kublik2013implicit,chu2018volumetric} for the case when $\mathcal{M}$ is a hypersurface.

The loss function defined with the noisy data $\left\{ \left(\widetilde{\bm x}_i, g_i\right) \right\}_{i=1}^N$ ($g_i=g(\bm x_i)$) can be written as
\begin{equation}\label{eq:noisy_loss}
	J(\theta)=\frac{1}{2N} \sum_{i=1}^N\left| f_\theta(\widetilde{\bm x}_i) - g_i\right|^2 = \frac{1}{2N} \sum_{i=1}^N\left| f_\theta(\widetilde{\bm x}_i) - g(\mathcal{P}_\mathcal{M} \widetilde{x}_i)\right|^2.
\end{equation}
In other words, $J$ can be interpreted as the empirical loss of the following continuous loss
\begin{equation}\label{eq:noisy_loss_continuous}
	\overline{\mathcal{J}}(\theta) := \frac{1}{2} \int_{T_\sigma } \left|
	f_\theta({\bm x}) - \overline g({\bm x}) \right|^2 \rho_\sigma (\bm x) d {\bm x}.
\end{equation}
where 
$
\overline g({\bm x}) := g\left(\mathcal P_{\mathcal M} {\bm x}\right),
$
is the constant extension of $g(\bm x)$ along the normal directions.  
This implies that the ``regularization" effect from using this type of noisy data is
the ``automatic learning" of $\overline{g}$ on $\mathbb{R}^d$.

\subsection{Accuracy/stability trade-off}
Assuming that we do not know the geometry of the data manifold so we add noise to every component in the ambient space indifferently.

For simplicity, we assume the data set $D_N = \{(\bm x_i, g_i)\}_{i=1}^N$ consists of
\begin{equation*}
	\bm x_i = \begin{pmatrix}
		x_i \\ 0
	\end{pmatrix} + \sigma \begin{pmatrix}
		\epsilon_{i,x} \\ \epsilon_{i,y}
	\end{pmatrix} \in \mathbb R^{d_x + d_y},
\end{equation*}
where $(\epsilon_{i,x}, \epsilon_{i,y}) = \bm \epsilon_i \sim N(0,I_{d_x+d_y})$ sampled as the noise part.
In addition, the ``label" in the data are clean and followed by $g_i = g(x_i)$ for every $x_i \in \mathbb R^{d_x}$. 

Thus, we have
\begin{equation*}
	g_i = g(x_i + \sigma\epsilon_{i,x} - \sigma\epsilon_{i,x}) = g(x_i + \sigma\epsilon_{i,x}) + \mathcal O\left(\sigma\right).
\end{equation*}
This means we can interpret the noisy data as
\begin{equation}\label{eq:noise-added-data}
	(\bm x_i, g_i) = \left(\begin{pmatrix}x_i + \sigma\epsilon_{i,x} \\ \sigma\epsilon_{i,y}\end{pmatrix}, g(x_i)\right) = \left(\begin{pmatrix}\tilde x_i \\ \sigma\epsilon_{i,y}\end{pmatrix}, g(\tilde x_i) + \mathcal O\left(\sigma\right) \right),
\end{equation}
where $\tilde x_i = x_i + \sigma\epsilon_{i,x}$. Thus, for any trained machine learning model $f(x,y)$, 
we can decompose the generalization error as
\begin{equation*}
	\left\|f(x,y) - gx)\right\|^2 \le  \left\|f(x,0) - g(x)\right\|^2 + \left\|f(x,y) - f(x,0)\right\|^2.
\end{equation*}
The interpolation error $\|f(x,0) - g(x)\|^2$ corresponds to the error for the classical learning task 
with noisy label data $(\tilde x_i, \tilde g_i)$ where $\tilde x_i = x_i + \sigma \epsilon_{i,x} \sim \widetilde X := X + \sigma N(0,I_{d_x}) \in R^{d_x}$ and $\tilde g_i = g(\tilde x_i)+\mathcal O\left(\sigma\right) \in \mathbb R$. 
Then it follows that 
\begin{equation}\label{eq:normal-fitting-error}
	\begin{aligned}
		\mathbb E_{x\sim X} \left[\|f(x,0) - g(x)\|^2 \right] 
		&\le \mathbb E_{\tilde x\sim \widetilde X} \left[\|f(\tilde x,0) - g(\tilde x)\|^2 \right] + \mathcal O(\sigma^2) \\
		&=\mathbb E_{\tilde x\sim \widetilde X}\left[ \|f(\tilde x,0) - \tilde g(\tilde x) + \mathcal O(\sigma) \|^2\right] + \mathcal O(\sigma^2)\\
		&\le 
		\left<\|f(\tilde x,0)-\tilde g\|^2\right>_N + \mathcal O\left( \sigma^2\right) + \mathcal O\left(N^{-1}\right).
	\end{aligned}
\end{equation}
Here, $\left<\|f(\tilde x,0)-\tilde g\|^2\right>_N$ is the empirical loss which can be bounded by the approximation power of
one-hidden-layer ($L=2$) neural networks~\cite{barron1993universal,siegel2021sharp} and deep ($L>2$) neural networks~\cite{yarotsky2017error,shen2022optimal}.

The stability metric, by which we mean $\left\|\triangle_yf(x,y)\right\|=\|f(x,y) - f(x,0)\|^2$, for LNNs 
is estimated to be $\mathcal{O}\left( \left(\sigma^2 N\right)^{-1}\right).$ 
If $\epsilon_{i,x} = 0$, the reciprocal relation between $\left\|\triangle_yf(x,y)\right\|^2$ and variance $\sigma^2$ is observed in ReLU DNNs in Figure~\ref{fig:DNN_DN_L}. 
However, Figure~\ref{fig:DNN_DN_L} suggests that
$\left\|\triangle_yf(x,y)\right\|^2$ for ReLU DNNs is reciprocal to ${N^{\beta}}$ with $\beta<1$, in contrast to $\beta=1$ in the case of LNNs.

We present Table~\ref{tab:my}, which summarizes a series of further numerical experiments and reveals how $\beta$ is related to network's depth and the co-dimensions, $d_y$, of the data manifold. 
In the table, $\beta$ is fitted by using the linear regression for $\|\triangle_yf(x,y)\|^2$ and $N$ in the logarithmic scale.
\begin{table}
	\centering
	\caption{Linear regression results of $\beta$ with different $L$ and co-dimension $d_y$.}\medskip
	\label{tab:my}
	\begin{tabular}{|c|c|c|c|c|c|c|c|c|c|c|}
		\hline
		$d_y$ & 1 & 2 & 3 & 4 & 5 & 6 & 7 & 8 & 9 & 10 \\
		\hline
		$L=2$ & 0.40 & 0.40 & 0.35 & 0.36 & 0.35 & 0.39 & 0.38 & 0.33 & 0.38 & 0.37 \\
		\hline
		$L=4$ & 0.42 & 0.37 & 0.40 & 0.34 & 0.43 & 0.39 & 0.42 & 0.43 & 0.43 & 0.41 \\
		\hline
		$L=6$ & 0.46 & 0.38 & 0.40 & 0.44 & 0.44 & 0.41 & 0.43 & 0.46 & 0.46 & 0.46 \\
		\hline
	\end{tabular}
\end{table}

From the table, we find that
\begin{equation}\label{eq:error-in-the-normal}
	\left\|\triangle_yf(x,y)\right\|^2 \approx \mathcal O\left(\frac{1}{\sigma^2 N^{\beta}}\right),
\end{equation}
where $\beta<1/2$ seems to relate to the depth of the network, but independent of the co-dimension of the data manifold.
Again, the experimental results are quite different from LNN case. The results suggest that nonlinear ReLU networks require more 
training data to control the variation of the neural networks in the $y$-directions (for small $\sigma$).

For any fixed data set (fixed $N$) \eqref{eq:normal-fitting-error} and \eqref{eq:error-in-the-normal} describe a trade-off between accuracy and stability: On the one hand, reducing the fitting errors in \eqref{eq:normal-fitting-error} requires smaller noise level for the $x$-components. On the other hand, small noise level in the $y$-direction will decrease the stability of $f$ in the $y$-direction. 
However, if the data manifold is not flat, the geometry of the manifold will impose an additional constraint to the maximal noise level. Too large of a noise level will lead to ill-conditioned optimization problem. 

\section{Summary}\label{sec:conclusion}
Surprising features in supervised learning problems arise when data are embedded in a high dimensional Euclidean space.
We derived estimates on the derivatives of the learning function in the direction transversal to the data subspace. 
When a neural network defines the learning function, a portion of its weights is untrainable by a typical gradient descent-based algorithm because the empirical loss function is independent of these weights.
Consequently, the learning function's values at points away from the data subspace depend on the initialization of the untrainable weights.

We showed that if noise in the codimension of the data subspace is present, the weights in question can be controlled, provided that the training data size is sufficiently large. However, the training data size only has to be large compared with the standard deviation $\sigma$, and seems independent of the number of codimensions.
For linear networks, we have shown that the price for this regularization is the slow convergence for those weights to small numbers. We have also demonstrated that the network's depth may provide a particular regularization effect if the network's weights are initialized in a suitable subregion of $\mathbb R^d$. 
For nonlinear networks activated by ReLU, similar to LNNs, there is still a set of parameters that are not trainable if the data subspace has non-zero number of codimensions.
We derived a stability estimate for the influence of the untrainable weights in a trained neural network.

Though adding noise to the data set may provide a desired regularization to the learning function, it also incurs a trade-off to the accuracy of the trained network and possibly renders the optimization model ill-conditioned, when the data manifold is not flat. It is also clear that if one has more information about the geometry of the data manifold, one can introduce noise adaptively according to the manifold's geometry and mitigate the loss of accuracy.

\section*{Acknowledgment}
The authors thank Lukas Taus for his help with the numerical experiments in Fig.~\ref{MNIST}. Tsai's research is supported partially by National Science Foundation Grants DMS-2110895 and by Army Research Office, under Cooperative Agreement Number W911NF-19-2-0333. Ward's research is supported in part by AFOSR MURI FA9550-19-1-0005, NSF DMS 1952735, NSF HDR-1934932, and NSF 2019844.  The views and conclusions contained in this document are those of the authors and should not be interpreted as representing the official policies, either expressed or implied, of the Army Research Office or the U.S. Government. The U.S. Government is authorized to reproduce and distribute reprints for Government purposes notwithstanding any copyright notation herein.

\bibliographystyle{abbrv}      
\bibliography{LowDimNN.bib}   

\newpage
\appendix
\section{Proofs for LNNs}
\subsection{Proof of Proposition~\ref{prop:rota_inva}}
\begin{proof}
	We first show that the loss function
	can be transformed naturally under the unitary mapping $Q$. 
	The original loss function can be formulated as
	\begin{equation*}
		J^e(\bm w) 
		=  \frac{1}{2N}\sum_{i=1}^N( \bm w^T Q \begin{pmatrix}
			x_i \\ \sigma y_i
		\end{pmatrix} - g_i)^2 = \frac{1}{2N}\sum_{i=1}^N( (Q^T \bm w)^T\begin{pmatrix}
			x_i \\ \sigma y_i
		\end{pmatrix} - g_i)^2 .
	\end{equation*}
	Thus, if we denote
	\begin{equation*}
		\widetilde{\bm w} = Q^T \bm w,
	\end{equation*}
	we can define the new loss function $\widetilde J^e(\widetilde{\bm w})$ with respect to the new variable $\widetilde{\bm w}$ as
	\begin{equation*}
		\widetilde J^e(\widetilde{\bm w}) = J^e(\bm w) = 
		\frac{1}{2N}\sum_{i=1}^N( \widetilde{\bm w}^T
		\begin{pmatrix}
			x_i \\ \sigma y_i
		\end{pmatrix} - g_i)^2.
	\end{equation*}
	In addition, by taking the gradient for $J^e(\bm w)$ with respect
	to $\bm w$, we have
	\begin{equation*}
		\nabla_{\bm w} J^e(\bm w) = Q\nabla_{\widetilde{\bm w}} \widetilde J^e(\widetilde{\bm w}).
	\end{equation*}
	Furthermore, we claim that the dynamic system for $w$ can be rotated to $\widetilde{\bm w}$ naturally. First, we can check
	\begin{equation*}
		\mathcal{P}_{\bm w}(\bm v) = \frac{\bm w \bm w^T}{\|\bm w\|^2} \bm v 
		= \frac{Q\widetilde{\bm w} \widetilde{\bm w}^T Q^T}{\|\widetilde{\bm w}\|^2} \bm v 
		= Q \mathcal{P}_{\widetilde{\bm w}}(Q^T \bm v).
	\end{equation*}
	Based on the dynamical system for $w$, we have
	\begin{equation*}
		\begin{aligned}
			\frac{d}{dt} \bm w &= -\|\bm w\|^{2 - \frac{2}{L}} \left(\nabla_{\bm w} J^e(\bm w) + (L-1) \mathcal{P}_{ \bm w}(\nabla_{\bm w} J^e(\bm w) ) \ \right) \\
			&= -\|\widetilde{\bm w}\|^{2 - \frac{2}{L}} 
			\left(Q\nabla_{\widetilde{\bm w}}\widetilde J^e(\widetilde{\bm w}) + (L-1) Q\mathcal{P}_{\widetilde{\bm w}}(Q^T Q\nabla_{\widetilde{\bm w}} \widetilde J^e(\widetilde{\bm w}) ) \right) \\
			&=-\|\widetilde{\bm w}\|^{2 - \frac{2}{L}} Q
			\left(\nabla_{\widetilde{\bm w}}\widetilde J^e(\widetilde{\bm w}) + (L-1) \mathcal{P}_{\widetilde{\bm w}}(\nabla_{\widetilde{\bm w}} \widetilde J^e(\widetilde{\bm w}) ) \right) .
		\end{aligned}
	\end{equation*}
	Finally, we can see that 
	\begin{equation*}
		\frac{d}{dt}{\widetilde{\bm w}} = Q^T \frac{d}{dt} \bm w 
		=-\|\widetilde{\bm w}\|^{2 - \frac{2}{L}}
		\left(\nabla_{\widetilde{\bm w}}\widetilde J^e(\widetilde{\bm w}) + (L-1) \mathcal{P}_{\widetilde{\bm w}}(\nabla_{\widetilde{\bm w}} \widetilde J^e(\widetilde{\bm w}) ) \right) .
	\end{equation*} 
\end{proof}

\subsection{Proof of Proposition~\ref{prop:w-stationary}}
\begin{proof}
	Since
	\begin{equation*}	
		\mathcal{P}_{\bm w}\left( \bm v \right) = \frac{\bm w^T \bm v}{\|\bm w\|^2}\bm w= \frac{\bm w \bm w^T}{\|\bm w\|^2} \bm v,
	\end{equation*}
	we have
	\begin{equation*}	
		\begin{aligned}
			\bm F(\bm w) &= -\|\bm w\|^{-\frac{2}{L}} \left( \|\bm w\|^2\nabla J^e(\bm w)  + (L-1) \bm w\bm w^T  \nabla J^e(\bm w)\right) \\
			&= -\|\bm w\|^{-\frac{2}{L}} \left( \left(\|\bm w\|^2 I_{d} + (L-1)  \bm w\bm w^T\right) \nabla J^e(\bm w)  \right) \\
			&= -\|\bm w\|^{-\frac{2}{L}} \bm M  \nabla J^e(\bm w),
		\end{aligned}
	\end{equation*}
	where $\bm M = \|\bm w\|^2 I_{d}  + (L-1)  \bm w\bm w^T \in \mathbb{R}^{d\times d}$ is a symmetric positive
	definite matrix if $\bm w \neq \bm 0$. Thus, $\bm F(\bm w) = 0$ if and only if $\bm w = \bm 0$ or $\nabla J^e(\bm w) = \mathbf{0}$.
	
	If $L>2$ and $J^e(\bm w)$ is strictly convex, there is a unique $\bm w^* \neq 0$ ( since $\left<\bm x g\right>_N \neq 0$) such that 
	$\nabla J^e(\bm w^*) = 0$ and the Hessian matrix $\nabla^2 J^e(\bm w^*) $ is a symmetric positive
	definite (SPD) matrix. Then, the Jacobian matrix of $\bm F(\bm w)$ at $\bm w^*$ is
	\begin{equation*}
		\nabla \bm F(\bm w^*) = -\|\bm w^*\|^{-\frac{2}{L}} \bm M(\bm w^*)  \nabla^2 J^e(\bm w^*).
	\end{equation*}
	Given $\bm w^* \neq 0$ and both $\bm M(\bm w^*)$ and $\nabla^2 J^e(\bm w^*)$ are SPD, we notice that 
	\begin{equation*}
		\nabla \bm F(\bm w^*) \sim \bm M^{-\frac{1}{2}}(\bm w^*) \nabla \bm F(\bm w^*) \bm M^{\frac{1}{2}}(\bm w^*) = -\|\bm w^*\|^{-\frac{2}{L}} \bm M^{\frac{1}{2}}(\bm w^*)  \nabla^2 J^e(\bm w^*) \bm M^{\frac{1}{2}}(\bm w^*),
	\end{equation*}
	which shows that all eigenvalues of $\nabla \bm F(\bm w^*)$ are negative. Since $L>2$, the leading order of $\bm F(\bm w)$ is $\mathcal O(\|\bm w\|^{2 - \frac{2}{L}})$ which is continuously differentiable at $\bm w = \bm 0$ with $\nabla \bm F(\bm w) = 0$. 
\end{proof}

\subsection{Proof of Corollary~\ref{coro:Cg}}
\begin{proof}
	From formula~\ref{eq:w*} in  Proposition~\ref{prop:critical-points-of-Le}, we have $\left<A\right>_N\begin{pmatrix}
		w_x^* \\ \sigma w_y^*
	\end{pmatrix}=\begin{pmatrix}
		\left< gx\right>_N \\ \left< gy\right>_N
	\end{pmatrix}$, where
	$(w_x^*, w_y^*)$ is the solution
	with respect to the target function $g(x)$.
	Given $g(x) = \tilde g(x) + \mu^T x$, 
	we have
	\begin{equation*}
		\left<A\right>_N\begin{pmatrix}
			w_x^* \\ \sigma w_y^*
		\end{pmatrix}
		= \begin{pmatrix}
			\left< gx\right>_N \\ \left< gy\right>_N
		\end{pmatrix}
		= \begin{pmatrix}
			\left< \tilde gx\right>_N \\ \left< \tilde gy\right>_N
		\end{pmatrix}
		+ \begin{pmatrix}
			\left< (\mu^Tx)x\right>_N \\ \left< (\mu^Tx)y\right>_N
		\end{pmatrix}.
	\end{equation*}
	Since 
	\begin{equation*}
		\left<A\right>_N\begin{pmatrix}
			\mu \\ 0
		\end{pmatrix}
		= \begin{pmatrix}
			\left< (\mu^Tx)x\right>_N \\ \left< (\mu^Tx)y\right>_N
		\end{pmatrix},
	\end{equation*}
	we have
	\begin{equation*}
		\left<A\right>_N\begin{pmatrix}
			w_x^* - \mu \\ \sigma w_y^*
		\end{pmatrix}
		=\begin{pmatrix}
			\left< \tilde gx\right>_N \\ \left< \tilde gy\right>_N
		\end{pmatrix}.
	\end{equation*}
	
	Recalling the block structure of $\left<A\right>_N$ and applying $\left<A\right>_N^{-1}$ on both sides of the above equation, we have
	\begin{equation}\label{eq:wx-nu}
		w_x^* - \mu = \widetilde \Sigma_X^{-1}\left(\left<\tilde g x\right>_N + \left<x y^T\right>_N S^{-1} \left<yx^T\right>_N \Sigma_X^{-1} \left<\tilde g x\right>_N - \left<xy^T\right>_N S^{-1} \left<\tilde g y\right>_N   \right)
	\end{equation}
	and
	\begin{equation}\label{eq:wy-nu}
		w_y^* = \sigma^{-1}S^{-1}\left(\left<\tilde g y\right>_N - \left<yx^T\right>_N \widetilde \Sigma_X^{-1}\left<\tilde gx\right>_N \right).
	\end{equation}
	Given $|\widetilde g(x)| \le \delta$ and the estimates in Theorem~\ref{thm:w*scale}, we have $\|\widetilde \Sigma_X^{-1}\|, \|S^{-1}\|  \sim \mathcal O(1)$, 
	$\|\left<\tilde g x\right>_N\|\lesssim \delta$, $\|\left<\tilde g y\right>_N\| \lesssim \frac{\delta}{\sqrt{N}}$, $\|\left<xy^T\right>_N \widetilde \Sigma_X^{-1}\left<\tilde gx\right>_N \| \lesssim \frac{\delta}{\sqrt{N}}$, 
	$\|\left<yx^T\right>_N \widetilde \Sigma_X^{-1}\left<\tilde gx\right>_N\| \lesssim \frac{\delta}{\sqrt{N}}$, and 
	$\|\left<x y^T\right>_N S^{-1} \left<yx^T\right>_N \Sigma_X^{-1} \left<\tilde g x\right>_N\|\lesssim \frac{\delta}{N}$. Here $\lesssim$ means that there is a constant which depends only on the distribution $X$ and $Y$.
	Then, the results can be obtained by taking norm in \eqref{eq:wx-nu} and \eqref{eq:wy-nu} and then substitute the previous estimates. 
\end{proof}

\subsection{Proof of Theorem~\ref{thm:escapetime-multiD}}
\begin{proof} 
	Let us denote 
	\begin{equation*}
		f_{\rm max} := \sup_{\bm w \in B_{\frac{\epsilon}{2}}(\bm w(0))} \left\| \frac{d}{dt} \bm w\right\|,
	\end{equation*}
	where $B_{\frac{\epsilon}{2}}(\bm w(0)) := \{\bm w ~:~ \|\bm w-\bm w(0)\| \le \frac{\epsilon}{2} \}$ is the $\frac{\epsilon}{2}$-ball centered at $\bm w(0)$.
	Given the continuity of $\bm w(t)$ and the definition of $T_L(\epsilon)$, we have
	\begin{equation*}
		\frac{\epsilon}{2} = \|\bm w(T_L(\epsilon))-\bm w(0)\| = \left\| \int_{0}^{T_L(\epsilon)}\frac{d}{dt} \bm w dt\right\| \le f_{\rm max} T_{L}(\epsilon).
	\end{equation*}
	It follows that
	\begin{equation*}
		T_{L}(\epsilon) \ge \frac{\epsilon}{2f_{\rm max}}.
	\end{equation*}
	For $f_{\rm max}$, we notice that
	\begin{equation*}
		\begin{split}
			\left\| \frac{d}{dt} \bm w\right\| &= \|\bm w\|^{2 - \frac{2}{L}}\left\|  \nabla_{\bm w} J^e(\bm w) + (L-1) \mathcal{P}_{ \bm w}(\nabla_{\bm w} J^e(\bm w) ) \right\| \\
			&\le L \|\bm w\|^{2 - \frac{2}{L}} \left\|  \nabla_{\bm w} J^e(\bm w) \right\|  \\
			&=  L \|\bm w\|^{2 - \frac{2}{L}} \left\| \left<A_\sigma\right>_N \bm w - \left<g\bm x\right>_N \right\| \\
			&\le L\|\bm w\|^{2 - \frac{2}{L}} \left( C_2\|\bm w\| + \|\left<g\bm x\right>_N\| \right).
		\end{split}
	\end{equation*}
	Since $\bm w\in B_{\frac{\epsilon}{2}}(\bm w(0))$, $\|\bm w(0)\|=\epsilon$, and $\epsilon <<1$,
	there exists $C$ depends on $L$, $C_2$, and $\|\left<g\bm x\right>_N\|=\mathcal O(1)$ such that
	\begin{equation*}
		L\|\bm w\|^{2 - \frac{2}{L}} \left( C_2\|\bm w\| + \|\left<g\bm x\right>_N\| \right) \le 
		\frac{1}{2C} \|\epsilon\|^{2-\frac{2}{L}}.
	\end{equation*}
	This means $f_{\rm max} \le \frac{1}{2C} \|\epsilon\|^{2-\frac{2}{L}}$, which finished the proof. 
\end{proof}

\subsection{Proof of Proposition~\ref{prop:E-}}
Before we show the proof of Proposition~\ref{prop:E-}, let us first present the following lemma.
\begin{lemma}\label{lem:wAu}
	Let $A \in \mathbb R^{d\times d}$ be a symmetric positive definite (SPD) matrix with $d\ge2$ and assume $a_1 \ge a_2 \ge \cdots \ge a_d$ are the its eigenvalues. Then, we have
	\begin{equation*}
		w^TAu \ge \frac{a_d - a_1}{2}\|w\|\|u\|,
	\end{equation*}
	if $w, u\in \mathbb R^d$ and $w^Tu=0$.
\end{lemma}
\begin{proof}
	First, we may assume the SVD decomposition for $A$ as $A = V^T \Sigma V$,
	where $V$ is a unitary matrix and $\Sigma = {\rm diag}(a_1, a_2, \cdots, a_d)$. 
	By denoting $Vw = \widetilde w$, $Vu = \widetilde u$, and $\widetilde w_i \widetilde u_i = b_i$, we have
	\begin{equation*}
		w^TAu = (Vw)^T \Sigma (Vu)
		= \sum_{i=1}^d a_i \widetilde w_i \widetilde u_i = \sum_{i=1}^d a_i b_i.
	\end{equation*}
	Let us denote $\sigma$ as the permutation of $\{1, 2, \cdots , d\}$ such that
	\begin{equation*}
		b_{\sigma(1)} \le b_{\sigma(2)} \le \cdots \le b_{\sigma(d)}.
	\end{equation*}
	Here, we notice that 
	\begin{equation*}
		\sum_{i=1}^d b_i = \sum_{i=1}^d \widetilde w_i \widetilde u_i = (\widetilde w)^T \widetilde u= (Vw)^T (Vu)= w^Tu = 0.
	\end{equation*}
	Thus, there is at least one positive integer $k$ such that $b_{\sigma(k)} \le 0$ and $b_{\sigma(k+1)} \ge 0$. That is, 
	\begin{equation*}
		\sum_{i=1}^k -b_{\sigma(i)} = \sum_{i>k}^d b_{\sigma(i)} = \frac{1}{2} \sum_{i=1}^d |b_i|.
	\end{equation*}
	By using the rearrangement inequality, we have
	\begin{equation*}
		\begin{aligned}
			\sum_{i=1}^d a_i b_i &\ge \sum_{i=1}^d a_i b_{\sigma(i)}  = \sum_{i=1}^k (-a_i) \left(-b_{\sigma(i)}\right) + \sum_{i>k}^d a_i b_{\sigma(i)} \\
			&\ge \sum_{i=1}^k (-a_1) \left(-b_{\sigma(i)}\right) + \sum_{i>k}^d a_d b_{\sigma(i)} = (a_d - a_1) \sum_{i=1}^k \left(-b_{\sigma(i)}\right) \\
			&= \frac{a_d-a_1}{2} \sum_{i=1}^d |b_i| = \frac{a_d-a_1}{2} \sum_{i=1}^d \left|\widetilde w_i \widetilde u_i \right| \\
			&= \frac{a_d-a_1}{2} |\widetilde w|^T |\widetilde u| \ge \frac{a_d-a_1}{2} \|\widetilde w\| \|\widetilde u\|  = \frac{a_d-a_1}{2} \|w\| \|u\|,
		\end{aligned}
	\end{equation*}
	where $|v| = (|v_1|, |v_2|, \cdots, |v_d|)$ for any $v \in \mathbb R^d$. 
\end{proof}

Now, we have the following proof for Proposition~\ref{prop:E-}.

\begin{proof}
	Given $\Sigma_X = cI_{d_x}$, we first denote that
	$A := \left< xx^T\right>_N = \Sigma_X + (\left< xx^T\right>_N-\Sigma_X) = cI_{d_x} + \|\Sigma_X - \left< xx^T\right>_N\| B$ with $\|B\| = 1$.
	According to the convergence of correlated matrix~\cite{adamczak2010quantitative,cai2010optimal}, we have $0 \le \epsilon := \|\Sigma_X - \left< xx^T\right>_N\| \le \frac{c}{2}$ with high probability if $N$ is large enough. 
	That is, we have 
	\begin{equation}\label{eq:A_epsilon}
		(c-\epsilon)\|u\|^2    \le u^TAu \le (c+\epsilon) \|u\|^2 \quad \text{and} \quad a_d - a_1 \ge -2\epsilon,
	\end{equation}
	for any $u\in \mathbb R^{d_x}$ if $N$ is large enough. Here, $a_1 \ge a_2 \ge \cdots \ge a_{d_x}$ denote the eigenvalues of $A$. 
	
	For simplicity, it is equivalent to prove that
	$\frac{d}{dt} \bm w \cdot n_{E} \le 0$ for any $\bm w = (w_x,w_y) \in E$ and $\|w_x - w_x^* \| \ge \frac{2\sqrt{3}}{c}\|w_x^*\|\epsilon$ for any $0 \le \epsilon \le \frac{c}{2}$, where $n_E$ denotes the exterior normal direction of $E$ at $\bm w$. 
	
	In addition, we recall that $\bm w \in E$ if and only if $w_x^T\left(A(w_x - w_x^*)\right) = 0$. Thus, for any $\bm w \in E$, we have
	\begin{equation*}
		\begin{aligned}
			\left(\frac{\sqrt{c+\epsilon}}{2}\|w_x^*\|\right)^2 &\ge\frac{(w_x^*)^TAw_x^*}{4} = \left(w_x-\frac{w_x^*}{2}\right)^TA\left(w_x-\frac{w_x^*}{2}\right) \ge \left(\sqrt{c-\epsilon}\|w_x -\frac{w_x^*}{2}\|\right)^2,
		\end{aligned}
	\end{equation*}
	which leads to 
	\begin{equation*}
		\frac{\sqrt{c+\epsilon}}{2}\|w_x^*\| \ge \sqrt{c-\epsilon}\|w_x -\frac{w_x^*}{2}\| \ge \sqrt{c-\epsilon}\left(\|w_x\| - \frac{\|w_x^*\|}{2}\right).
	\end{equation*}
	That is, we have 
	\begin{equation}\label{eq:estimate_wx}
		\|w_x\| \le \frac{\sqrt{c+\epsilon} + \sqrt{c-\epsilon}}{2\sqrt{c-\epsilon}} \|w_x^*\| \le \frac{\sqrt{c+\epsilon}}{\sqrt{c-\epsilon}} \|w_x^*\|
	\end{equation}
	for any $\bm w \in E$.
	
	Now, let us check the sign of $\frac{d}{dt} \bm w(t) \cdot n_{E}$ if $\bm w(t) \in E$ while $\|w_x(t) - w_x^*\|\ge \frac{2\sqrt{3}}{c}\|w_x^*\|\epsilon.$
	First, we notice that 
	$$
	n_E = \nabla_{\bm w}\left( w_x^T\frac{\partial J^e}{\partial w_x} \right) = 
	\begin{pmatrix}
		A(2w_x - w_x^*) \\
		0
	\end{pmatrix}.
	$$
	Thus, we have
	\begin{equation*}
		\begin{aligned}
			\frac{d}{dt} \bm w \cdot n_{E}  &= \left(A(2w_x - w_x^*)\right)^T \frac{d}{dt} w_x \\
			&= -\|\bm w\|^{2-\frac{2}{L}}\left(A(2w_x - w_x^*)^TA(w_x-w_x^*)\right) \\
			&= -\|\bm w\|^{2-\frac{2}{L}}\left((A(w_x - w_x^*))^TA(w_x-w_x^*) + w_x^T A^2(w_x - w_x^*)\right) \\
			&= -\|\bm w\|^{2-\frac{2}{L}}\left(\|A(w_x-w_x^*)\|^2 + w_x^T A^2(w_x - w_x^*)\right).
		\end{aligned}
	\end{equation*}
	for any $\bm w \in E$.  Give the Lemma~\ref{lem:wAu} and \eqref{eq:A_epsilon}, we have
	\begin{equation}\label{eq:estimate_wA2w}
		w_x^T A^2(w_x - w_x^*) = w_x^T A \left(A(w_x - w_x^*)\right) \ge -\epsilon \|w_x\|\|A(w_x - w_x^*)\|,
	\end{equation}
	since $w_x^T\left(A(w_x - w_x^*)\right) = 0$. 
	In the end, by combining the Lemma~\ref{lem:wAu}, \eqref{eq:A_epsilon}, \eqref{eq:estimate_wx}, and \eqref{eq:estimate_wA2w}, we have
	\begin{equation*}
		\begin{aligned}
			&\|A(w_x-w_x^*)\|^2 + w_x^T A^2(w_x - w_x^*) \\
			\ge &\|A(w_x-w_x^*)\|\left(\|A(w_x-w_x^*)\| - \epsilon \|w_x\| \right)\\
			\ge &\|A(w_x-w_x^*)\|\left((c-\epsilon)\|w_x-w_x^*\| - \epsilon \frac{\sqrt{c+\epsilon}}{\sqrt{c-\epsilon}} \|w_x^*\|\right) \\
			\ge &\|A(w_x-w_x^*)\|\left(\frac{c}{2}\|w_x-w_x^*\| - \epsilon \sqrt{3} \|w_x^*\|\right) \ge 0
		\end{aligned}
	\end{equation*}
	since $0 \le \epsilon \le \frac{c}{2}$ and $\|w_x-w_x^*\|  \ge \frac{2\sqrt{3}}{c}\|w_x^*\|\epsilon$. This finishes the proof. 
\end{proof}

\subsection{Proof of Lemma~\ref{lem:Rt}}
\begin{proof}
	First, we have
	\begin{equation*}
		\frac{d}{dt}\|w_x(t)\|^2 
		= -2\|\bm w\|^{-\frac{2}{L}}\left( \|\bm w\|^2 + (L-1)\|w_x\|^2\right)\left(w_x^T\frac{\partial J^e}{\partial w_x}\right).
	\end{equation*}
	This shows that $\frac{d}{dt}\|w_x(t)\|^2$ has the opposite sign to $w_x^T\frac{\partial J^e}{\partial w_x}$. Because of the Assumption~\ref{as:hyperplane-init} and the continuity of $\bm w(t)$, 
	we see that $w_x^T\frac{\partial J^e}{\partial w_x}$ keeps the same sign to the initialization since $w_x^T\frac{\partial J^e}{\partial w_x}=0$ if and only if $\bm w \in E$.
	
	Similar proof for $\frac{d}{dt} \frac{\|w_x(t)\|^2}{\|w_y(t)\|^2}$ can be shown by calculating directly.
	\begin{equation*}
		\begin{split}
			\frac{d}{dt} \frac{\|w_x(t)\|^2}{\|w_y(t)\|^2} &= \frac{1}{\|w_y(t)\|^4} \left( \left(\frac{d}{dt}\|w_x(t)\|^2 \right)\|w_y(t)\|^2 - \left(\frac{d}{dt}\|w_y(t)\|^2 \right)\|w_x(t)\|^2 \right) \\
			&= \frac{1}{\|w_y(t)\|^4} \left( \left(w^T_x(t)\frac{d}{dt}w_x(t) \right)\|w_y(t)\|^2 - \left(w_y^T(t)\frac{d}{dt}w_y(t) \right)\|w_x(t)\|^2 \right) \\
			&= -\|\bm w\|^{-\frac{2}{L}} \left( \frac{\|w_x(t)\|^2}{\|w_y(t)\|^2} + 1 \right) \left( w_x^T \frac{\partial J^e(w)}{\partial w_x}\right) < 0.
		\end{split}
	\end{equation*}
	Thus, $\frac{d}{dt} \frac{\|w_x(t)\|^2}{\|w_y(t)\|^2}$ also has the opposite sign to $w_x^T\frac{\partial J^e}{\partial w_x}$. 
\end{proof}

\subsection{Proof of Theorem~\ref{thm:LLN_wy_sigma=0}}
\begin{proof} 
	Let first consider $\bm w(0) \in U^+$. Then, we have
	\begin{equation*}
		\begin{split}
			\|w_y(T)\|^2 - \|w_y(0)\|^2 &= \int_{0}^T  \frac{d}{dt} \|w_y(t)\|^2 dt \\
			&= \int_{0}^T \left( \left.  \frac{d}{dt} \|w_y(t)\|^2 dt \right/ \frac{d}{dt} \|w_x(t)\|^2 \right) \frac{d}{dt} \|w_x(t)\|^2 dt\\
			&= \int_{0}^T \frac{(L-1)\|w_y\|^2}{L\|w_x\|^2 + \|w_y\|^2}  \frac{d}{dt} \|w_x(t)\|^2 dt \quad \left(\bm w(t)\cap E = \emptyset\right)\\
			&= \int_{0}^T \frac{L-1}{L(\|w_x\|^2/\|w_y\|^2) + 1}  \frac{d}{dt} \|w_x(t)\|^2 dt \quad \left(w_y(t) \neq 0\right)\\
			&\le 
			\frac{(L-1)\|w_y(0)\|^2}{L\|w_x(0)\|^2 + \|w_y(0)\|^2}\int_{0}^T  \frac{d}{dt}\|w_x(t)\|^2 dt  \\
			&= \frac{(L-1)\|w_y(0)\|^2}{L\|w_x(0)\|^2 + \|w_y(0)\|^2}\left( \|w_x(T)\|^2 - \|w_x(0)\|^2 \right).
		\end{split}
	\end{equation*} 
	The inequality holds since $\frac{d}{dt}\|w_x(t)\|^2 \le 0$ and $\frac{d}{dt} \frac{\|w_x\|^2}{\|w_y^2\|} \le 0$ for $0 \le t \le T$ if $\bm w(0)\in U^+$ according to Lemma~\ref{lem:Rt}. In addition, $\|w_y(T)\|^2 - \|w_y(0)\|^2 \le 0$ comes from the fact that $\|w_x(T)\|^2 \le \|w_x(0)\|^2 $ since $\frac{d}{dt}\|w_x(t)\|^2 \le 0$.
	
	If $\bm w(0)\in E^-$, we have $\frac{d}{dt}\|w_x(t)\|^2 \ge 0$ and $\frac{d}{dt} \frac{\|w_x\|^2}{\|w_y^2\|} \ge 0$ for $0 \le t \le T$ according to Lemma~\ref{lem:Rt}. Thus, we can prove it with the same calculation above. 
\end{proof}

\section{Proofs for ReLU DNNs}
\subsection{Proof of Lemma~\ref{lamm:e^2(x,y)}}
\begin{proof}
	For any $i$ and fixed $y \in \mathbb{R}^{d_y}$, let us first assume $W^1_{i,y}y\ge0$. Given the definition of ${\rm ReLU}$ activation function, we may consider the following four sets
	\begin{equation*}
		\begin{aligned}
			\{\overline W_{i,x}^1 x + \overline b^1_i \ge 0\} &\cap \{\overline W_{i,x}^1x + \overline b_i^1 +  W^1_{i,y}y\ge 0\}, \\
			\{\overline W_{i,x}^1 x + \overline b^1_i \ge 0\} &\cap \{\overline W_{i,x}^1x + \overline b_i^1 +  W^1_{i,y}y\le 0\}, \\
			\{\overline W_{i,x}^1 x + \overline b^1_i \le 0\} &\cap \{\overline W_{i,x}^1x + \overline b_i^1 +  W^1_{i,y}y\ge 0\}, \\
			\{\overline W_{i,x}^1 x + \overline b^1_i \le 0\} &\cap \{\overline W_{i,x}^1x + \overline b_i^1 +  W^1_{i,y}y\le 0\},
		\end{aligned}
	\end{equation*}
	to calculate $e_i(x,y)$ explicitly.
	Since $W^1_{i,y}y\ge0$, we have 
	\begin{equation*}
		\{\overline W_{i,x}^1x + \overline b_i^1 +  W^1_{i,y}y\ge 0\} \subset \{\overline W_{i,x}^1 x + \overline b^1_i \ge 0\},
	\end{equation*}
	which means
	$$
	\{\overline W_{i,x}^1 x + \overline b^1_i \ge 0\} \cap \{\overline W_{i,x}^1x + \overline b_i^1 +  W^1_{i,y}y\le 0\} = \emptyset.
	$$
	In addition, we have $e_i(x,y) = 0$ on $\{\overline W_{i,x}^1 x + \overline b^1_i \le 0\} \cap \{\overline W_{i,x}^1x + \overline b_i^1 +  W^1_{i,y}y\le 0\}$.
	As a result, we focus only on
	\begin{align}\label{eq:Omega-decomposed}
		\Omega_{i,x}^+ &:={\{\overline W_{i,x}^1x + \overline b^1_i \ge 0\}  \cap \Omega_x}\\
		\Omega_{i,x}^-&:={\{\overline W_{i,x}^1 x + \overline b^1_i \le 0\} \cap \{\overline W_{i,x}^1x + \overline b_i^1 +  W^1_{i,y}y\ge 0\} \cap \Omega_x}.
	\end{align}
	Then, if follows that
	\begin{align*}
		\|e_i(x,y)\|^2_{L^2(\Omega_x )} 
		&= \int_{\Omega_x} |e_i(x,y)|^2 dx \\
		&= \underbrace{\int_{\Omega_{i,x}^+} |\overline W^2_iW^1_{i,y}y|^2dx}_{I^+} 
		+\underbrace{\int_{\Omega_{i,x}^-} |\overline W^2_i(\overline W_{i,x}^1x+\overline b^1_i +W^1_{i,y} y )|^2dx}_{I^-}.
	\end{align*}
	For $I^+$, we have
	$$
	\int_{\Omega_{i,x}^+} |\overline W^2_iW^1_{i,y}y|^2dx 
	= |W^1_{i,y}y|^2 \int_{\Omega_{i,x}^+} |\overline W^2_i|^2dx
	= \frac{|W^1_{i,y}y|^2\|\nabla h_i(x)\|_{L^2(\Omega_x)}^2}{\|\overline W^1_{i,x}\|^2}
	$$
	because of the definition of $h_i(x)$ and the truncation property of ${\rm ReLU}$ activation function.
	
	For $I^-$, we first denote a series of parallel hyperplanes in $\mathbb R^{d_x}$ as
	$$
	H_{i,s} = \{ x ~|~ \overline W_{i,x}^1x + \overline b^1_i + sW^1_{i,y}y= 0\}.
	$$
	Then, we define 
	$$
	\mathcal{P}_{H_{i,s}}(\Omega_x) := \{ \mathcal{P}_{H_{i,s}}(x)  |  x \in \Omega_x \} \subset H_{i,s}
	$$
	as the projection of $\Omega_x$ onto $H_{i,s}$. 
	We notice that 
	$\mathcal{P}_{H_{i,s}}(\Omega_x)$ have the same measure for all $s \in [0,1]$.
	Finally, we define $\widetilde \Omega_{i,x}^-$ as the right cylinder which uses 
	$\mathcal{P}_{H_{i,0}}(\Omega_x)$ and $\mathcal{P}_{H_{i,1}}(\Omega_x)$ as its bases. More precisely, we can write it as
	\begin{equation}\label{eq:para_Omega}
		\widetilde \Omega_{i,x}^- := \left\{x \in \mathcal{P}_{H_{i,s}}(\Omega_x)~|~s \in [0,1] ~\right\}.
	\end{equation}
	Here, we present Figure~\ref{fig:Omegaix-} as a diagram about $\widetilde \Omega_{i,x}^-$.
	\begin{figure}[h]
		\centering
		\includegraphics[width=.8\textwidth]{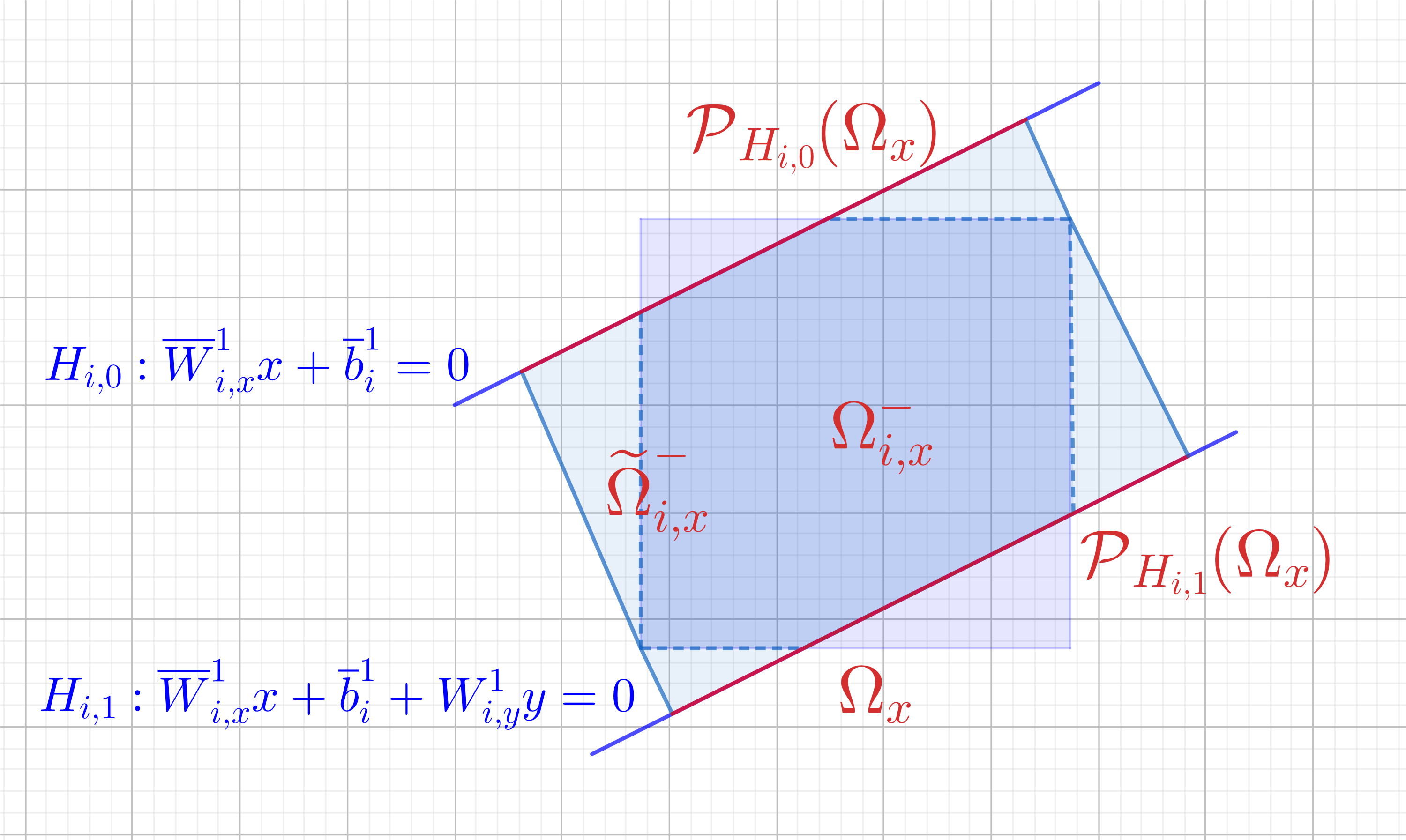}
		\caption{Diagram of $\Omega_{i,x}^-$ and $\widetilde \Omega_{i,x}^-$.}
		\label{fig:Omegaix-}
	\end{figure}
	Then, we have the following estimate by the parameterization of  
	$\widetilde \Omega_{i,x}^-$ in~\eqref{eq:para_Omega} and integral by substitution
	$$
	\begin{aligned}
		I^- &\le \int_{\widetilde \Omega_{i,x}^-} |\overline W^2_i(\overline W_{i,x}^1x+\overline b^1_i +W^1_{i,y} y )|^2dx \\
		&=\int_{0}^{1} \int_{\mathcal{P}_{H_{i,s}}(\Omega_x)} |\overline W^2_i(\overline W_{i,x}^1x+\overline b^1_i +W^1_{i,y} y )|^2 d\tilde x \frac{|W^1_{i,y}y|}{\|\overline W^1_{i,x}\|} ds   \\
		&= \int_{0}^{1}\frac{|W^1_{i,y}y|}{\|\overline W^1_{i,x}\|}
		|\overline W^2_i|^2(1-s)^2|W_{i,y}^1y|^2 \left| \mathcal{P}_{H_{i,0}}(\Omega_x)\right| ds \\
		&= \left|\mathcal{P}_{H_{i,0}}(\Omega_x) \right| \frac{|\overline W^2_i|^2|W_{i,y}^1y|^3}{3\|\overline W_{i,x}^1\|} \le C_{d_x} \frac{|\overline W^2_i|^2|W_{i,y}^1y|^3}{3\|\overline W_{i,x}^1\|}.
	\end{aligned}
	$$
	Here, we notice that $dx = d \tilde x \frac{|W^1_{i,y}y|}{\|\overline W^1_{i,x}\|} ds$ since the distance between $H_{i,0}$ and $H_{i,1}$ is $\frac{|W^1_{i,y}y|}{\|\overline W^1_{i,x}\|}$. In addition, $\left| \mathcal{P}_{H_{i,0}}(\Omega_x)\right|$ denotes the measure of $\mathcal{P}_{H_{i,0}}(\Omega_x)$ and we have
	$$
	C_{d_x} := \sup_{H} \left| \mathcal{P}_{H}(\Omega_x) \right| \ge \left| \mathcal{P}_{H_{i,0}}(\Omega_x)\right|
	$$
	for any $i=1:n$, where $H$ means a hyperplane in $\mathbb R^{d_x}$.
	
	If $W^1_{i,y}y\le0$, we denote
	\begin{align}
		\Omega_{i,x}^+ &:={\{\overline W_{i,x}^1x + \overline b^1_i + W^1_{i,y}y \ge 0\}  \cap \Omega_x}\\
		\Omega_{i,x}^-&:={\{\overline W_{i,x}^1 x + \overline b^1_i \ge 0\} \cap \{\overline W_{i,x}^1x + \overline b_i^1 +  W^1_{i,y}y\le 0\} \cap \Omega_x}.
	\end{align}
	Then, we still have $I^+$ and $I^-$ correspondingly. For $I^-$, we can follow the same strategy by defining $\widetilde \Omega_{i,x}^-$ and calculate the integral by decomposition and substitution. For $I^+$, we notice that
	$$
	\Omega_{i,x}^+ \subset {\{\overline W_{i,x}^1x + \overline b^1_i \ge 0\}  \cap \Omega_x}
	$$
	since $W^1_{i,y}y\le0$. This means 
	$$
	I^+ = \int_{\Omega_{i,x}^+} |\overline W^2_iW^1_{i,y}y|^2dx \le 
	\int_{{\{\overline W_{i,x}^1x + \overline b^1_i \ge 0\}  \cap \Omega_x}} |\overline W^2_iW^1_{i,y}y|^2dx.
	$$
	Then, we have the same estimate results as for the case $W^1_{i,y}y\ge0$. This finishes the proof. 
\end{proof}

\subsection{Proof of Corollary~\ref{coro:stabilityL>2}}
\begin{proof}
	Given the definition of $D(y)$ and the estimate in~\eqref{eq:estimate_L>2}, we have
	$$
	\left\|\triangle_y f(\cdot,y)\right\|^2_{L^2(\Omega_x)} \le \left(\prod_{\ell = 3}^L \left\|\overline{W}^{\ell}\right\|^2\right)
	\sum_{\substack{i=1:n_2 \\j=1:n_1}}\sum_{k=1:d_y}\left(
	\left|\left[W^1_{j,y}\right]_k\right|^2 + \left|\left[W^1_{j,y}\right]_k\right|^3
	\right)D(y).
	$$
	Since $\left[W^1_{j,y}\right]_k \sim \mathcal N(0,\nu^2)$ for all $j=1:n_1$ and $k=1:d_y$, it follows by the Monte Carlo estimate that with high probability there exists $\widetilde D_1$ such that
	$$
	\frac{1}{n_2n_2d_y}\sum_{\substack{i=1:n_2 \\j=1:n_1\\k=1:d_y}}\left|\left[W^1_{j,y}\right]_k\right|^2 -  \mathbb E \left[ \left|\left[W^1_{j,y}\right]_k\right|^2\right]\le \widetilde D_1\frac{1}{\sqrt{n_2n_1d_y}}.
	$$
	This leads to 
	$$
	\begin{aligned}
		\sum_{\substack{i=1:n_2\\j=1:n_1\\k=1:d_y}}\left|\left[W^1_{j,y}\right]_k\right|^2 &\le n_2n_1d_y \mathbb E \left[ \left|\left[W^1_{j,y}\right]_k\right|^2\right] +  \widetilde D_1\sqrt{n_2n_1d_y} \\
		&= n_2n_1d_y \nu^2 +  \widetilde D_1\sqrt{n_2n_1d_y},
	\end{aligned}
	$$
	with high probability.
	Similarly, we have
	$$
	\sum_{\substack{i=1:n_2\\j=1:n_1\\k=1:d_y}}\left|\left[W^1_{j,y}\right]_k\right|^3 \le n_2n_1d_y 2\sqrt{\frac{2}{\pi}}\nu^3 + \widetilde D_2\sqrt{n_2n_1d_y}.
	$$
	This proof is completed by taking $\widetilde D = \widetilde D_1+\widetilde D_2$. 
\end{proof}

\end{document}